\documentclass{article}
\usepackage{iclr2027_conference,times}

\usepackage{amsmath,amsfonts,bm}

\def\eqref#1{equation~\ref{#1}}
\def\Eqref#1{Equation~\ref{#1}}
\def\1{\bm{1}}

\def\vmu{{\bm{\mu}}}

\def\vc{{\bm{c}}}

\def\vh{{\bm{h}}}

\def\vs{{\bm{s}}}

\def\vv{{\bm{v}}}
\def\vw{{\bm{w}}}
\def\vx{{\bm{x}}}
\def\vy{{\bm{y}}}
\def\vz{{\bm{z}}}

\def\mA{{\bm{A}}}

\def\mD{{\bm{D}}}

\def\mG{{\bm{G}}}

\def\mM{{\bm{M}}}

\def\mS{{\bm{S}}}

\def\mW{{\bm{W}}}

\def\mLambda{{\bm{\Lambda}}}

\DeclareMathAlphabet{\mathsfit}{\encodingdefault}{\sfdefault}{m}{sl}
\SetMathAlphabet{\mathsfit}{bold}{\encodingdefault}{\sfdefault}{bx}{n}

\usepackage{graphicx}
\usepackage{wrapfig}
\usepackage{float}

\usepackage{titletoc}

\usepackage{hyperref}
\usepackage{url}

\newcommand{\methodbase}{iCReN}
\newcommand{\methodfullnamebase}{\methodbase, a framework for identifying latent states and their instantaneous and lagged causal relations using nonstationarity of transition-noise distributions}

\newcommand{\methodname}{\mbox{\methodbase}}
\newcommand{\methodfullname}{\methodfullnamebase}
\newcommand{\methodTCL}{\mbox{\methodbase-T}}
\newcommand{\methodGCL}{\mbox{\methodbase-G}}

\newcommand{\veps}{\bm{\epsilon}}
\newcommand{\vlambda}{\bm{\lambda}}
\newcommand{\vpsi}{\bm{\psi}}
\newcommand{\Einst}{\mathcal E_{\mathrm{inst}}}
\newcommand{\Elag}{\mathcal E_{\mathrm{lag}}}
\newcommand{\widehatEinst}{\widehat{\mathcal E}_{\mathrm{inst}}}
\newcommand{\widehatElag}{\widehat{\mathcal E}_{\mathrm{lag}}}

\newtheorem{theorem}{Theorem}

\newtheorem{lemma}{Lemma}
\newtheorem{definition}{Definition}
\newtheorem{assumption}{Assumption}
\newtheorem{remark}{Remark}

\newcounter{identcondition}
\renewcommand{\theidentcondition}{(\roman{identcondition})}
\newcommand{\identcondition}[2]{%
  \refstepcounter{identcondition}%
  \par\noindent\textbf{\theidentcondition\ #2}\label{#1}%
}

\newcounter{appendixcondition}[assumption]
\renewcommand{\theappendixcondition}{(\roman{appendixcondition})}
\newcommand{\appendixcondition}[2]{%
  \refstepcounter{appendixcondition}%
  \par\noindent\textbf{\theappendixcondition\ #2}\label{#1}%
}

\newcommand{\assumptionconditionref}[2]{%
  Assumption~\ref{#1}\ref{#2}%
}

\newenvironment{proof}{\par\noindent\textit{Proof.}\ }{\hfill$\square$\par}

\title{Causal Representation Learning with Instantaneous and Lagged Relations via Nonstationarity}

\iclrarxivtrue
\author{
Tatsuya Yamada$^{1,2}$ \quad
Hiroshi Morioka$^{3,1,2}$ \quad
Yoshinobu Kawahara$^{1,2}$ \\[0.3em]
$^1$The University of Osaka \quad
$^2$RIKEN \quad
$^3$Shiga University \\[0.3em]
{\normalfont\small\texttt{t-yamada@ist.osaka-u.ac.jp}} \\[0.1em]
{\normalfont\small\texttt{hiroshi-morioka@biwako.shiga-u.ac.jp}} \\[0.1em]
{\normalfont\small\texttt{kawahara@ist.osaka-u.ac.jp}}
}

\begin{document}

\maketitle

\begin{abstract}
Causal representation learning for time-series data aims to identify latent states and their causal relations from observations. In this setting, an important challenge is to model both lagged causal relations across observation intervals and faster causal effects that appear as instantaneous relations within an interval, while accounting for nonstationarity in time-series data. However, methods that jointly handle these causal relations and nonstationarity remain limited. To address this gap, we establish sufficient conditions for identifying latent states up to component permutation and component-wise invertible transformations, and their instantaneous and lagged causal structures up to the same permutation, using an observed auxiliary variable, such as time or a condition label, associated with changes in transition-noise distributions. Based on these results, we propose iCReN, a framework that uses contrastive learning with discrete or continuous auxiliary variables to learn latent representations and estimate their instantaneous and lagged causal structures. Experiments demonstrate accurate recovery of latent states and both instantaneous and lagged causal structures on synthetic data and the utility of the learned representations for downstream forecasting on real-world data.
\end{abstract}

\section{Introduction}
\label{sec:introduction}

Time-series analysis is used for a wide range of tasks~\citep{Yue2022a,Wu2023a}, including weather forecasting~\citep{Bi2023a,Lam2023a} and human motion analysis~\citep{Butepage2017a,Martinez2017a}.
Many time-series methods model dependencies among observed variables, while learning the factors underlying the observations and their causal relations may additionally support generalization under distribution shifts~\citep{Scholkopf2021a,Bagi2023a,Song2026a}.
When these factors are directly observed, causal discovery methods can be used to infer their causal structure~\citep{Shimizu2006a,Zheng2018a,Runge2020a}, whereas in many applications, the relevant factors are not directly observed but are instead entangled in observations such as sensor measurements or images.
Causal representation learning (CRL) therefore seeks to identify these factors and their causal relations in learned representations~\citep{Komanduri2024a}.
However, different latent representations can produce the same distribution of observations, so identifying the underlying factors requires additional conditions or information~\citep{Hyvarinen1999a,Khemakhem2020a}.
Consequently, a central question in CRL is under what conditions and with what additional information the underlying factors and their causal relations become identifiable~\citep{Moran2026a}.

Two important aspects of CRL for time-series data are modeling the underlying factors as latent states and their causal relations within and across time steps, and accounting for nonstationarity of the data distribution.
First, causal interactions occurring within an observation interval are represented as relations among variables at the same time step, termed instantaneous causal relations, whereas interactions across time steps are represented as lagged causal relations.
While earlier CRL work for time-series data largely focused on lagged causal relations~\citep{Yao2022b,Yao2022a}, recent work explicitly models both instantaneous and lagged relations~\citep{Lippe2023a,Li2025a,Li2025b,Song2025a}.
These relations can be represented using structural causal models (SCMs)~\citep{Pearl2009a}, where each component of the latent state is modeled as a function of its instantaneous and lagged parents and exogenous noise~\citep{Hyvarinen2010a,Peters2013a}.
Second, the distribution of time-series data can change across seasons, times of day, operating conditions, or experimental conditions; accounting for such nonstationarity is important for time-series modeling~\citep{Liu2022a,Ditzler2015a,Song2023a}.
Within this formulation, one way to model nonstationarity is to allow the distribution of the exogenous noise---which we call transition noise---to vary over time or across conditions.
Prior work has used such transition-noise variation~\citep{Yao2022b} or, more generally, nonstationarity in the transitions between latent states~\citep{Yao2022a,Song2023a,Song2024b} to identify latent states and their temporal relations.
However, methods that explicitly model such nonstationarity and use it to identify latent states and their instantaneous and lagged causal structures remain underexplored.

To address this gap, we propose \methodfullname{}.
We model nonstationarity through variation in transition-noise distributions associated with an observed auxiliary variable, while the structural functions governing transitions between latent states and the mapping from latent states to observations remain invariant.
Under this setting, our main contributions are as follows.
First, we establish sufficient conditions for identifying latent states and their instantaneous and lagged causal structures.
Second, we develop \methodTCL{} and \methodGCL{}, which instantiate \methodname{} using time-contrastive learning (TCL) and generalized contrastive learning (GCL), respectively.
Third, we demonstrate accurate recovery of latent states and causal structures on synthetic data and predictive utility on real-world data.
\section{Related Work}
\label{sec:related_work}

CRL methods for time-series data differ in what causal relations they identify and what information they use to do so.
While causal discovery methods for time-series data identify instantaneous and lagged causal relations among directly observed variables under assumptions based on non-Gaussianity, conditional independence, or regime-switching nonstationarity~\citep{Hyvarinen2010a,Runge2020a,Hulsman2026a}, CRL must identify both the latent states underlying the observations and their causal relations, requiring further assumptions or information.
Several CRL approaches identify latent states and often their lagged causal relations using nonstationarity in transition-noise or latent transition distributions~\citep{Yao2022b,Yao2022a,Song2023a,Song2024b,Chen2024a} or component-level interventions~\citep{Lippe2022a}.
In particular, similar to our approach, LEAP~\citep{Yao2022b} exploits variation in transition-noise distributions while assuming invariant structural functions and an invariant observation process, but excludes instantaneous parents.
Among methods that allow instantaneous causal relations, iCITRIS~\citep{Lippe2023a} assumes known intervention targets and interventions that remove each targeted variable's dependence on its instantaneous causal parents. 
G-CaRL~\citep{Morioka2024a} can also be applied to time-series data with both instantaneous and lagged causal relations; when applied in this temporal setting, however, it identifies only lagged causal relations.
IDOL~\citep{Li2025a} identifies latent states and their instantaneous and lagged causal relations from stationary time-series data without auxiliary information, using sufficiently diverse dependencies on past states together with sparsity of the causal structure.
More broadly, prior work has studied how sparse dependencies can support identification of latent variables~\citep{Lachapelle2022a}. 
Our work identifies latent states and both instantaneous and lagged causal structures when an observed auxiliary variable is associated with changes in the transition-noise distributions.

The methods above also differ in how representations are learned.
LEAP and IDOL use variational inference with variational autoencoders (VAEs), while iCITRIS also offers a normalizing-flow variant. 
In contrast, our identification framework uses contrastive learning without requiring a variational posterior or decoder.
Our approach builds on independent innovation analysis (IIA)~\citep{Morioka2021a}, which uses distributional changes associated with an auxiliary variable to identify transition noise up to a component permutation and component-wise invertible transformations. 
Its TCL and GCL variants accommodate discrete and general auxiliary variables, respectively, and identify the transition-noise components by exploiting sufficient variability in the associated distributional changes.
However, IIA alone identifies neither the latent states nor their causal structure.
Under additional structural constraints on the latent data-generating process and restrictions on the learned model class, we extend its transition-noise identification result to the latent states and their instantaneous and lagged causal structures.
Table~\ref{tab:related_work} summarizes the differences among the most closely related methods.

\begin{table}[tbp]
\centering
\caption{Comparison of closely related CRL methods for time-series data by the causal relations they identify, the information they use for identification, and how they learn representations. For iCITRIS (*), the distribution changes are induced by interventions with known targets.}
\label{tab:related_work}
\begingroup
\footnotesize
\setlength{\tabcolsep}{2pt}
\renewcommand{\arraystretch}{1.15}
\newcommand{\yes}{\checkmark}
\newcommand{\no}{$\times$}
\resizebox{\textwidth}{!}{%
\begin{tabular}{l|ccc|c}
\hline
 & \shortstack{Instantaneous\\causal structure\\identification}
 & \shortstack{Nonstationarity}
 & \shortstack{No intervention-target\\labels required}
 & \shortstack{No variational\\inference required} \\
\hline
LEAP~\citep{Yao2022b} & \no & \yes & \yes & \no \\
iCITRIS~\citep{Lippe2023a} & \yes & \yes* & \no & \yes \\
G-CaRL~\citep{Morioka2024a} & \no & \no & \yes & \yes \\
IDOL~\citep{Li2025a} & \yes & \no & \yes & \no \\
\hline
\methodname{} & \yes & \yes & \yes & \yes \\
\hline
\end{tabular}
}
\endgroup
\end{table}
\section{Problem Setup}
\label{sec:problem-setup}

We consider a generative model for time-series observations with both instantaneous and lagged causal relations among latent states.
As shown in Figure~\ref{fig:proposal_overview}, the latent states evolve according to structural functions and are mapped to observations through an observation function.
In our setting, the transition-noise distributions may vary depending on an auxiliary variable, whereas the structural functions and observation function remain invariant.

\begin{wrapfigure}{r}{0.47\columnwidth}
    \vspace{-0.5\baselineskip}
    \centering
    \includegraphics[width=\linewidth]{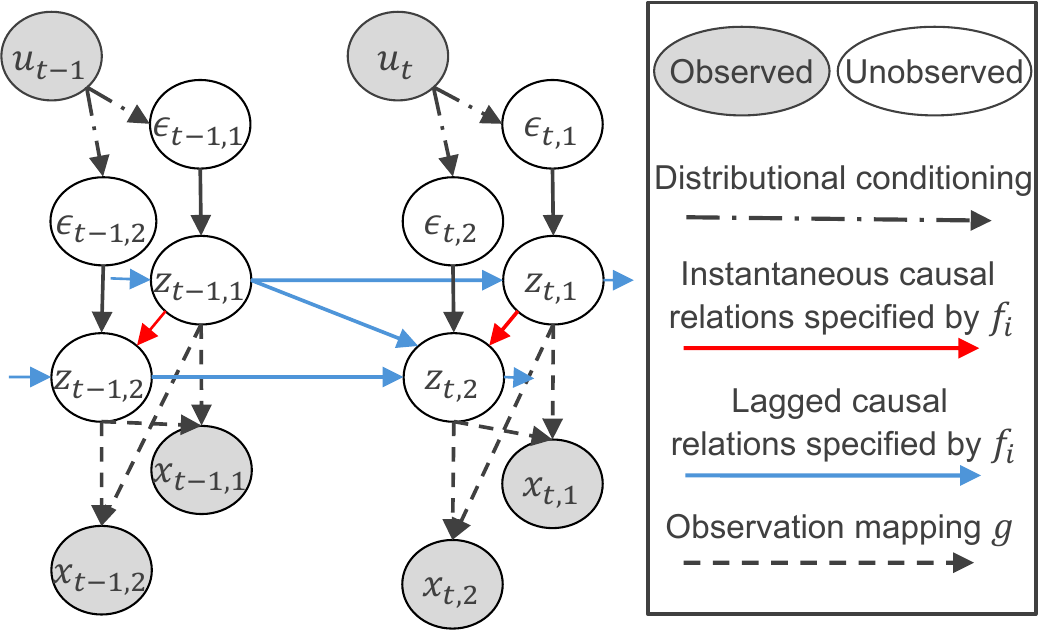}
    \caption{Generative model with $d_z=2$ and $L=1$. Gray dash-dotted, solid black, red, blue, and dashed black arrows denote conditioning of the transition-noise distributions on the auxiliary variable, transition-noise inputs to the structural functions, instantaneous causal relations, lagged causal relations, and the observation mapping, respectively.}
    \vspace{-1.5\baselineskip}
    \label{fig:proposal_overview}
\end{wrapfigure}

\paragraph{Observation model.}
At each time step $t$, let $\vx_t\in\mathcal X$ denote an observation, with a corresponding continuous latent state $\vz_t\in\mathbb{R}^{d_z}$.
The observation is generated as $\vx_t=g(\vz_t)$, where $g$ is an observation function that represents the process through which the latent states generate observations.

\paragraph{Latent dynamics.}
We use structural equations~\citep{Pearl2009a} to model the temporal evolution of the latent states, with instantaneous causal relations within each time step and lagged causal relations across time steps.
We consider a finite-lag model with maximum lag $L\geq1$, in which the current latent state may receive direct effects from latent states up to $L$ time steps in the past.
For each latent component $z_{t,i}$, we let $f_i$ denote its structural function and define the same-time and past latent variables on which $f_i$ depends nontrivially as its instantaneous and lagged parents, respectively.
These parent sets satisfy
\begin{equation}
    \mathrm{pa}_{\mathrm{inst}}(z_{t,i})
    \subseteq \{z_{t,j}:j\ne i\},
    \qquad
    \mathrm{pa}_{\mathrm{lags}}(z_{t,i})
    \subseteq
    \{z_{t-\ell,j}:1\leq\ell\leq L,\ 1\leq j\leq d_z\}.
\end{equation}
Each latent component is generated as
\begin{equation}
    z_{t,i}
    =f_i\bigl(
        \mathrm{pa}_{\mathrm{inst}}(z_{t,i}),
        \mathrm{pa}_{\mathrm{lags}}(z_{t,i}),
        \epsilon_{t,i}
    \bigr).
    \label{eq:latent_dynamics}
\end{equation}
Here, $\epsilon_{t,i}$ is the transition noise for component $i$.
This formulation includes, as special cases, models in which the instantaneous parent sets are empty and models in which transition noise enters additively.
Each parent relation corresponds to a directed causal edge from the parent variable to $z_{t,i}$.
The instantaneous causal structure is assumed to form a directed acyclic graph (DAG).
We further assume that there are no unmodeled common causes of the latent components.

\paragraph{Nonstationary transition noise.}
To model nonstationarity, we allow each conditional transition-noise distribution $p_i(\epsilon_{t,i}\mid u_t)$ to vary depending on an observed auxiliary variable $u_t$.
The auxiliary variable may be continuous, such as time, or discrete, such as a label indicating an environment, experimental condition, phase, season, or time of day.
Conditioned on the auxiliary-variable sequence $u_{1:T}$, the transition-noise components are independent across components and time steps, and $\veps_t$ is independent of the latent states in the past $L$ time steps given $u_t$:
\begin{equation}
    p(\veps_{1:T}\mid u_{1:T})
    =
    \prod_{t=1}^{T}\prod_{i=1}^{d_z}
    p_i(\epsilon_{t,i}\mid u_t),
    \qquad
    \veps_t
    \perp
    (\vz_{t-1},\ldots,\vz_{t-L})
    \mid u_t .
    \label{eq:transition_noise_assumptions}
\end{equation}
Under this model, we seek to identify the latent states up to a component permutation and component-wise invertible transformations, and their instantaneous and lagged causal structures up to the same permutation, from the joint distribution of the observation sequence $(\vx_t)_{t=1}^{T}$ and the auxiliary-variable sequence $(u_t)_{t=1}^{T}$.
\section{Identifiability Analysis}
\label{sec:identifiability}

% Topic1: purpose, setting
In this section, we establish sufficient conditions for identifying the latent states and for identifying their instantaneous and lagged causal structures.
For simplicity, we set $L=1$ in this section and extend the results to any finite maximum lag $L$ in Appendix~\ref{app:multiple_lags}.

% Topic2: identify latent components
\subsection{Identification of Latent Components}
\label{subsec:latent_component_identification}
We first establish conditions for identifying the components of the latent state.
Let $\phi$ be an encoder such that $\hat{\vz}_t=\phi(\vx_t)$.
Because $\vx_t=g(\vz_t)$, the learned representation satisfies $\hat{\vz}_t=(\phi\circ g)(\vz_t)$.
When $\phi\circ g$ is invertible, we denote its inverse by $H=(\phi\circ g)^{-1}$, so that $\vz_t=H(\hat{\vz}_t)$.
% <Instantaneous and cross-time Markov networks>
We next define the conditional-dependence structures.
Let $\mathcal M_z$ be the pairwise Markov network representing conditional-dependence relations among the components of $\vz_t$ given $\vz_{t-1}$ and $u_t$, and let $\Einst$ denote its edge set.
We further let $\Elag$ denote the set of conditional-dependence edges between the components of $\vz_{t-1}$ and $\vz_t$ in the pairwise Markov network for $(\vz_{t-1},\vz_t)$ given $u_t$.
For distinct $i,j$ and any $a$, the absence of an edge is characterized by
\begin{align}
\{z_{t,i},z_{t,j}\}\notin \Einst
&\Longleftrightarrow
z_{t,i}\perp z_{t,j}
\mid \vz_{t,\setminus{i,j}},\vz_{t-1},u_t,
\label{eq:instantaneous_markov_relations} \\
\{z_{t-1,a},z_{t,i}\}\notin \Elag
&\Longleftrightarrow
z_{t-1,a}\perp z_{t,i}
\mid \vz_{t-1,\setminus{a}},\vz_{t,\setminus{i}},u_t.
\label{eq:lagged_markov_relations}
\end{align}
Here, $\vz_{s,\setminus S}$ denotes the vector obtained from $\vz_s$ by excluding the components indexed by $S$.
We denote the corresponding quantities for the learned representation by
$\mathcal M_{\hat z}$, $\widehatEinst$, and $\widehatElag$.
Because these edge sets are defined by conditional dependence, their edges may correspond to direct causal edges or connect variables with no direct causal edge.
For example, in $z_{t,i}\to z_{t,k}\leftarrow z_{t,j}$, the conditioning set in~\eqref{eq:instantaneous_markov_relations} includes the common child $z_{t,k}$, and conditioning on it can induce dependence between $z_{t,i}$ and $z_{t,j}$, even though no direct causal edge connects them.

% Distributional variation used for identification
We next introduce quantities that characterize how differences in the conditional log transition density across auxiliary-variable values depend on the current and past latent states.
Fix a value $u_t^{(0)}$ of the auxiliary variable as a reference, and let $u_t^{(1)},\ldots,u_t^{(N_{\mathrm{aux}})}$ denote the comparison values.
For each $n=1,\ldots,N_{\mathrm{aux}}$, define
\begin{equation}
\Delta_n(\vz_t,\vz_{t-1})
:=
\log p(\vz_t\mid \vz_{t-1},u_t^{(n)})
-
\log p(\vz_t\mid \vz_{t-1},u_t^{(0)}),
\label{eq:instantaneous_log_transition_difference}
\end{equation}
and define $\widehat\Delta_n(\hat{\vz}_t, \hat{\vz}_{t-1})$ in the same way for the learned representation.
We then define the instantaneous and cross-time variation vectors
\begin{align}
\vw_n(\vz_t,\vz_{t-1})
&:=
\left[
\left(
\frac{\partial^2\Delta_n}{\partial z_{t,i}^2}
\right)_{i=1}^{d_z},
\left(
\frac{\partial^2\Delta_n}
{\partial z_{t,i}\partial z_{t,j}}
\right)_{\{z_{t,i},z_{t,j}\}\in \Einst},
\left(
\frac{\partial\Delta_n}{\partial z_{t,i}}
\right)_{i=1}^{d_z}
\right]^\top,
\label{eq:instantaneous_variation_vector}\\
\vv_{ia}(\vz_t,\vz_{t-1})
&:=
\left(
\frac{\partial^2\Delta_n}
{\partial z_{t,i}\partial z_{t-1,a}}
\right)_{n=1}^{N_{\mathrm{aux}}}.
\label{eq:cross_time_variation_vector}
\end{align}
The corresponding learned quantity $\hat{\vv}_{ia}(\hat{\vz}_t,\hat{\vz}_{t-1})$ is obtained from $\widehat\Delta_n(\hat{\vz}_t,\hat{\vz}_{t-1})$ in the same manner.
To examine cross-time relations between $\vz_t$ and $\vz_{t+1}$, we also use $\vv_{bi}(\vz_{t+1},\vz_t)$, with $\hat{\vv}_{bi}(\hat{\vz}_{t+1},\hat{\vz}_t)$ denoting its learned counterpart.
Because $H$ does not depend on the value of the auxiliary variable, the Jacobian term introduced by the change of variables cancels in the log-density ratio, giving $\widehat\Delta_n(\hat{\vz}_t,\hat{\vz}_{t-1}) =\Delta_n(H(\hat{\vz}_t),H(\hat{\vz}_{t-1}))$.
For fixed learned-component indices $j$ and $a$, differentiating this equation expresses $\hat{\vv}_{ja}(\hat{\vz}_t,\hat{\vz}_{t-1})$ as a linear combination of the true variation vectors $\vv_{km}(H(\hat{\vz}_t),H(\hat{\vz}_{t-1}))$ over current and past component indices $k$ and $m$, where the coefficients are products of Jacobian entries of $H$.
In Theorem~\ref{thm:elementwise_identification}, the instantaneous variation vector $\vw_n$ and the cross-time variation vector $\vv_{ia}$ play complementary roles in establishing component-wise identification.
First, the condition involving $\vw_n$ restricts possible mixing between latent components to a specific type of neighboring pair in $\mathcal M_z$.
Following \citet{Zhang2024a}, we call $z_{t,j}$ an intimate neighbor of $z_{t,i}$ if $z_{t,j}$ is adjacent to $z_{t,i}$ and to every other neighbor of $z_{t,i}$; the formal definition is given in Appendix~\ref{app:latent_component_identification}.
The condition involving $\vv_{ia}$ then eliminates this remaining mixing, yielding component-wise identification.

\begin{theorem}[Identification of latent components]
\label{thm:elementwise_identification}
Suppose that the following conditions hold.
\identcondition{cond:latent-regularity}{Regularity.}
The inverse $H=(\phi\circ g)^{-1}$ exists and is a $C^2$ diffeomorphism independent of time and the value of the auxiliary variable.
The transition densities are positive, and their log densities are twice continuously differentiable.

\identcondition{cond:instantaneous-variation}{Instantaneous variation.}
Let $N_{\mathrm{inst}}:=2d_z+|\Einst|$. 
For every $\vz_t$, there exists a set of $N_{\mathrm{inst}}$ pairs $(n_q,\bar{\vz}_{t-1}^{(q)})\in \{1,\ldots,N_{\mathrm{aux}}\}\times\mathbb R^{d_z}$, $q=1,\ldots,N_{\mathrm{inst}}$, such that
\begin{equation}
\mW(\vz_t):=
\left[
\vw_{n_q}(\vz_t,\bar{\vz}_{t-1}^{(q)})
\right]_{q=1}^{N_{\mathrm{inst}}}
\quad\text{has rank }N_{\mathrm{inst}}.
\label{eq:instantaneous_variation_matrix}
\end{equation}
Here, the indices $n_q$ need not be distinct from one another.

\identcondition{cond:cross-time-variation}{Cross-time variation and separation.}
For each learned variation vector $\hat{\vv}_{ja}(\hat{\vz}_t,\hat{\vz}_{t-1})$, the vectors $\vv_{km}(\vz_t,\vz_{t-1})$ indexed by all pairs $(k,m)$ that can contribute to its linear expansion above are linearly independent almost everywhere.
Moreover, for every intimate neighbor $z_{t,j}$ of $z_{t,i}$ in $\mathcal M_z$, there exists a past or future component with a conditional-dependence edge to $z_{t,i}$ and no such edge to $z_{t,j}$.
The precise index sets and past and future conditions are given in Appendix~\ref{app:latent_component_identification}.

\identcondition{cond:learned-edge-bounds}{Bounds on the learned edge sets.}
$|\widehatEinst|\leq|\Einst|$ and $|\widehatElag|\leq|\Elag|$ hold.

\par\noindent
Under these conditions, there exist a permutation $\pi_z$ and univariate $C^2$ diffeomorphisms $\xi_1,\ldots,\xi_{d_z}$ such that $\hat z_{t,\pi_z(i)}=\xi_i(z_{t,i})$ for every $i$ and $t$.
The same permutation $\pi_z$ also defines an isomorphism between $\mathcal M_z$ and $\mathcal M_{\hat z}$.
\end{theorem}

% discussion of the assumptions
The following discussion clarifies the scope of the conditions in Theorem~\ref{thm:elementwise_identification} and the role of Condition~\ref{cond:cross-time-variation} when intimate neighbors are present.
Compared with IDOL, whose sufficient-variability condition uses only variation across past states, Condition~\ref{cond:instantaneous-variation} additionally exploits changes in transition-noise distributions across auxiliary-variable values, including changes in their mean or variance.
Consequently, our condition may hold for additive Gaussian transition noise in cases where IDOL's condition does not~\citep{Li2025a}.
However, as in IDOL, Condition~\ref{cond:instantaneous-variation} cannot hold when each transition-noise component can be expressed as the sum of a function of the current state and a function of the past state in the presence of instantaneous edges, as shown in Remark~\ref{remark:supp_instantaneous_additive_separation_rank_bound}.
Condition~\ref{cond:cross-time-variation} requires linear independence to prevent cancellation among the relevant cross-time variation vectors and uses differences in past or future conditional-dependence relations to eliminate mixing between intimate neighbors.
Complete proofs are given in Appendices~\ref{app:instantaneous_markov_identification} and~\ref{app:latent_component_identification}.

% Topic3: Graph identification
\subsection{Identification of the Causal Graph}
\label{subsec:causal_graph_identification}

We then establish conditions for identifying the causal graph.
For the learned representation, let $\hat r$ be a learned inverse transition function that estimates transition noise as $\hat{\veps}_t=\hat r(\hat{\vz}_t,\hat{\vz}_{t-1})$.
To express direct causal edges through derivatives, we regard each $f_i$ as a function of all candidate instantaneous and lagged variables and characterize its parent sets by the variables whose corresponding partial derivatives are not identically zero.
Theorem~\ref{thm:causal_graph_identification} builds on the latent-component identification in Theorem~\ref{thm:elementwise_identification} and component-wise identification of transition noise to establish a correspondence between Jacobian entries of $\hat r$ that are not identically zero and direct causal edges.

\begin{theorem}[Identification of the causal graph]
\label{thm:causal_graph_identification}
Suppose that Conditions~\ref{cond:latent-regularity}--\ref{cond:learned-edge-bounds} of Theorem~\ref{thm:elementwise_identification} and the following additional conditions hold, and let $\pi_z$ denote the resulting permutation between the true and learned latent components.

\identcondition{cond:transition-noise-identification}{Transition-noise identification.}
The conditional transition-noise distribution factorizes across components, with each factor belonging to an exponential family with at least two twice continuously differentiable and nonredundant sufficient statistics, such as $\epsilon_{t,i}$ and $\epsilon_{t,i}^2$ for a Gaussian distribution.
The natural parameters vary sufficiently across auxiliary-variable values to satisfy the full-rank condition required by IIA~\citep{Morioka2021a}.
The remaining regularity conditions on the true and learned transformations, together with the estimation conditions for contrastive learning, are given in Assumption~\ref{assumption:supp_iia_transition_noise_identification}.

\identcondition{cond:structural-function-regularity}{Differentiability of the structural functions.}
Each $f_i$ is continuously differentiable.

\par\noindent
Under the above conditions, there exist a permutation $\pi_\epsilon$ and univariate $C^2$ diffeomorphisms $q_i$ such that $\hat\epsilon_{t,\pi_\epsilon(i)}=q_i(\epsilon_{t,i})$ for every $i$ and $t$.
Moreover, because the instantaneous causal structure is a DAG, the instantaneous Jacobian support of $\hat r$ uniquely determines the correspondence between the indices $\pi_\epsilon(i)$ and $\pi_z(i)$.
Under this correspondence, the following equivalence holds for
$\ell\in\{0,1\}$:
\begin{equation}
\frac{\partial\hat r_{\pi_\epsilon(i)}}
{\partial\hat z_{t-\ell,\pi_z(j)}}
\not\equiv0
\quad\Longleftrightarrow\quad
\frac{\partial f_i}
{\partial z_{t-\ell,j}}
\not\equiv0.
\label{eq:transition_noise_jacobian_and_causal_edge}
\end{equation}
Here, for $\ell=0$, the case $i=j$ is excluded because instantaneous causal edges are defined only between distinct components.
Consequently, the true instantaneous and lagged causal edges are identifiable up to the component permutation $\pi_z$ from the corresponding Jacobian supports of $\hat r$.
\end{theorem}

We discuss the scope of the conditions and theoretical setting in Theorem~\ref{thm:causal_graph_identification}.
Condition~\ref{cond:transition-noise-identification}, inherited from the IIA identification result~\citep{Morioka2021a}, requires the conditional transition-noise distributions to vary sufficiently across auxiliary-variable values and therefore does not cover data for which these distributions are invariant across those values.
The smoothness in Conditions~\ref{cond:transition-noise-identification} and~\ref{cond:structural-function-regularity} is used to derive the equivalence in~\eqref{eq:transition_noise_jacobian_and_causal_edge}.
Using this equivalence to identify causal edges from the Jacobian support of $\hat r$ applies the approach of \citet{Reizinger2023a} to latent transition dynamics.
Appendix~\ref{app:transition_noise_identification} establishes transition-noise identification under the sufficient setting $\mathcal X=\mathbb R^{d_z}$ used in the IIA argument, and Appendix~\ref{app:causal_graph_identification} provides the complete causal-graph identification proof.
The latent-component result relies on the $C^2$ diffeomorphism $H$ in Condition~\ref{cond:latent-regularity}, whereas the causal-graph result additionally uses the transition-noise correspondence; neither identification argument otherwise depends on the dimension of the observations.
Related invertibility assumptions between learned and true latent states are also used in CRL analyses~\citep{Lippe2023a,Morioka2024a,Li2025a}.
For higher-dimensional observations, the same arguments would apply if the learned representation admitted such an $H$ and allowed the transition noise to be identified; however, we do not establish conditions under which such a representation can be learned.

\section{Learning Framework}
\label{sec:learning-method}

This section describes the learning framework, summarized in Figure~\ref{fig:learning_framework}, and how the causal graph is subsequently estimated from the learned model.
We first apply the shared encoder $\phi$ separately to each observation $\vx_{t-\ell}$ for $\ell=0,\ldots,L$ and then apply the learned inverse transition function $\hat r$ to the resulting representations.
For contrastive learning, we also use a history feature map $\eta$ to obtain a history feature $\vh_t$ from the past learned representations:
\begin{equation}
    \hat{\vz}_{t-\ell}=\phi(\vx_{t-\ell}),
    \qquad
    \hat{\veps}_t=\hat r(\hat{\vz}_t,\hat{\vz}_{t-1},\ldots,\hat{\vz}_{t-L}),
    \qquad
    \vh_t=\eta(\hat{\vz}_{t-1},\ldots,\hat{\vz}_{t-L}).
    \label{eq:learning_latent_and_noise}
\end{equation}
Further parameterization details are provided in Appendix~\ref{app:learning-model-parameterization}.

\paragraph{Contrastive learning with auxiliary variables.}
Based on IIA~\citep{Morioka2021a}, we consider two variants of contrastive learning: TCL for discrete auxiliary variables, such as labels indicating an environment, experimental condition, motion phase, season, or time of day, and GCL for general auxiliary variables, including continuous values such as time.
TCL and GCL use the following regression functions, respectively:
\begin{equation}
 \begin{aligned}
 T_u^{\mathrm{TCL}}(\hat{\veps}_t,\vh_t)
 &={}
 \sum_{i=1}^{d_z}\sum_{k=1}^{K}
 w_{uik}\psi_{ik}(\hat\epsilon_{t,i})
 +\varphi(\vh_t,u)+b_u,\\
 T^{\mathrm{GCL}}(\hat{\veps}_t,\vh_t,u_t)
 &={}
 \sum_{i=1}^{d_z}\sum_{k=1}^{K}
 \psi_{ik}(\hat\epsilon_{t,i})
 \mu_{ik}(u_t)
 +\varphi(\vh_t,u_t)
 +B(\hat{\veps}_t,\vh_t).
 \end{aligned}
 \label{eq:learning_contrastive_scores}
\end{equation}
Here, $K$ is the number of feature functions for each transition-noise component, and $\psi_{ik}$ denotes the $k$-th scalar-valued feature function applied to $\hat\epsilon_{t,i}$.
The feature functions may be learned jointly with the remaining model components or, when the conditional noise family is specified in advance, fixed to include its sufficient statistics.
For TCL, $u$ indexes a class, $T_u^{\mathrm{TCL}}$ is its logit, and $w_{uik}$ and $b_u$ are learned separately for each class.
For GCL, the learned function $\mu_{ik}$ determines the coefficient of the $k$-th feature of component $i$ from the auxiliary variable.
The learned function $\varphi$ models the relation between the history feature and the auxiliary variable, while $B$ collects terms independent of the auxiliary variable.
In both regression functions, the terms that depend jointly on transition noise and the auxiliary variable decompose additively across components, as required by IIA.

The TCL logits over all possible values of the discrete auxiliary variable define a multinomial logistic regression model, which is trained using softmax cross-entropy with the observed $u_t$ as the target.
In GCL, $T^{\mathrm{GCL}}$ serves as the logit of a binary logistic regression model trained using binary cross-entropy, with matched tuples $(\hat{\veps}_t,\vh_t,u_t)$ as positive examples and mismatched tuples $(\hat{\veps}_t,\vh_t,u'_t)$ as negative examples; $u'_t$ is obtained by permuting auxiliary-variable values within each minibatch.
We write $\mathcal L_{\mathrm{IIA}}$ for the selected TCL or GCL loss; see Appendix~\ref{app:learning-contrastive} for details.

\begin{wrapfigure}{r}{0.35\columnwidth}
    \vspace{-1.3\baselineskip}
    \centering
    \includegraphics[width=\linewidth]{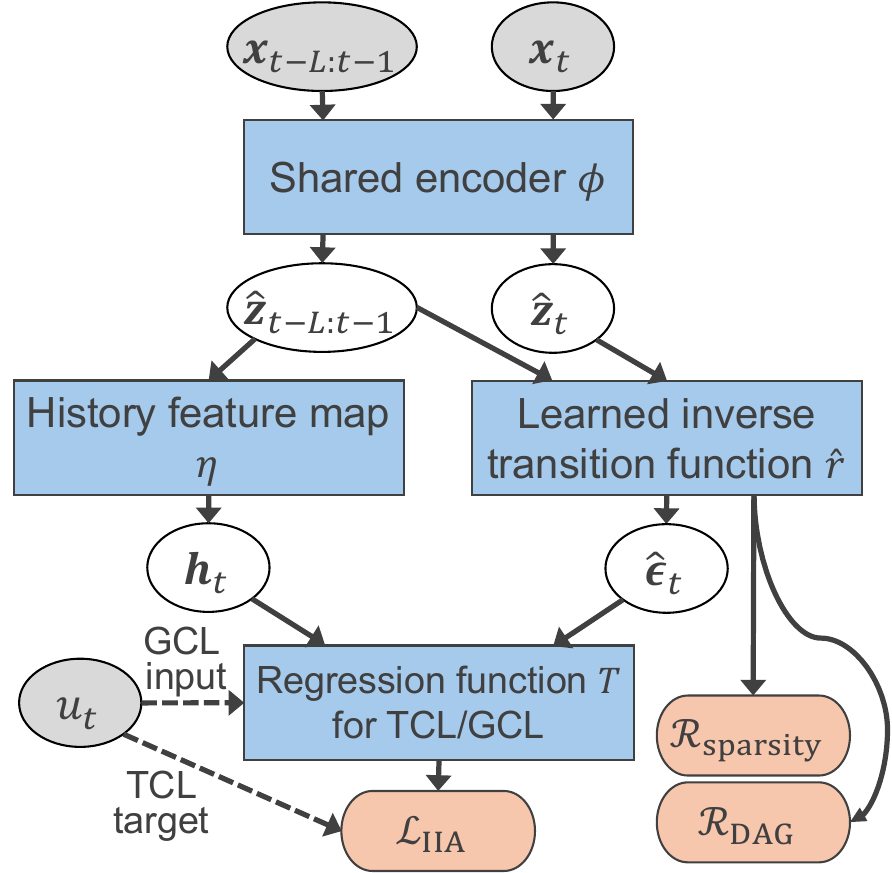}
    \vspace{-1.5\baselineskip}
    \caption{Learning framework. Gray nodes denote observed variables; white nodes denote learned representations, estimated transition noise, and history features. Dashed arrows denote auxiliary-variable inputs to TCL/GCL. The encoder $\phi$ is shared across time, and blue modules are trained jointly using the orange objectives.}
    \label{fig:learning_framework}
\end{wrapfigure}

\paragraph{Structural regularization and training objective.}
Motivated by the bounds on learned edge sets in Condition~\ref{cond:learned-edge-bounds}, we impose sparsity regularization on the instantaneous and lagged input dependencies of $\hat r$.
Following IDOL's use of sparsity regularization~\citep{Li2025a}, we use $\ell_1$-based regularizers $\mathcal R_{\mathrm{inst}}$ and $\mathcal R_{\mathrm{lag}}$ to encourage sparse instantaneous and lagged input dependencies of $\hat r$, respectively.
We denote their sum by $\mathcal R_{\mathrm{sparsity}}$.
Because latent components are identifiable only up to permutation, any acyclic instantaneous structure can be represented under a fixed triangular mask after an appropriate permutation of the latent coordinates. 
However, this requires the encoder to align the learned components with the fixed order encoded by the mask. 
We therefore also consider a variant using the NOTEARS acyclicity regularizer $\mathcal R_{\mathrm{DAG}}$~\citep{Zheng2018a}, which does not fix an order during training.
We optimize all learnable parameters using the objective $\mathcal L$:
\begin{equation}
    \mathcal L
    :=
    \mathcal L_{\mathrm{IIA}}
    +\lambda_{\mathrm{sparsity}}\mathcal R_{\mathrm{sparsity}}
    +\lambda_{\mathrm{DAG}}\mathcal R_{\mathrm{DAG}}.
    \label{eq:learning_full_objective}
\end{equation}
Here, $\lambda_{\mathrm{sparsity}}$ and $\lambda_{\mathrm{DAG}}$ are nonnegative coefficients controlling the strengths of the sparsity and acyclicity regularization, respectively.
When the fixed triangular mask is used, the $\mathcal R_{\mathrm{DAG}}$ term can be omitted.
Further details of the regularizers and training objective are provided in Appendix~\ref{app:learning-structural-regularization}.

\paragraph{Causal-graph estimation.}
After learning, we estimate a causal graph over the learned latent coordinates from the Jacobian support of $\hat r$; under the conditions of Theorem~\ref{thm:causal_graph_identification}, the estimated edges correspond to the true causal edges up to a component permutation.
For the case $L=1$ considered in the theory, we estimate that an instantaneous edge $\hat z_{t,j}\to\hat z_{t,i}$ is present for $i\neq j$ when $\partial\hat r_i/\partial\hat z_{t,j}\not\equiv0$, and that a lagged edge $\hat z_{t-1,j}\to\hat z_{t,i}$ is present when $\partial\hat r_i/\partial\hat z_{t-1,j}\not\equiv0$.
The same rule can be applied to each lag $\ell=1,\ldots,L$.

\section{Experiments}
\label{sec:experiments}

We evaluate latent-state and causal-graph recovery on synthetic data and the utility of learned representations for downstream forecasting on real-world data.
The evaluated CRL baselines are IDOL~\citep{Li2025a}, G-CaRL~\citep{Morioka2024a}, CtrlNS~\citep{Song2024b}, and LEAP~\citep{Yao2022b}.
iCITRIS~\citep{Lippe2023a} is not evaluated because the experimental settings provide neither its required intervention targets nor interventions that remove incoming instantaneous edges.
In addition, iVAE~\citep{Khemakhem2020a} and $\beta$-VAE~\citep{Higgins2017a} are included as representation-learning baselines that use and do not use the auxiliary variable, respectively. 
iVAE is evaluated only on synthetic data because, unlike the other methods, its inference model conditions on $u$ when encoding observations at downstream inference.

\subsection{Experiments on Synthetic Data}
\label{sec:exp_synth}

We generate synthetic data with nonlinear structural and observation functions defined by three-layer neural networks that remain invariant across auxiliary-variable values, while Gaussian transition-noise distributions vary discretely across environments or continuously over time.
The main comparison uses $\vx_t,\vz_t\in\mathbb R^5$ and $L=1$, with $2^{16}$ samples from 256 environments for discrete-variation data and a sequence of length $2^{18}$ with time as the auxiliary variable for continuous-variation data.
We measure latent-state recovery by the mean correlation coefficient (MCC) and graph recovery by instantaneous and lagged F1 scores.
Each ranges from 0 to 1, with higher values indicating better recovery.
Details of data generation and the evaluation protocol are provided in Appendix~\ref{app:synthetic_common_experimental_settings}.

We evaluate \methodTCL{} and \methodGCL{} under both transition-noise variation settings.
For discrete-variation data, both use the observed environment labels; for continuous-variation data, \methodGCL{} uses continuous time and \methodTCL{} uses its discretization into 256 bins.
The results in Table~\ref{tab:synthetic_main_bns} show that \methodTCL{} attains the highest mean score for every metric in both settings.
\methodGCL{} attains comparable MCC and lagged F1 scores, but its instantaneous F1 is lower on discrete-variation data.
The sample-size analysis in Appendix~\ref{app:synthetic_sample_size_sensitivity} indicates that \methodGCL{} is more sensitive to sample size than \methodTCL{} and that its lower instantaneous F1 on discrete-variation data is partly a finite-sample effect.
Reducing the depth of the observation function from three layers to one increases mean MCC for every evaluated method, with the largest increases for IDOL and LEAP, which remain below both proposed methods, as shown in Appendix~\ref{app:synthetic_mapping_complexity}.
Appendix~\ref{app:experimental-details} additionally reports sensitivity analyses of latent causal structure and latent dimension, along with ablations of the objective components.

\begin{table}[tbp]
    \centering
    \caption{Comparison of the proposed methods with baselines on synthetic data. The left and right blocks report results for discrete-variation data and continuous-variation data, respectively. Each entry is the mean $\pm$ standard deviation over 10 seeds; boldface indicates the best value in each column, and \texttt{--} indicates a metric not applicable to the method.}
    \label{tab:synthetic_main_bns}
    \small
    \setlength{\tabcolsep}{3pt}
    \resizebox{\textwidth}{!}{
    \begin{tabular}{l|ccc|ccc}
        \hline
        & \multicolumn{3}{c|}{Discrete variation} & \multicolumn{3}{c}{Continuous variation} \\
        \cline{2-7}
        Method & MCC & Inst. F1 & Lagged F1 & MCC & Inst. F1 & Lagged F1 \\
        \hline
        \methodTCL{} & $\mathbf{0.951\pm0.011}$ & $\mathbf{0.950\pm0.158}$ & $\mathbf{0.980\pm0.063}$ & $\mathbf{0.959\pm0.010}$ & $\mathbf{0.975\pm0.079}$ & $\mathbf{0.970\pm0.095}$ \\
        \methodGCL{} & $0.940\pm0.010$ & $0.700\pm0.158$ & $\mathbf{0.980\pm0.042}$ & $0.945\pm0.013$ & $\mathbf{0.975\pm0.079}$ & $0.950\pm0.071$ \\
        \hline
        IDOL & $0.468\pm0.046$ & $0.225\pm0.142$ & $0.480\pm0.079$ & $0.534\pm0.090$ & $0.225\pm0.142$ & $0.520\pm0.123$ \\
        G-CaRL & $0.764\pm0.065$ & \texttt{--} & $0.440\pm0.117$ & $0.762\pm0.047$ & \texttt{--} & $0.340\pm0.084$ \\
        CtrlNS & $0.432\pm0.071$ & \texttt{--} & $0.600\pm0.000$ & $0.398\pm0.102$ & \texttt{--} & $0.600\pm0.000$ \\
        LEAP & $0.373\pm0.107$ & \texttt{--} & $0.380\pm0.092$ & $0.365\pm0.086$ & \texttt{--} & $0.380\pm0.063$ \\
        iVAE & $0.230\pm0.106$ & \texttt{--} & \texttt{--} & $0.373\pm0.110$ & \texttt{--} & \texttt{--} \\
        $\beta$-VAE & $0.482\pm0.067$ & \texttt{--} & \texttt{--} & $0.473\pm0.075$ & \texttt{--} & \texttt{--} \\
        \hline
    \end{tabular}
    }
\end{table}

\subsection{Experiments on Real-world Data}
\label{sec:exp_real}

Because ground-truth latent states and causal structures are unavailable for the real-world datasets, we assess the utility of learned representations through two downstream forecasting tasks under held-out conditions: future-motion prediction at a walking speed and weather forecasting in a later period.
For each task, we construct discrete auxiliary labels and use \methodTCL{} with an observation reconstruction loss during pretraining; we then freeze each encoder and train a forecasting head.

\paragraph{Treadmill motion forecasting.}
Using the treadmill walking data of \citet{Riglet2024}, we evaluate future-motion forecasting at the Comfortable speed after pretraining and model selection on the Slow and Fast speeds.
Participants are split among the training, validation, and test sets without overlap; during pretraining, \methodname{} and LEAP use gait phase divided into 12 classes as the auxiliary variable.
With each encoder frozen, a forecasting head is trained and selected on the Slow and Fast data to predict the following 24 frames (1.92 seconds) from the latent representations of the preceding 12 frames (0.96 seconds).
The head is then evaluated without updates on the Comfortable condition, which is excluded from training and model selection.
We report mean marker error, defined as the three-dimensional Euclidean distance between predicted and true marker positions averaged over test windows and markers.
The results in Table~\ref{tab:real_world_forecasting} show that \methodname{} achieves the lowest error at every reported horizon, suggesting that the learned representation supports forecasting at a walking speed excluded from training and model selection.
Details of the experimental setup and results with standard deviations are provided in Appendix~\ref{app:supp_treadmill_motion_forecasting}.

\paragraph{Weather forecasting.}
Using the ERA5 global reanalysis~\citep{Hersbach2020}, we examine whether representations pretrained on historical weather observations support forecasting of relative humidity in a later period across eight regions.
We construct three-hourly observations from five cities in each region, using 2-m temperature, mean sea-level pressure, 10-m zonal and meridional wind components, and precipitation over the preceding three hours.
Relative humidity is excluded from representation pretraining and serves as the downstream prediction target.
For \methodname{} and LEAP, the auxiliary variable combines 12 months, four periods within each month, and eight UTC times at three-hour intervals, yielding 384 classes.
Each method is pretrained and selected using data from 1981--2013; with the selected encoder frozen, a forecasting head is trained on 2014--2016, selected using 2017 validation data, and evaluated on 2018--2020.
The head predicts relative humidity three hours ahead from the latent representations at the preceding eight time points.
Performance is evaluated using standardized mean squared error (MSE).
Among the compared representation-learning methods, \methodname{} achieves the lowest mean standardized MSE in five regions and the second-lowest in the remaining three, suggesting that the learned representations support forecasting in a later evaluation period across diverse regions.
Details of the experimental setup and results with standard deviations are provided in Appendix~\ref{app:supp_weather_forecasting}.

\begin{table*}[tbp]
    \centering
    \caption{Forecasting results on real-world data. Entries are means across runs; lower is better. Treadmill reports mean marker error in mm, and weather reports standardized MSE for three-hour-ahead relative-humidity forecasting. Weather columns are southeastern Australia (AU), Canada Great Lakes--St.~Lawrence (CAN), China (CN), France (FR), Germany (GER), Japan (JP), South Korea (KR), and US Atlantic (US). Boldface marks the best result in each column.}
    \label{tab:real_world_forecasting}
    \footnotesize
    \setlength{\tabcolsep}{2pt}
    \begin{tabular}{l|cccc|cccccccc}
        \hline
        & \multicolumn{4}{c|}{Treadmill: marker error [mm] $\downarrow$}
        & \multicolumn{8}{c}{Weather: standardized MSE $\downarrow$} \\
        Method & 0.48 s & 0.96 s & 1.44 s & 1.92 s & AU & CAN & CN & FR & GER & JP & KR & US \\
        \hline
        \methodname{}
        & $\mathbf{42.060}$ & $\mathbf{58.650}$ & $\mathbf{82.034}$ & $\mathbf{107.309}$
        & $\mathbf{0.352}$ & $0.481$ & $\mathbf{0.546}$ & $\mathbf{0.422}$ & $0.474$ & $\mathbf{0.414}$ & $\mathbf{0.291}$ & $0.427$ \\
        IDOL
        & $44.076$ & $62.400$ & $86.791$ & $115.146$
        & $0.359$ & $0.484$ & $0.555$ & $0.438$ & $\mathbf{0.469}$ & $0.440$ & $0.304$ & $0.440$ \\
        G-CaRL
        & $52.197$ & $75.790$ & $103.133$ & $135.169$
        & $0.649$ & $0.673$ & $0.664$ & $0.750$ & $0.754$ & $0.615$ & $0.535$ & $0.554$ \\
        CtrlNS
        & $43.593$ & $62.337$ & $86.040$ & $117.059$
        & $0.385$ & $0.522$ & $0.618$ & $0.459$ & $0.504$ & $0.501$ & $0.382$ & $0.462$ \\
        LEAP$^{*}$
        & $44.220$ & $63.431$ & $87.338$ & $113.210$
        & $0.396$ & $0.524$ & $0.591$ & $0.486$ & $0.566$ & $0.482$ & $0.327$ & $0.460$ \\
        $\beta$-VAE
        & $43.391$ & $63.305$ & $91.089$ & $123.960$
        & $0.357$ & $\mathbf{0.475}$ & $0.561$ & $0.445$ & $0.484$ & $0.426$ & $0.328$ & $\mathbf{0.416}$ \\
        \hline
    \end{tabular}

    \vspace{2pt}
    \parbox{0.96\linewidth}{\footnotesize $^{*}$For LEAP, pretraining on treadmill data was unstable for two seeds; results are based on eight runs.}
\end{table*}

\section{Discussion and Conclusion}
\label{sec:discussion-conclusion}

The following discussion compares our framework with related methods and considers the limitations of our analysis and experiments.
Our framework requires an observed auxiliary variable associated with sufficient variation in transition-noise distributions, but not the intervention-target labels required by iCITRIS~\citep{Lippe2023a}.
IIA-HMM~\citep{Morioka2021a} may offer a route to relaxing the requirement for an observed auxiliary variable by modeling nonstationarity with a hidden Markov model.
For latent-component identification, our analysis draws on the identification argument of IDOL~\citep{Li2025a} by combining variation across auxiliary-variable values and past states; transition-noise identification builds on IIA~\citep{Morioka2021a}.
The full-rank requirement in Condition~\ref{cond:instantaneous-variation} does not cover models with instantaneous relations when the transition noise can be expressed as a sum of functions of the current and past latent states, including linear transition models.
Consistent with this limitation, Appendix~\ref{app:synthetic_mapping_complexity} shows lower instantaneous F1 for linear transitions.
Finally, because ground-truth latent states and causal structures are unavailable for the real-world datasets, we evaluate predictive utility rather than recovery of these quantities.
Validating the learned states and causal relations on real-world data would require additional evidence, such as domain knowledge or interventional information.
\par
We established sufficient conditions for identifying latent states and their instantaneous and lagged causal structures from time-series observations, using an observed auxiliary variable associated with changes in transition-noise distributions.
Based on these results, we developed \methodname{}, a contrastive learning framework that accommodates discrete and continuous auxiliary variables.
Experiments demonstrated accurate recovery of latent states and their causal structures on synthetic data and the utility of the learned representations for downstream forecasting on real-world data.
Our results show how transition-noise nonstationarity associated with observed auxiliary variables can support CRL with both instantaneous and lagged causal relations.
\subsection*{AI use statement}

In this work, we used generative AI tools to assist with
(i) theoretical work,
including derivations and proof checking, as well as finding potentially relevant literature;
(ii) experimental work and implementation, including experimental design, implementation and review of models and code for synthetic-data generation and real-world-data preprocessing, and interpretation of results; and
(iii) manuscript preparation, including drafting, translation, and editing for clarity and readability.
We have reviewed all AI-assisted work and take responsibility for the final
content of this work.

\subsection*{Ethics statement}

This work involved no new recruitment of human participants or collection of
human-subject data. The motion-prediction experiment used the previously
published, publicly available motion-capture dataset of
\citet{Riglet2024}. We used only the released kinematic measurements and
pseudonymous subject identifiers and did not attempt to identify participants.

\subsection*{Reproducibility statement}

The assumptions and complete proofs underlying our identifiability results are
provided in Appendix~\ref{app:proofs}.
Appendix~\ref{app:learning-framework-details} specifies the model
parameterizations, training objectives, structural regularization, and
joint learning procedure. Appendices~\ref{app:experimental-details}
and~\ref{app:additional-results} document the data generation and preprocessing,
optimization settings, data splits, auxiliary variables, method-specific inputs, model selection,
evaluation protocols, sensitivity analyses, and ablations.

\section*{Acknowledgements}

This work was supported by JSPS KAKENHI Grant Numbers JP26H02498, JP24H02177, and JP26K02971.
It was also supported by JST PRESTO, Japan, Grant Number JPMJPR24KC, and JST BOOST Program, Japan, Grant Number JPMJBY24H4.

\bibliography{references}
\bibliographystyle{iclr2027_conference}

\appendix
\startcontents[appendices]
\section*{Appendix Contents}
\printcontents[appendices]{}{1}[2]{}
\section{Proofs}
\label{app:proofs}

\subsection{Identification of the Instantaneous Markov Network}
\label{app:instantaneous_markov_identification}

This section establishes an isomorphism between the true and learned instantaneous Markov networks.
The argument draws on the Markov-network identification step of IDOL~\citep{Li2025a}, using a variability condition that incorporates variation across both past states and values of the auxiliary variable.
Suppose that their latent representations are related by a $C^2$ diffeomorphism $H$ that is independent of time and the value of the auxiliary variable:
\begin{equation*}
\vz_t=H(\hat{\vz}_t),
\qquad
\vz_{t-1}=H(\hat{\vz}_{t-1}).
\end{equation*}
We use nonedges in the learned instantaneous network to derive constraints on $J_H$ and combine them with instantaneous edge-count minimality to establish an isomorphism between the true and learned instantaneous Markov networks.
Let $\mathcal M_z$ and $\mathcal M_{\hat z}$ denote the pairwise Markov networks over the components of $\vz_t$ and $\hat{\vz}_t$, respectively, defined by the instantaneous conditional-independence relations in the true and learned transition distributions, and write their edge sets as
\begin{equation*}
\Einst:=\mathcal E(\mathcal M_z),
\qquad
\widehatEinst:=\mathcal E(\mathcal M_{\hat z}).
\end{equation*}
For $i\neq j$ and $k\neq l$, their nonedges are characterized by
\begin{align*}
\{z_{t,i},z_{t,j}\}\notin \Einst
&\quad\Longleftrightarrow\quad
z_{t,i}\perp z_{t,j}
\mid \vz_{t,\setminus\{i,j\}},\vz_{t-1},u_t,\\
\{\hat z_{t,k},\hat z_{t,l}\}\notin \widehatEinst
&\quad\Longleftrightarrow\quad
\hat z_{t,k}\perp\hat z_{t,l}
\mid \hat{\vz}_{t,\setminus\{k,l\}},\hat{\vz}_{t-1},u_t.
\end{align*}
Here and below, we write each edge as an unordered pair of its latent-variable nodes.

A nonedge in the learned instantaneous network implies that the corresponding mixed derivative of the log-density ratio in the learned coordinates is zero.
After expansion through $H$, however, this condition is a linear combination with Jacobian products as coefficients.
These products need not vanish individually because nonzero terms may cancel.
To rule out such cancellation and additional learned edges, we impose sufficient variability and instantaneous edge-count minimality, respectively.

\begin{assumption}[Sufficient variability and instantaneous edge-count minimality]
\label{assumption:supp_instantaneous_identification}
The following conditions hold.

\appendixcondition{appcond:instantaneous-regularity}{Regularity.}
For every auxiliary-variable value $u_t^{(n)}$, the density $p(\vz_t\mid \vz_{t-1},u_t^{(n)})$ is positive on $\mathbb R^{d_z}\times\mathbb R^{d_z}$, and its logarithm is twice continuously differentiable in $(\vz_t,\vz_{t-1})$.

\medskip
\appendixcondition{appcond:instantaneous-variation}{Sufficient variability across auxiliary-variable values and past states.}
Choose a reference value $u_t^{(0)}$, let $N_{\mathrm{aux}}$ be the number of comparison values, and define, for $n\in\{1,\ldots,N_{\mathrm{aux}}\}$,
\begin{equation*}
 \Delta_n(\vz_t,\vz_{t-1})
 :=
 \log p(\vz_t\mid \vz_{t-1},u_t^{(n)})
 -
 \log p(\vz_t\mid \vz_{t-1},u_t^{(0)}).
\end{equation*}
Set $N_{\mathrm{inst}}:=2d_z+|\Einst|$ and define
\begin{equation*}
\begin{aligned}
\vw_n(\vz_t,\vz_{t-1}):={}&
\left(
\frac{\partial^2\Delta_n}{\partial z_{t,i}^2}
\right)_{i=1}^{d_z}
\oplus
\left(
\frac{\partial^2\Delta_n}{\partial z_{t,i}\partial z_{t,j}}
\right)_{\{z_{t,i},z_{t,j}\}\in \Einst}
\oplus
\left(
\frac{\partial\Delta_n}{\partial z_{t,i}}
\right)_{i=1}^{d_z}
\in\mathbb R^{N_{\mathrm{inst}}}.
\end{aligned}
\end{equation*}
For every $\vz_t$, suppose that there exist pairs of auxiliary-variable indices and past-state values
\begin{equation*}
\bigl(n_1,\bar{\vz}_{t-1}^{(1)}\bigr),
\ldots,
\bigl(n_{N_{\mathrm{inst}}},\bar{\vz}_{t-1}^{(N_{\mathrm{inst}})}\bigr)
\in
\{1,\ldots,N_{\mathrm{aux}}\}\times\mathbb R^{d_z}.
\end{equation*}
Here, the overbar indicates past-state evaluation values selected with $\vz_t$ fixed, rather than the random variable $\vz_{t-1}$; the indices $n_1,\ldots,n_{N_{\mathrm{inst}}}$ need not be distinct.
Using these pairs, define
\begin{equation*}
\mW(\vz_t)
:=
\left[
\vw_{n_1}(\vz_t,\bar{\vz}_{t-1}^{(1)}),
\ldots,
\vw_{n_{N_{\mathrm{inst}}}}(\vz_t,\bar{\vz}_{t-1}^{(N_{\mathrm{inst}})})
\right]
\in\mathbb R^{N_{\mathrm{inst}}\times N_{\mathrm{inst}}},
\end{equation*}
and assume that $\mW(\vz_t)$ has full rank.

\medskip
\appendixcondition{appcond:instantaneous-edge-minimality}{Instantaneous edge-count minimality.}
The learned instantaneous edge set satisfies
\begin{equation*}
|\widehatEinst|
\leq
|\Einst|.
\end{equation*}
\end{assumption}

We first fix two nonadjacent components in $\mathcal M_{\hat z}$ and derive Jacobian conditions that prevent a true component, or the endpoints of a true instantaneous edge, from depending simultaneously on these two learned components.

\begin{lemma}[Jacobian conditions implied by a learned instantaneous nonedge]
\label{lemma:instantaneous_jacobian_conditions}
Consider the learned representation above satisfying $\vz_t=H(\hat{\vz}_t)$, and suppose that \assumptionconditionref{assumption:supp_instantaneous_identification}{appcond:instantaneous-regularity}--\ref{appcond:instantaneous-variation} holds.
Fix distinct $k,l$ such that $\{\hat z_{t,k},\hat z_{t,l}\}\notin\widehatEinst$.
Then, at every $\hat{\vz}_t$ in the domain of $H$,
\begin{align}
\forall i,\quad
 \frac{\partial z_{t,i}}{\partial\hat z_{t,k}}
 \frac{\partial z_{t,i}}{\partial\hat z_{t,l}}
&=0,
 \label{eq:markov_network_node_identification}
 \\
\forall \{z_{t,i},z_{t,j}\}\in \Einst,\quad
 \frac{\partial z_{t,i}}{\partial\hat z_{t,k}}
 \frac{\partial z_{t,j}}{\partial\hat z_{t,l}}
&=0.
 \label{eq:markov_network_edge_identification}
\end{align}
\end{lemma}
The relation in~\eqref{eq:markov_network_node_identification} states that a true component cannot depend simultaneously on two nonadjacent components of $\mathcal M_{\hat z}$.
The relation in~\eqref{eq:markov_network_edge_identification} states that, for a true instantaneous edge, its two endpoints cannot separately depend on the two nonadjacent learned components.

\begin{proof}
Define the learned log-density ratio by
\begin{equation*}
\widehat\Delta_n(\hat{\vz}_t,\hat{\vz}_{t-1})
:=
\log p(\hat{\vz}_t\mid\hat{\vz}_{t-1},u_t^{(n)})
-
\log p(\hat{\vz}_t\mid\hat{\vz}_{t-1},u_t^{(0)}).
\end{equation*}
The density transformation induced by $\vz_t=H(\hat{\vz}_t)$ and $\vz_{t-1}=H(\hat{\vz}_{t-1})$ gives
\begin{align*}
p(\hat{\vz}_t,\hat{\vz}_{t-1},u_t)
={}&p(\vz_t,\vz_{t-1},u_t)
\left|\det J_H(\hat{\vz}_t)\right|
\left|\det J_H(\hat{\vz}_{t-1})\right|,\\
p(\hat{\vz}_{t-1},u_t)
={}&p(\vz_{t-1},u_t)
\left|\det J_H(\hat{\vz}_{t-1})\right|.
\end{align*}
Dividing the first equality by the second gives the conditional-density transformation.
Taking logarithms and subtracting the reference expression from each comparison expression cancels the Jacobian term because it does not depend on the auxiliary variable, yielding
\begin{equation}
\label{eq:a1_1_loglikelihood_difference}
\widehat\Delta_n(\hat{\vz}_t,\hat{\vz}_{t-1})
=
\Delta_n\bigl(H(\hat{\vz}_t),H(\hat{\vz}_{t-1})\bigr).
\end{equation}

Because $\{\hat z_{t,k},\hat z_{t,l}\}\notin\widehatEinst$, the pairwise Markov property gives
\begin{align*}
\log p(\hat{\vz}_t\mid\hat{\vz}_{t-1},u_t)
={}&
\log p(\hat z_{t,k}\mid
\hat{\vz}_{t,\setminus\{k,l\}},\hat{\vz}_{t-1},u_t)\\
&+
\log p(\hat z_{t,l}\mid
\hat{\vz}_{t,\setminus\{k,l\}},\hat{\vz}_{t-1},u_t)
+
\log p(\hat{\vz}_{t,\setminus\{k,l\}}\mid\hat{\vz}_{t-1},u_t).
\end{align*}
No term on the right-hand side depends on both $\hat z_{t,k}$ and $\hat z_{t,l}$, so its mixed derivative with respect to these variables is zero for every auxiliary-variable value.
Taking the difference between the comparison and reference values therefore gives $\partial^2\widehat\Delta_n/\partial\hat z_{t,k}\partial\hat z_{t,l}=0$.
Substituting this result into the left-hand side of~\eqref{eq:a1_1_loglikelihood_difference} and applying the chain rule to the right-hand side gives
\begin{align*}
0={}&
\sum_{a=1}^{d_z}
\frac{\partial^2\Delta_n}{\partial z_{t,a}^2}
\frac{\partial z_{t,a}}{\partial\hat z_{t,k}}
\frac{\partial z_{t,a}}{\partial\hat z_{t,l}}
+
\sum_{\{z_{t,a},z_{t,b}\}\in \Einst}
\frac{\partial^2\Delta_n}{\partial z_{t,a}\partial z_{t,b}}
\left(
\frac{\partial z_{t,a}}{\partial\hat z_{t,k}}
\frac{\partial z_{t,b}}{\partial\hat z_{t,l}}
+
\frac{\partial z_{t,b}}{\partial\hat z_{t,k}}
\frac{\partial z_{t,a}}{\partial\hat z_{t,l}}
\right)\\
&+
\sum_{a=1}^{d_z}
\frac{\partial\Delta_n}{\partial z_{t,a}}
\frac{\partial^2z_{t,a}}
{\partial\hat z_{t,k}\partial\hat z_{t,l}}.
\end{align*}
Here, all derivatives of $\Delta_n$ are evaluated at $(H(\hat{\vz}_t),H(\hat{\vz}_{t-1}))$.
In the second sum, only pairs in $\Einst$ remain because, by the same pairwise Markov property, the mixed derivative is zero for every true instantaneous nonedge.

Denote the coefficients of the derivatives of $\Delta_n$ in this expansion by
\begin{align*}
\kappa^{\mathrm{diag}}_a
&:=
\frac{\partial z_{t,a}}{\partial\hat z_{t,k}}
\frac{\partial z_{t,a}}{\partial\hat z_{t,l}},
&
\kappa^{\mathrm{edge}}_{\{a,b\}}
&:=
\frac{\partial z_{t,a}}{\partial\hat z_{t,k}}
\frac{\partial z_{t,b}}{\partial\hat z_{t,l}}
+
\frac{\partial z_{t,b}}{\partial\hat z_{t,k}}
\frac{\partial z_{t,a}}{\partial\hat z_{t,l}},
&
\kappa^{\mathrm{grad}}_a
&:=
\frac{\partial^2z_{t,a}}
{\partial\hat z_{t,k}\partial\hat z_{t,l}}.
\end{align*}
Stack them in the same order as $\vw_n$ into $\kappa\in\mathbb R^{N_{\mathrm{inst}}}$.
Then $\vw_n(\vz_t,\vz_{t-1})^\top\kappa=0$ for all $n$ and $\vz_{t-1}$.
For fixed $\vz_t$, $\kappa$ is independent of the past state and auxiliary-variable value.
Stacking the equations for the pairs selected in \assumptionconditionref{assumption:supp_instantaneous_identification}{appcond:instantaneous-variation} gives
\begin{equation*}
\mW(\vz_t)^\top\kappa=0,
\qquad
\operatorname{rank}\mW(\vz_t)=N_{\mathrm{inst}}
\quad\Longrightarrow\quad
\kappa=0.
\end{equation*}
In particular, $\kappa_i^{\mathrm{diag}}=0$, which proves~\eqref{eq:markov_network_node_identification}.

For $\{z_{t,i},z_{t,j}\}\in \Einst$, let
\begin{equation*}
A_{ij}:=
\frac{\partial z_{t,i}}{\partial\hat z_{t,k}}
\frac{\partial z_{t,j}}{\partial\hat z_{t,l}},
\qquad
B_{ij}:=
\frac{\partial z_{t,j}}{\partial\hat z_{t,k}}
\frac{\partial z_{t,i}}{\partial\hat z_{t,l}}.
\end{equation*}
The equalities $\kappa^{\mathrm{edge}}_{\{i,j\}}=0$ and~\eqref{eq:markov_network_node_identification} imply
\begin{equation*}
A_{ij}+B_{ij}=0,
\qquad
A_{ij}B_{ij}=0.
\end{equation*}
Thus, $A_{ij}=B_{ij}=0$, proving~\eqref{eq:markov_network_edge_identification}.
\end{proof}

We next use invertibility of the Jacobian of $H$ to match every true latent component locally to a distinct learned coordinate.

\begin{lemma}[Local component matching from an invertible Jacobian]
\label{lemma:local_coordinate_correspondence}
For every $\hat{\vz}_t$, there exists a permutation $\pi_{\hat{\vz}_t}$ such that
\begin{equation*}
\frac{\partial z_{t,i}}
{\partial\hat z_{t,\pi_{\hat{\vz}_t}(i)}}\ne0
\qquad(\forall i).
\end{equation*}
\end{lemma}

\begin{proof}
Since $H$ is a diffeomorphism, $\det J_H(\hat{\vz}_t)\ne0$ at every $\hat{\vz}_t$.
% Expanding the determinant as
The determinant expands as
\begin{equation*}
\det J_H(\hat{\vz}_t)
=
\sum_{\pi\in\mathfrak S_{d_z}}
\operatorname{sgn}(\pi)
\prod_{i=1}^{d_z}
\frac{\partial z_{t,i}}
{\partial\hat z_{t,\pi(i)}},
\end{equation*}
where $\mathfrak S_{d_z}$ denotes the set of all permutations of $\{1,\ldots,d_z\}$.
Since the determinant is nonzero, at least one product in the sum is nonzero.
Choosing its associated permutation as $\pi_{\hat{\vz}_t}$ proves the claim.
\end{proof}

The Jacobian conditions and local component matching now show that every true instantaneous edge is preserved in the learned representation; instantaneous edge-count minimality then yields network isomorphism.

\begin{theorem}[Identification of the instantaneous Markov network]
\label{theorem:local_markov_network_isomorphism}
Fix $\hat{\vz}_t$ and take a local permutation $\pi_{\hat{\vz}_t}$ from Lemma~\ref{lemma:local_coordinate_correspondence}.
Under Assumption~\ref{assumption:supp_instantaneous_identification}, for every $i\ne j$,
\begin{equation*}
\{z_{t,i},z_{t,j}\}\in \Einst
\quad\Longleftrightarrow\quad
\{\hat z_{t,\pi_{\hat{\vz}_t}(i)},
\hat z_{t,\pi_{\hat{\vz}_t}(j)}\}\in\widehatEinst.
\end{equation*}
Thus, $\mathcal M_z$ and $\mathcal M_{\hat z}$ are isomorphic under $\pi_{\hat{\vz}_t}$.
At this stage, $\pi_{\hat{\vz}_t}$ may depend on the fixed point $\hat{\vz}_t$.
\end{theorem}

\begin{proof}
Take any $\{z_{t,i},z_{t,j}\}\in \Einst$.
If its image were a learned nonedge, Lemma~\ref{lemma:instantaneous_jacobian_conditions} would give
\begin{equation*}
\frac{\partial z_{t,i}}
{\partial\hat z_{t,\pi_{\hat{\vz}_t}(i)}}
\frac{\partial z_{t,j}}
{\partial\hat z_{t,\pi_{\hat{\vz}_t}(j)}}
=0,
\end{equation*}
contradicting Lemma~\ref{lemma:local_coordinate_correspondence}.
Thus, every true instantaneous edge maps injectively to a learned instantaneous edge, and
\begin{equation*}
|\Einst|
\leq
|\widehatEinst|,
\qquad
|\widehatEinst|
\leq
|\Einst|.
\end{equation*}
The first inequality follows from edge preservation and injectivity of $\pi_{\hat{\vz}_t}$, and the second from \assumptionconditionref{assumption:supp_instantaneous_identification}{appcond:instantaneous-edge-minimality}.
The edge sets therefore have equal cardinality, and the images of the true edges exhaust the learned edge set, proving the isomorphism.
\end{proof}

The following remark examines the full-rank condition in \assumptionconditionref{assumption:supp_instantaneous_identification}{appcond:instantaneous-variation} for additively separable inverse transitions.

\begin{remark}[Rank bound for an additively separable inverse transition function]
\label{remark:supp_instantaneous_additive_separation_rank_bound}
Suppose that the inverse transition function satisfies
$\epsilon_{t,i}=r_i(\vz_t,\vz_{t-1})=a_i(\vz_t)+b_i(\vz_{t-1})$,
where $a_i$ depends only on $\vz_t$ and $b_i$ only on $\vz_{t-1}$.
Define the component-wise log-density ratio of the transition noise by
\begin{equation*}
\delta_{i,n}(\epsilon_{t,i})
:=
\log p_i(\epsilon_{t,i}\mid u_t^{(n)})
-
\log p_i(\epsilon_{t,i}\mid u_t^{(0)}).
\end{equation*} 
Conditioning on $\vz_{t-1}$ and $u_t$, component-wise independence of the transition noise and the change of variables
$\veps_t=r(\vz_t,\vz_{t-1})$ give
\begin{equation*}
p(\vz_t\mid \vz_{t-1},u_t)
=
\prod_{i=1}^{d_z}
p_i\bigl(r_i(\vz_t,\vz_{t-1})\mid u_t\bigr)
\left|
\det\frac{\partial r(\vz_t,\vz_{t-1})}{\partial \vz_t}
\right|.
\end{equation*}
Because $r$ is independent of the auxiliary-variable value, its Jacobian term cancels in the log-density ratio, yielding
\begin{equation*}
\Delta_n(\vz_t,\vz_{t-1})
=
\sum_{i=1}^{d_z}
\delta_{i,n}\bigl(r_i(\vz_t,\vz_{t-1})\bigr).
\end{equation*}
For all $j,k\in\{1,\ldots,d_z\}$, evaluating the derivatives of $\delta_{i,n}$ at $\epsilon_{t,i}=r_i(\vz_t,\vz_{t-1})$ gives
\begin{align*}
\frac{\partial\Delta_n}{\partial z_{t,j}}
&=
\sum_{i=1}^{d_z}
\frac{\partial\delta_{i,n}}{\partial\epsilon_{t,i}}
\frac{\partial a_i}{\partial z_{t,j}},\\
\frac{\partial^2\Delta_n}{\partial z_{t,j}\partial z_{t,k}}
&=
\sum_{i=1}^{d_z}
\left\{
\frac{\partial^2\delta_{i,n}}{\partial\epsilon_{t,i}^2}
\frac{\partial a_i}{\partial z_{t,j}}
\frac{\partial a_i}{\partial z_{t,k}}
+
\frac{\partial\delta_{i,n}}{\partial\epsilon_{t,i}}
\frac{\partial^2a_i}{\partial z_{t,j}\partial z_{t,k}}
\right\}.
\end{align*}
For fixed $\vz_t$, all derivatives of $a_i$ are fixed, and variation with the auxiliary-variable value and past state appears only through the two scalar coefficients $\partial\delta_{i,n}/\partial\epsilon_{t,i}$ and $\partial^2\delta_{i,n}/\partial\epsilon_{t,i}^2$ for each $i$.
Thus, every $\vw_n$ is a linear combination of at most $2d_z$ fixed directions, and
\begin{equation*}
\operatorname{rank}\mW(\vz_t)\leq2d_z.
\end{equation*}
If $|\Einst|>0$, then $N_{\mathrm{inst}}=2d_z+|\Einst|>2d_z$, so the required full-rank condition cannot hold.
A linear inverse transition function is a special case of this additively separable form, so the theorem does not cover such models when instantaneous edges are present.

Mixed derivatives with respect to a current and past component, however, may retain terms such as
\begin{equation*}
\frac{\partial^2\delta_{i,n}}{\partial\epsilon_{t,i}^2}
\frac{\partial a_i}{\partial z_{t,j}}
\frac{\partial b_i}{\partial z_{t-1,m}}.
\end{equation*}
Therefore, when the transition-noise variance changes with the auxiliary-variable value, cross-time sufficient variability may hold even for an additively separable model.
\end{remark}

\subsection{Identification of Latent Components}
\label{app:latent_component_identification}

This section builds on the local isomorphism between the instantaneous Markov networks established in Appendix~\ref{app:instantaneous_markov_identification} to identify the latent components.
We first show that any remaining component mixing is restricted to intimate neighbors in the instantaneous Markov network.
When such neighbors exist, we eliminate the remaining mixing using cross-time relations and their variation across auxiliary-variable values.
Finally, we show that the resulting local permutations reduce to a single permutation independent of state and time, yielding component-wise identification.
Our argument draws on the intimate-neighbor proof strategy of IDOL~\citep{Li2025a}, while treating the instantaneous and cross-time relations separately and using changes across values of the auxiliary variable for the latter.

Throughout this section, suppose that the true and learned representations are related by the same $C^2$ diffeomorphism $H$, independent of time and the value of the auxiliary variable, and that Assumption~\ref{assumption:supp_instantaneous_identification} holds.

\subsubsection{Definitions and Notation}

To describe component mixing that cannot be ruled out by the isomorphism of the instantaneous Markov networks alone, we define neighborhoods and intimate neighbors in $\mathcal M_z$.
The neighborhood of $z_{t,i}$ in $\mathcal M_z$ is
\begin{equation*}
\mathcal N_{\mathcal M_z}(z_{t,i})
:=
\left\{
z_{t,j}
\;\middle|\;
\{z_{t,i},z_{t,j}\}\in \Einst
\right\}.
\end{equation*}

\begin{definition}[Intimate neighbor~\citep{Zhang2024a}]
In the pairwise Markov network $\mathcal M_z$, the intimate-neighbor set of each node $z_{t,i}$ is
\begin{equation*}
\Psi_{\mathcal M_z}(z_{t,i})
:=
\left\{
z_{t,j}
\;\middle|\;
z_{t,j}\in\mathcal N_{\mathcal M_z}(z_{t,i})
\ \text{and}\
\mathcal N_{\mathcal M_z}(z_{t,i})\setminus\{z_{t,j}\}
\subseteq
\mathcal N_{\mathcal M_z}(z_{t,j})\setminus\{z_{t,i}\}
\right\}.
\end{equation*}
Thus, $z_{t,j}$ is an intimate neighbor of $z_{t,i}$ if they are adjacent in $\mathcal M_z$ and every neighbor of $z_{t,i}$ other than $z_{t,j}$ is also a neighbor of $z_{t,j}$ other than $z_{t,i}$.
This relation is generally asymmetric: $z_{t,j}\in\Psi_{\mathcal M_z}(z_{t,i})$ does not imply $z_{t,i}\in\Psi_{\mathcal M_z}(z_{t,j})$.
\end{definition}

Even after the instantaneous Markov networks have been shown to be isomorphic, mixing may remain between components corresponding to intimate neighbors.
We express this possibility using off-diagonal entries of the Jacobian of the mapping $H:\hat{\vz}_t\mapsto \vz_t$.

\begin{definition}[Residual mixing]
\label{definition:residual_mixing}
Fix a point $\hat{\vz}_t$ and choose the local permutation $\pi_{\hat{\vz}_t}$ given by Lemma~\ref{lemma:local_coordinate_correspondence}.
For $i\neq j$, we say that residual mixing remains between $z_{t,i}$ and $\hat z_{t,\pi_{\hat{\vz}_t}(j)}$ under $H$ if
\begin{equation*}
\frac{\partial z_{t,i}}
{\partial\hat z_{t,\pi_{\hat{\vz}_t}(j)}}
\neq0.
\end{equation*}
\end{definition}

The following lemma shows that residual mixing is restricted to components corresponding to intimate neighbors.

\begin{lemma}[Residual mixing is restricted to intimate neighbors]
\label{lemma:residual_mixing_is_intimate}
Fix a point $\hat{\vz}_t$ and choose the local permutation $\pi_{\hat{\vz}_t}$ given by Lemma~\ref{lemma:local_coordinate_correspondence}.
For every $i\ne j$,
\begin{equation*}
\frac{\partial z_{t,i}}
{\partial\hat z_{t,\pi_{\hat{\vz}_t}(j)}}\ne0
\quad\Longrightarrow\quad
z_{t,j}\in\Psi_{\mathcal M_z}(z_{t,i}).
\end{equation*}
\end{lemma}

\begin{proof}
Suppose, for contradiction, that residual mixing occurs, so that
\begin{equation}
\label{eq:residual_mixing}
\frac{\partial z_{t,i}}
{\partial\hat z_{t,\pi_{\hat{\vz}_t}(j)}}\ne0,
\end{equation}
while $z_{t,j}\notin\Psi_{\mathcal M_z}(z_{t,i})$.

First consider the case $\{z_{t,i},z_{t,j}\}\notin \Einst$.
Theorem~\ref{theorem:local_markov_network_isomorphism} implies
\begin{equation*}
\{\hat z_{t,\pi_{\hat{\vz}_t}(i)},
\hat z_{t,\pi_{\hat{\vz}_t}(j)}\}
\notin\widehatEinst.
\end{equation*}
The relation in~\eqref{eq:markov_network_node_identification} from Lemma~\ref{lemma:instantaneous_jacobian_conditions} therefore gives
\begin{equation*}
\frac{\partial z_{t,i}}
{\partial\hat z_{t,\pi_{\hat{\vz}_t}(i)}}
\frac{\partial z_{t,i}}
{\partial\hat z_{t,\pi_{\hat{\vz}_t}(j)}}=0.
\end{equation*}
The two factors are nonzero by Lemma~\ref{lemma:local_coordinate_correspondence} and~\eqref{eq:residual_mixing}, respectively, yielding a contradiction.

Next consider $\{z_{t,i},z_{t,j}\}\in \Einst$.
Because $z_{t,j}$ is not an intimate neighbor of $z_{t,i}$, there exists an index $k$ such that
\begin{equation*}
\{z_{t,i},z_{t,k}\}\in \Einst,
\qquad
\{z_{t,j},z_{t,k}\}\notin \Einst.
\end{equation*}
Theorem~\ref{theorem:local_markov_network_isomorphism} then implies
\begin{equation*}
\{\hat z_{t,\pi_{\hat{\vz}_t}(j)},
\hat z_{t,\pi_{\hat{\vz}_t}(k)}\}
\notin\widehatEinst.
\end{equation*}
Applying~\eqref{eq:markov_network_edge_identification} in Lemma~\ref{lemma:instantaneous_jacobian_conditions} to the true edge $\{z_{t,i},z_{t,k}\}$ yields
\begin{equation*}
\frac{\partial z_{t,i}}
{\partial\hat z_{t,\pi_{\hat{\vz}_t}(j)}}
\frac{\partial z_{t,k}}
{\partial\hat z_{t,\pi_{\hat{\vz}_t}(k)}}=0.
\end{equation*}
Again, the two factors are nonzero by~\eqref{eq:residual_mixing} and Lemma~\ref{lemma:local_coordinate_correspondence}, respectively, which is a contradiction.
Thus, whether or not $z_{t,i}$ and $z_{t,j}$ are adjacent in $\mathcal M_z$, assuming $z_{t,j}\notin\Psi_{\mathcal M_z}(z_{t,i})$ leads to a contradiction.
Hence $z_{t,j}\in\Psi_{\mathcal M_z}(z_{t,i})$.
\end{proof}

\subsubsection{Empty Intimate-Neighbor Sets}
\label{subsubsec:element_wise_identifiability_empty_case}

When every intimate-neighbor set is empty, Lemma~\ref{lemma:residual_mixing_is_intimate} immediately rules out all residual mixing.
Specifically, suppose
\begin{equation*}
\Psi_{\mathcal M_z}(z_{t,i})=\emptyset
\qquad(\forall i),
\end{equation*}
and fix any point $\hat{\vz}_t$ and the local permutation $\pi_{\hat{\vz}_t}$ given by Lemma~\ref{lemma:local_coordinate_correspondence}.
By Lemma~\ref{lemma:residual_mixing_is_intimate}, a nonzero off-diagonal entry of the locally permuted Jacobian would imply the existence of an intimate neighbor, contradicting the assumption.
Therefore,
\begin{equation*}
\frac{\partial z_{t,i}}
{\partial\hat z_{t,\pi_{\hat{\vz}_t}(j)}}
=0
\qquad(\forall i,\ \forall j\ne i).
\end{equation*}
After its columns are arranged according to $\pi_{\hat{\vz}_t}$, the local Jacobian is thus diagonal with nonzero diagonal entries.

\subsubsection{Nonempty Intimate-Neighbor Sets: Using Cross-Time Relations}
\label{subsubsec:element_wise_identifiability_non_empty_time_relation}

When intimate neighbors remain as candidates for residual mixing after considering the local structure of $\mathcal M_z$, we use conditional-dependence relations with the past and future time steps to eliminate residual mixing.
We first define the cross-time relations and their changes across auxiliary-variable values, derive Jacobian conditions from learned cross-time nonedges, and then combine these conditions with cross-time edge-count minimality to match the true and learned cross-time relations.

\begin{definition}[Cross-time conditional-dependence relations]
Within this section, combine the latent variables at two time steps as
\begin{equation*}
\vc_t:=(\vz_{t-1},\vz_t),
\qquad
\hat{\vc}_t:=(\hat{\vz}_{t-1},\hat{\vz}_t).
\end{equation*}
Let $\mathcal M_{\vc_t}$ be the pairwise Markov network for the conditional joint distribution $p(\vz_{t-1},\vz_t\mid u_t)$, and define $\mathcal M_{\hat{\vc}_t}$ analogously for $p(\hat{\vz}_{t-1},\hat{\vz}_t\mid u_t)$.
The current--current part of $\mathcal M_{\vc_t}$ represents the instantaneous conditional-independence relations given the past state and therefore coincides with $\mathcal M_z$ and its edge set $\Einst$ defined in Appendix~\ref{app:instantaneous_markov_identification}.
Here, we use only the past--current part of these Markov networks.

Define the relations between the past and current time steps by
\begin{equation*}
\Elag^-
:=
\left\{
\{z_{t-1,a},z_{t,i}\}
\;\middle|\;
z_{t-1,a}\not\!\perp z_{t,i}
\mid \vz_{t-1,\setminus\{a\}},\vz_{t,\setminus\{i\}},u_t
\right\}.
\end{equation*}
For the learned representation, define
\begin{equation*}
\widehatElag^-
:=
\left\{
\{\hat z_{t-1,a},\hat z_{t,i}\}
\;\middle|\;
\hat z_{t-1,a}\not\!\perp\hat z_{t,i}
\mid
\hat{\vz}_{t-1,\setminus\{a\}},
\hat{\vz}_{t,\setminus\{i\}},u_t
\right\}.
\end{equation*}
Thus, $\Elag^-$ and $\widehatElag^-$ are the past--current parts of $\mathcal M_{\vc_t}$ and $\mathcal M_{\hat{\vc}_t}$, respectively.
They correspond to $\Elag$ and $\widehatElag$ in the main text.
The future-directed sets $\Elag^+$ and $\widehatElag^+$ are obtained by applying the same definitions to $\vc_{t+1}=(\vz_t,\vz_{t+1})$ and $\hat{\vc}_{t+1}=(\hat{\vz}_t,\hat{\vz}_{t+1})$.
\end{definition}

For a reference value $u_t^{(0)}$ and comparison value $u_t^{(n)}$, define the log-density ratio between the corresponding transition densities by
\begin{equation*}
\Delta_n^-(\vz_t,\vz_{t-1})
:=
\log p(\vz_t\mid \vz_{t-1},u_t^{(n)})
-
\log p(\vz_t\mid \vz_{t-1},u_t^{(0)}).
\end{equation*}
For each current--past component pair $(i,a)$, collect these mixed derivatives over the comparison values as
\begin{equation*}
\vv^-_{ia}(\vz_t,\vz_{t-1})
:=
\left(
\frac{\partial^2\Delta_n^-(\vz_t,\vz_{t-1})}
{\partial z_{t,i}\partial z_{t-1,a}}
\right)_{n=1}^{N_{\mathrm{aux}}}
\in\mathbb R^{N_{\mathrm{aux}}}.
\end{equation*}
The quantities $\Delta_n$ and $\vv_{ia}$ in the main text correspond to $\Delta_n^-$ and $\vv^-_{ia}$, respectively.
Define $\widehat\Delta_n^-$ and $\hat{\vv}^-_{ia}$ analogously for the learned representation.
Analogously, the future-directed quantities $\Delta_n^+$, $\vv^+_{bi}$, and their learned counterparts are obtained by shifting these definitions forward by one time step.
The following lemma relates cross-time nonedges to zeros of the corresponding derivative vectors.

\begin{lemma}[A cross-time nonedge implies zero mixed derivatives of the log-density ratio]
\label{lemma:time_nonedge_zero_variation}
For the true representation, the past- and future-directed statements are, respectively,
\begin{align*}
\{z_{t-1,a},z_{t,i}\}\notin \Elag^-
&\Longrightarrow \vv^-_{ia}(\vz_t,\vz_{t-1})=0,\\
\{z_{t,i},z_{t+1,b}\}\notin \Elag^+
&\Longrightarrow \vv^+_{bi}(\vz_{t+1},\vz_t)=0.
\end{align*}
For the learned representation, the corresponding statements are
\begin{align*}
\{\hat z_{t-1,m},\hat z_{t,k}\}\notin \widehatElag^-
&\Longrightarrow \hat{\vv}^-_{km}(\hat{\vz}_t,\hat{\vz}_{t-1})=0,\\
\{\hat z_{t,k},\hat z_{t+1,m}\}\notin \widehatElag^+
&\Longrightarrow \hat{\vv}^+_{mk}(\hat{\vz}_{t+1},\hat{\vz}_t)=0.
\end{align*}
\end{lemma}

\begin{proof}
We first consider the true distribution in the past case.
Suppose $\{z_{t-1,a},z_{t,i}\}\notin \Elag^-$.
By the definition of $\Elag^-$, for every auxiliary-variable value $u_t$,
\begin{equation*}
z_{t-1,a}\perp z_{t,i}
\mid \vz_{t-1,\setminus\{a\}},\vz_{t,\setminus\{i\}},u_t.
\end{equation*}
The transition density therefore factorizes as
\begin{align*}
p(\vz_t\mid \vz_{t-1},u_t)
={}&
p\bigl(z_{t,i}\mid \vz_{t,\setminus\{i\}},\vz_{t-1,\setminus\{a\}},u_t\bigr)
p(\vz_{t,\setminus\{i\}}\mid \vz_{t-1},u_t).
\end{align*}
The first factor does not depend on $z_{t-1,a}$, and the second does not depend on $z_{t,i}$.
Because the transition density is positive, taking its logarithm gives
\begin{equation*}
\frac{\partial^2}{\partial z_{t,i}\partial z_{t-1,a}}
\log p(\vz_t\mid \vz_{t-1},u_t)=0
\end{equation*}
for every auxiliary-variable value.
Taking the difference between each comparison value and the reference value therefore makes every entry of $\vv^-_{ia}(\vz_t,\vz_{t-1})$ zero, and hence
\begin{equation*}
\vv^-_{ia}(\vz_t,\vz_{t-1})=0.
\end{equation*}

For the learned representation, the same conclusion follows directly from the definition of a nonedge in the learned Markov network.
In particular, if
\begin{equation*}
\{\hat z_{t-1,m},\hat z_{t,k}\}\notin\widehatElag^-,
\end{equation*}
then, for every $u_t$,
\begin{equation*}
\hat z_{t-1,m}\perp\hat z_{t,k}
\mid
\hat{\vz}_{t-1,\setminus\{m\}},
\hat{\vz}_{t,\setminus\{k\}},u_t.
\end{equation*}
The same factorization argument gives
\begin{equation*}
\frac{\partial^2}
{\partial\hat z_{t,k}\partial\hat z_{t-1,m}}
\log p(\hat{\vz}_t\mid\hat{\vz}_{t-1},u_t)=0,
\end{equation*}
and hence $\hat{\vv}^-_{km}(\hat{\vz}_t,\hat{\vz}_{t-1})=0$.
Shifting time forward by one step proves the statements for the future case for both representations.
\end{proof}

The Jacobian conditions in Appendix~\ref{app:instantaneous_markov_identification} restrict which true components can have nonzero derivatives with respect to each learned coordinate according to the local structure of $\mathcal M_z$.
To express this restriction at each time step, let $[d_z]:=\{1,\ldots,d_z\}$ and, for each time $s$ and index $a\in[d_z]$, define the candidate component index set
\begin{equation*}
\mathcal C_s(a)
:=
\left\{
i\in[d_z]
\;\middle|\;
i=a
\ \text{or}\
z_{s,a}\in\Psi_{\mathcal M_z}(z_{s,i})
\right\}.
\end{equation*}
By Lemma~\ref{lemma:residual_mixing_is_intimate},
\begin{equation*}
\frac{\partial z_{t,i}}
{\partial\hat z_{t,\pi_{\hat{\vz}_t}(a)}}\ne0
\quad\Longrightarrow\quad
i\in\mathcal C_t(a),
\end{equation*}
and the analogous implication holds after shifting time by one step.
Thus, for the learned current and past coordinates indexed by $j$ and $a$, respectively, only true components indexed by $\mathcal C_t(j)$ and $\mathcal C_{t-1}(a)$ can have nonzero derivatives.

Among the index pairs in $\mathcal C_t(j)\times\mathcal C_{t-1}(a)$, only those corresponding to true cross-time relations need to be retained, because Lemma~\ref{lemma:time_nonedge_zero_variation} makes the derivative vector zero for every cross-time nonedge.
We therefore define the past-directed candidate block as the following set of index pairs:
\begin{equation*}
\mathcal B^-(j,a)
:=
\left\{
(k,m)\in[d_z]^2
\;\middle|\;
k\in\mathcal C_t(j),\ m\in\mathcal C_{t-1}(a),\
\{z_{t-1,m},z_{t,k}\}\in \Elag^-
\right\}.
\end{equation*}
In the future direction, define
\begin{equation*}
\mathcal B^+(j,b)
:=
\left\{
(k,m)\in[d_z]^2
\;\middle|\;
k\in\mathcal C_t(j),\ m\in\mathcal C_{t+1}(b),\
\{z_{t,k},z_{t+1,m}\}\in \Elag^+
\right\}.
\end{equation*}
Here, $\mathcal B^-(j,a)$ contains index pairs, whereas edges in $\Elag^-$ are written as unordered pairs of the corresponding latent-variable nodes.
Using these candidate blocks, we next express each learned cross-time derivative vector as a linear combination of the corresponding true derivative vectors.

\begin{lemma}[Transformation of cross-time derivatives over local candidate blocks]
\label{lemma:time_variation_transformation}
Fix $(\hat{\vz}_t,\hat{\vz}_{t-1})$ and choose the corresponding local permutations $\pi_{\hat{\vz}_t}$ and $\pi_{\hat{\vz}_{t-1}}$ from Lemma~\ref{lemma:local_coordinate_correspondence}.
For any $j,a\in[d_z]$,
\begin{equation*}
\hat{\vv}^-_{\pi_{\hat{\vz}_t}(j),\pi_{\hat{\vz}_{t-1}}(a)}
(\hat{\vz}_t,\hat{\vz}_{t-1})
=
\sum_{(k,m)\in\mathcal B^-(j,a)}
\chi_{km}\vv^-_{km}(\vz_t,\vz_{t-1}),
\end{equation*}
where $\vz_t=H(\hat{\vz}_t)$, $\vz_{t-1}=H(\hat{\vz}_{t-1})$, and
\begin{equation*}
\chi_{km}
:=
\frac{\partial z_{t,k}}
{\partial\hat z_{t,\pi_{\hat{\vz}_t}(j)}}
\frac{\partial z_{t-1,m}}
{\partial\hat z_{t-1,\pi_{\hat{\vz}_{t-1}}(a)}}.
\end{equation*}
The analogous equality holds in the future direction after shifting time by one step.
\end{lemma}

\begin{proof}
We prove the past case.
Transforming the conditional density under $\vz_s=H(\hat{\vz}_s)$ gives
\begin{equation*}
p(\hat{\vz}_t\mid\hat{\vz}_{t-1},u_t)
=
p\bigl(H(\hat{\vz}_t)\mid H(\hat{\vz}_{t-1}),u_t\bigr)
\left|\det J_H(\hat{\vz}_t)\right|.
\end{equation*}
Applying this equality under the reference and comparison values and taking their difference cancels $\log|\det J_H(\hat{\vz}_t)|$ because $H$ is independent of the auxiliary variable.
Thus, for every $n$,
\begin{equation*}
\widehat\Delta_n^-(\hat{\vz}_t,\hat{\vz}_{t-1})
=
\Delta_n^-\bigl(H(\hat{\vz}_t),H(\hat{\vz}_{t-1})\bigr).
\end{equation*}
Differentiate this equality once with respect to the learned current coordinate $\hat z_{t,\pi_{\hat{\vz}_t}(j)}$ and once with respect to the learned past coordinate $\hat z_{t-1,\pi_{\hat{\vz}_{t-1}}(a)}$.
Since the transformation acts separately at each time step, the chain rule leaves only products of first-order Jacobian entries:
\begin{equation*}
\frac{\partial^2\widehat\Delta_n^-(\hat{\vz}_t,\hat{\vz}_{t-1})}
{\partial\hat z_{t,\pi_{\hat{\vz}_t}(j)}
 \partial\hat z_{t-1,\pi_{\hat{\vz}_{t-1}}(a)}}
=
\sum_{k=1}^{d_z}\sum_{m=1}^{d_z}
\chi_{km}
\frac{\partial^2\Delta_n^-(\vz_t,\vz_{t-1})}
{\partial z_{t,k}\partial z_{t-1,m}}.
\end{equation*}
Collecting this equality over $n=1,\ldots,N_{\mathrm{aux}}$ gives the corresponding equality for $\vv^-$.
By Lemma~\ref{lemma:residual_mixing_is_intimate}, $\chi_{km}$ can be nonzero only if $k\in\mathcal C_t(j)$ and $m\in\mathcal C_{t-1}(a)$.
Moreover, Lemma~\ref{lemma:time_nonedge_zero_variation} gives $\vv^-_{km}=0$ whenever $\{z_{t-1,m},z_{t,k}\}\notin \Elag^-$.
The sum can therefore be restricted to $\mathcal B^-(j,a)$.
The equality for the future case follows by shifting time forward by one step.
\end{proof}

Figure~\ref{fig:intimate_neighbor_example} illustrates how differences in cross-time relations can resolve intimate-neighbor ambiguities that remain after considering the instantaneous Markov network alone.
Based on the preceding results, we impose the following three conditions to eliminate the remaining residual mixing.
The first is a local linear-independence condition that prevents cancellation among the cross-time derivative vectors in the preceding transformation.
The second is a cross-time edge-count minimality condition that rules out additional learned cross-time edges.
The third requires each intimate-neighbor ambiguity left by the instantaneous Markov network to be resolved by a difference in the cross-time relations with some past or future component.

\begin{figure}[tpb]
    \centering
    \includegraphics[width=0.6\textwidth]{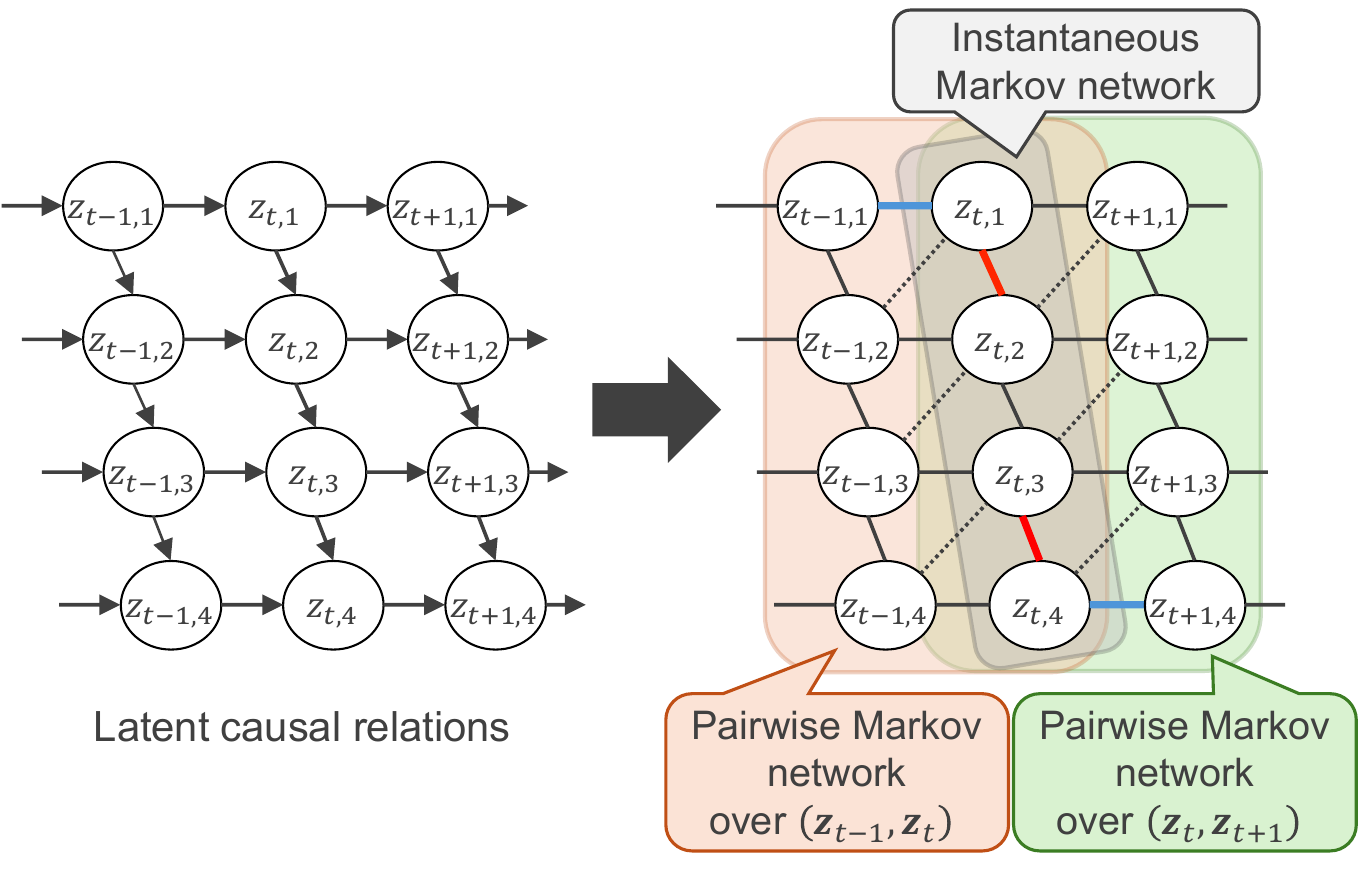}
    \caption{Illustration of how cross-time relations resolve intimate-neighbor ambiguities.
    The left panel shows a four-component latent causal model, and the right panel shows the corresponding instantaneous and cross-time pairwise Markov networks.
    In the instantaneous network, $z_{t,2}\in\Psi_{\mathcal M_z}(z_{t,1})$ and $z_{t,3}\in\Psi_{\mathcal M_z}(z_{t,4})$, as indicated by the red edges.
    The blue cross-time relations resolve the ambiguity of the first pair through $z_{t-1,1}$ and that of the second through $z_{t+1,4}$, illustrating that either a past or a future relation can suffice for each intimate-neighbor pair.
    Dotted edges denote conditional-dependence relations that are not direct causal edges.
    All Markov-network relations are defined conditionally on the auxiliary variable.
    }
    \label{fig:intimate_neighbor_example}
\end{figure}

\begin{assumption}[Cross-time variation, separation, and edge-count minimality]
\label{assumption:supp_cross_time_conditions}
For every time $t$, the following conditions hold.
\medskip
\appendixcondition{appcond:cross-time-variation}{Local cross-time variability.}
For every $j,a\in[d_z]$, at almost every $(\vz_t,\vz_{t-1})$,
\begin{equation*}
\operatorname{rank}
\left[
\vv^-_{km}(\vz_t,\vz_{t-1})
\right]_{(k,m)\in\mathcal B^-(j,a)}
=
|\mathcal B^-(j,a)|,
\end{equation*}
where the vectors $\vv^-_{km}(\vz_t,\vz_{t-1})$ form the columns of the matrix.
\medskip
\appendixcondition{appcond:cross-time-edge-minimality}{Cross-time edge-count minimality.}
The learned cross-time edge set has no greater cardinality than the true one:
\begin{equation*}
|\widehatElag^-|
\leq
|\Elag^-|.
\end{equation*}
\medskip
\appendixcondition{appcond:intimate-neighbor-separation}{Intimate-neighbor separation.}
For every $z_{t,j}\in\Psi_{\mathcal M_z}(z_{t,i})$, at least one of the following conditions holds:
\begin{enumerate}
\item There exists a past component $z_{t-1,a}$ such that
\begin{equation*}
\{z_{t-1,a},z_{t,i}\}\in \Elag^-,
\qquad
\{z_{t-1,a},z_{t,j}\}\notin \Elag^-.
\end{equation*}
\item There exists a future component $z_{t+1,b}$ such that
\begin{equation*}
\{z_{t,i},z_{t+1,b}\}\in \Elag^+,
\qquad
\{z_{t,j},z_{t+1,b}\}\notin \Elag^+.
\end{equation*}
\end{enumerate}
\end{assumption}

We first note a dimensional requirement for the local cross-time variability condition in \assumptionconditionref{assumption:supp_cross_time_conditions}{appcond:cross-time-variation}.
A necessary condition is
\begin{equation*}
N_{\mathrm{aux}}
\geq
\max_{j,a}|\mathcal B^-(j,a)|.
\end{equation*}
This dimensional bound is necessary but does not by itself guarantee linear independence.
Because Assumption~\ref{assumption:supp_cross_time_conditions} holds at every time step, Conditions~\ref{appcond:cross-time-variation} and~\ref{appcond:cross-time-edge-minimality} also apply in the future direction after shifting time forward by one step.
Thus, no separate variability, dimensional, or edge-count condition is required for the future case.
We next combine the transformation in Lemma~\ref{lemma:time_variation_transformation} with Condition~\ref{appcond:cross-time-variation} to derive Jacobian constraints from learned cross-time nonedges.

\begin{lemma}[Jacobian condition implied by a learned cross-time nonedge]
\label{lemma:time_nonedge_jacobian_condition}
Fix $(\hat{\vz}_t,\hat{\vz}_{t-1})$ and choose the corresponding local permutations $\pi_{\hat{\vz}_t}$ and $\pi_{\hat{\vz}_{t-1}}$ from Lemma~\ref{lemma:local_coordinate_correspondence}.
Suppose that $\{z_{t-1,a},z_{t,i}\}\in \Elag^-$ and $i\in\mathcal C_t(j)$.
Then
\begin{equation*}
\{\hat z_{t-1,\pi_{\hat{\vz}_{t-1}}(a)},
\hat z_{t,\pi_{\hat{\vz}_t}(j)}\}
\notin\widehatElag^-
\qquad\Longrightarrow
\frac{\partial z_{t,i}}
{\partial\hat z_{t,\pi_{\hat{\vz}_t}(j)}}
\frac{\partial z_{t-1,a}}
{\partial\hat z_{t-1,\pi_{\hat{\vz}_{t-1}}(a)}}=0.
\end{equation*}
The future-directed result follows by applying the same argument after shifting time forward by one step.
\end{lemma}

\begin{proof}
We prove the past case.
The assumed learned nonedge and Lemma~\ref{lemma:time_nonedge_zero_variation} imply
\begin{equation*}
\hat{\vv}^-_{\pi_{\hat{\vz}_t}(j),\pi_{\hat{\vz}_{t-1}}(a)}
(\hat{\vz}_t,\hat{\vz}_{t-1})=0.
\end{equation*}
Applying Lemma~\ref{lemma:time_variation_transformation} gives
\begin{equation*}
0=
\sum_{(k,m)\in\mathcal B^-(j,a)}
\chi_{km}\vv^-_{km}(\vz_t,\vz_{t-1}).
\end{equation*}
Since $a\in\mathcal C_{t-1}(a)$ by definition, the conditions $i\in\mathcal C_t(j)$ and $\{z_{t-1,a},z_{t,i}\}\in\Elag^-$ imply $(i,a)\in\mathcal B^-(j,a)$.
By the linear independence in \assumptionconditionref{assumption:supp_cross_time_conditions}{appcond:cross-time-variation}, all coefficients in this linear combination must be zero.
In particular, $\chi_{ia}=0$, which gives the claimed Jacobian condition.

The rank condition in \assumptionconditionref{assumption:supp_cross_time_conditions}{appcond:cross-time-variation} holds almost everywhere, so it remains to consider points where this condition may fail.
At such a point, the Jacobian entries associated with the chosen local permutations are nonzero by Lemma~\ref{lemma:local_coordinate_correspondence}.
By continuity, these entries remain nonzero in a sufficiently small neighborhood, so the same local permutations can be used there.
Since the rank condition holds almost everywhere, every sufficiently small neighborhood contains points where it holds.
At such points, the preceding argument gives $\chi_{ia}=0$, and continuity then gives $\chi_{ia}=0$ at the original point as well.
The future-directed result follows by applying the same argument after shifting time forward by one step.
\end{proof}

Lemma~\ref{lemma:time_nonedge_jacobian_condition} shows that every true cross-time edge is preserved under the local permutations.
Combined with cross-time edge-count minimality, this yields the local correspondence between the true and learned cross-time edge sets stated next.

\begin{lemma}[Local isomorphism of cross-time conditional-dependence relations]
\label{lemma:local_time_relation_isomorphism}
Fix $(\hat{\vz}_t,\hat{\vz}_{t-1})$ and choose the corresponding local permutations $\pi_{\hat{\vz}_t}$ and $\pi_{\hat{\vz}_{t-1}}$ from Lemma~\ref{lemma:local_coordinate_correspondence}.
Under Assumptions~\ref{assumption:supp_instantaneous_identification} and~\ref{assumption:supp_cross_time_conditions}, for every $i,a\in[d_z]$,
\begin{align*}
\{z_{t-1,a},z_{t,i}\}\in \Elag^-
\quad\Longleftrightarrow\quad
\{\hat z_{t-1,\pi_{\hat{\vz}_{t-1}}(a)},
\hat z_{t,\pi_{\hat{\vz}_t}(i)}\}
\in\widehatElag^-.
\end{align*}
The future-directed result follows by applying the same argument after shifting time forward by one step.
\end{lemma}

\begin{proof}
We first show that every true cross-time relation is preserved under the local permutations.
Take $\{z_{t-1,a},z_{t,i}\}\in \Elag^-$.
Suppose, for contradiction, that the corresponding learned pair is a nonedge:
\begin{equation*}
\{\hat z_{t-1,\pi_{\hat{\vz}_{t-1}}(a)},
\hat z_{t,\pi_{\hat{\vz}_t}(i)}\}
\notin\widehatElag^-.
\end{equation*}
Applying Lemma~\ref{lemma:time_nonedge_jacobian_condition} with $j=i$ then gives
\begin{equation*}
\frac{\partial z_{t,i}}{\partial\hat z_{t,\pi_{\hat{\vz}_t}(i)}}
\frac{\partial z_{t-1,a}}{\partial\hat z_{t-1,\pi_{\hat{\vz}_{t-1}}(a)}}
=0.
\end{equation*}
Both factors are nonzero by Lemma~\ref{lemma:local_coordinate_correspondence}, yielding a contradiction.
Therefore,
\begin{equation*}
\{z_{t-1,a},z_{t,i}\}\in\Elag^-
\quad\Longrightarrow\quad
\{\hat z_{t-1,\pi_{\hat{\vz}_{t-1}}(a)},
\hat z_{t,\pi_{\hat{\vz}_t}(i)}\}
\in\widehatElag^-.
\end{equation*}

Because $\pi_{\hat{\vz}_t}$ and $\pi_{\hat{\vz}_{t-1}}$ are permutations, this edge-preserving map is injective, and hence
\begin{equation*}
|\Elag^-|\leq|\widehatElag^-|.
\end{equation*}
\assumptionconditionref{assumption:supp_cross_time_conditions}{appcond:cross-time-edge-minimality} gives the reverse inequality, so the two edge sets have equal cardinality.
Since every true cross-time relation is already mapped to a learned one, the learned edge set contains no additional relations, proving the equivalence.
The future-directed result follows by applying the same argument after shifting time forward by one step.
\end{proof}

Lemma~\ref{lemma:residual_mixing_is_intimate} leaves only intimate neighbors as candidates for residual mixing.
\assumptionconditionref{assumption:supp_cross_time_conditions}{appcond:intimate-neighbor-separation} requires each such pair to differ in its cross-time relation with at least one past or future component.
The local cross-time correspondence established in Lemma~\ref{lemma:local_time_relation_isomorphism} then allows this difference to rule out the remaining residual mixing, as shown next.

\begin{lemma}[Elimination of residual mixing associated with intimate neighbors]
\label{lemma:eliminate_residual_mixing_by_time_relation}
Under Assumptions~\ref{assumption:supp_instantaneous_identification} and~\ref{assumption:supp_cross_time_conditions}, for every $z_{t,j}\in\Psi_{\mathcal M_z}(z_{t,i})$,
\begin{equation*}
\frac{\partial z_{t,i}}
{\partial\hat z_{t,\pi_{\hat{\vz}_t}(j)}}=0.
\end{equation*}
\end{lemma}

\begin{proof}
Fix $z_{t,j}\in\Psi_{\mathcal M_z}(z_{t,i})$ and suppose, for contradiction, that
\begin{equation*}
\frac{\partial z_{t,i}}
{\partial\hat z_{t,\pi_{\hat{\vz}_t}(j)}}\ne0.
\end{equation*}
By the definition of $\mathcal C_t(j)$, $z_{t,j}\in\Psi_{\mathcal M_z}(z_{t,i})$ implies $i\in\mathcal C_t(j)$.

First suppose that the past case in \assumptionconditionref{assumption:supp_cross_time_conditions}{appcond:intimate-neighbor-separation} holds for some component $z_{t-1,a}$:
\begin{equation*}
\{z_{t-1,a},z_{t,i}\}\in \Elag^-,
\qquad
\{z_{t-1,a},z_{t,j}\}\notin \Elag^-.
\end{equation*}
By Lemma~\ref{lemma:local_time_relation_isomorphism}, the second relation implies
\begin{equation*}
\{\hat z_{t-1,\pi_{\hat{\vz}_{t-1}}(a)}, 
\hat z_{t,\pi_{\hat{\vz}_t}(j)}\}
\notin\widehatElag^-.
\end{equation*}
Together with
$\{z_{t-1,a},z_{t,i}\}\in\Elag^-$ and $i\in\mathcal C_t(j)$, Lemma~\ref{lemma:time_nonedge_jacobian_condition} then gives
\begin{equation*}
\frac{\partial z_{t,i}}
{\partial\hat z_{t,\pi_{\hat{\vz}_t}(j)}}
\frac{\partial z_{t-1,a}}
{\partial\hat z_{t-1,\pi_{\hat{\vz}_{t-1}}(a)}}=0.
\end{equation*}
The second factor is nonzero by Lemma~\ref{lemma:local_coordinate_correspondence}, so
\begin{equation*}
\frac{\partial z_{t,i}}
{\partial\hat z_{t,\pi_{\hat{\vz}_t}(j)}}=0,
\end{equation*}
contradicting the assumed residual mixing.

If instead the future case in \assumptionconditionref{assumption:supp_cross_time_conditions}{appcond:intimate-neighbor-separation} holds for some $z_{t+1,b}$, then
\begin{equation*}
\{z_{t,i},z_{t+1,b}\}\in \Elag^+,
\qquad
\{z_{t,j},z_{t+1,b}\}\notin \Elag^+.
\end{equation*}
Applying Lemma~\ref{lemma:local_time_relation_isomorphism} and the future-directed form of Lemma~\ref{lemma:time_nonedge_jacobian_condition} gives, by the same argument,
\begin{equation*}
\frac{\partial z_{t,i}}
{\partial\hat z_{t,\pi_{\hat{\vz}_t}(j)}}=0,
\end{equation*}
again contradicting the assumed residual mixing.
Thus, residual mixing is ruled out for every intimate-neighbor pair.
\end{proof}

\subsubsection{From Local Jacobian Conditions to Component-Wise Transformations}

With residual mixing eliminated, we now use the resulting local Jacobian structure to obtain a global component-wise transformation.
Lemma~\ref{lemma:residual_mixing_is_intimate} restricts potentially nonzero off-diagonal Jacobian entries to intimate-neighbor pairs, while Lemma~\ref{lemma:eliminate_residual_mixing_by_time_relation} shows that these entries are also zero.
Indeed, fix $\hat{\vz}_t$ and its local permutation $\pi_{\hat{\vz}_t}$.
If, for some $i\ne j$,
\begin{equation*}
\frac{\partial z_{t,i}}
{\partial\hat z_{t,\pi_{\hat{\vz}_t}(j)}}\ne0,
\end{equation*}
then Lemma~\ref{lemma:residual_mixing_is_intimate} gives $z_{t,j}\in\Psi_{\mathcal M_z}(z_{t,i})$, whereas Lemma~\ref{lemma:eliminate_residual_mixing_by_time_relation} forces the same derivative to be zero.
Therefore,
\begin{equation*}
\frac{\partial z_{t,i}}
{\partial\hat z_{t,\pi_{\hat{\vz}_t}(j)}}=0
\qquad(i\ne j).
\end{equation*}
The local Jacobian is thus diagonal with nonzero diagonal entries after its columns are arranged according to $\pi_{\hat{\vz}_t}$.
The next lemma makes this pointwise permutation independent of the state.

\begin{lemma}[Globalizing the local latent-component permutation]
\label{lemma:global_permutation_from_local_jacobian}
Let $H$ be a $C^2$ diffeomorphism from the range of the learned representation onto $\mathbb R^{d_z}$.
Suppose that for every $\hat{\vz}_t$ in the domain of $H$ there exists a permutation $\pi_{\hat{\vz}_t}$ such that, for every $i$,
\begin{equation*}
\frac{\partial z_{t,i}}
{\partial\hat z_{t,k}}=0
\quad\left(k\ne\pi_{\hat{\vz}_t}(i)\right),
\qquad
\frac{\partial z_{t,i}}
{\partial\hat z_{t,\pi_{\hat{\vz}_t}(i)}}\ne0.
\end{equation*}
Then there exist a permutation $\pi_z$ independent of $\hat{\vz}_t$ and univariate $C^2$ diffeomorphisms $\xi_i$ such that
\begin{equation*}
\hat z_{t,\pi_z(i)}=\xi_i(z_{t,i})
\qquad(\forall i).
\end{equation*}
\end{lemma}

\begin{proof}
For each $\hat{\vz}_t$ in the domain of $H$, the assumed derivative pattern uniquely determines $\pi_{\hat{\vz}_t}$:
\begin{equation*}
\frac{\partial z_{t,i}}{\partial\hat z_{t,k}}=0
\quad\text{if }k\ne\pi_{\hat{\vz}_t}(i),
\qquad
\frac{\partial z_{t,i}}
{\partial\hat z_{t,\pi_{\hat{\vz}_t}(i)}}\ne0.
\end{equation*}
By continuity of $J_H$, the $d_z$ selected nonzero entries at any fixed point remain nonzero in a common neighborhood.
Since the assumption allows only one nonzero entry in each row at every point, the same permutation holds throughout this neighborhood.
Thus, $\pi_{\hat{\vz}_t}$ is locally constant.

The domain of $H$ is connected because $H$ is a diffeomorphism from that domain onto the connected space $\mathbb R^{d_z}$.
Therefore, the permutation is constant over the domain; denote it by $\pi_z$.

Differentiating the identity $H(H^{-1}(\vz))=\vz$ gives
\begin{equation*}
J_{H^{-1}}(\vz_t)
=
J_H\!\left(H^{-1}(\vz_t)\right)^{-1}.
\end{equation*}
Since $J_H$ has exactly one nonzero entry in each row and column under the fixed permutation $\pi_z$, its inverse has the corresponding pattern.
Consequently, writing $(H^{-1})_k$ for the $k$-th component of $H^{-1}$, for every $i$ and every $\vz_t$,
\begin{equation*}
\frac{\partial (H^{-1})_{\pi_z(i)}(\vz_t)}{\partial z_{t,j}}=0
\quad(j\ne i),
\qquad
\frac{\partial (H^{-1})_{\pi_z(i)}(\vz_t)}{\partial z_{t,i}}\ne0.
\end{equation*}
Since the domain of $H^{-1}$ is the connected set $\mathbb R^{d_z}$, the zero cross-component derivatives imply that $(H^{-1})_{\pi_z(i)}$ depends only on $z_i$.
Denote the resulting univariate function by $\xi_i$.
Because $H^{-1}$ is bijective and has this component-wise form, each $\xi_i$ is bijective between the corresponding coordinate ranges.
Since $H^{-1}$ is $C^2$ and the derivative of each $\xi_i$ is everywhere nonzero, each $\xi_i$ is a univariate $C^2$ diffeomorphism.
Thus,
\begin{equation*}
\hat z_{t,\pi_z(i)}=\xi_i(z_{t,i})
\qquad(i=1,\ldots,d_z).
\end{equation*}
\end{proof}

The preceding results yield the following latent-component identification theorem.
The first case shows that no cross-time conditions are needed when all intimate-neighbor sets are empty, whereas the second case gives the general result used in the main theorem.

\begin{theorem}[Identification of latent components]
\label{theorem:element_wise_identifiability}
Consider a learned representation related to the true representation by a $C^2$ diffeomorphism $H$, independent of time and the value of the auxiliary variable, from the range of the learned representation onto $\mathbb R^{d_z}$, with $\vz_t=H(\hat{\vz}_t)$, and suppose Assumption~\ref{assumption:supp_instantaneous_identification} holds.
Suppose further that either
\begin{enumerate}
\item $\Psi_{\mathcal M_z}(z_{t,i})=\emptyset$ for every $i$; or
\item Assumption~\ref{assumption:supp_cross_time_conditions} holds.
\end{enumerate}
Then there exists a permutation $\pi_z$, independent of time and the values of the latent variables, under which $\mathcal M_z$ and $\mathcal M_{\hat z}$ are isomorphic via the correspondence $z_{t,i}\mapsto\hat z_{t,\pi_z(i)}$.
Moreover, there exist univariate $C^2$ diffeomorphisms $\xi_1,\ldots,\xi_{d_z}$ such that
\begin{equation*}
\hat z_{t,\pi_z(i)}=\xi_i(z_{t,i})
\qquad(i=1,\ldots,d_z).
\end{equation*}
\end{theorem}

\begin{proof}
In the first case, the argument in Section~\ref{subsubsec:element_wise_identifiability_empty_case} shows that, at every point, arranging the local Jacobian columns according to the local permutation produces a diagonal matrix with nonzero diagonal entries.
In the second case, Lemma~\ref{lemma:residual_mixing_is_intimate} restricts all off-diagonal mixing to intimate-neighbor pairs, and Lemma~\ref{lemma:eliminate_residual_mixing_by_time_relation} eliminates this remaining mixing.
Thus, in either case, the local Jacobian becomes diagonal with nonzero diagonal entries at every point after its columns are arranged according to some local permutation $\pi_{\hat{\vz}_t}$.

Lemma~\ref{lemma:global_permutation_from_local_jacobian} gives a permutation $\pi_z$ independent of $\hat{\vz}_t$ and univariate $C^2$ diffeomorphisms $\xi_i$ such that
\begin{equation*}
\hat z_{t,\pi_z(i)}=\xi_i(z_{t,i})
\qquad(\forall i).
\end{equation*}

Substituting $\pi_{\hat{\vz}_t}=\pi_z$ into the local isomorphism of the instantaneous Markov networks established in Theorem~\ref{theorem:local_markov_network_isomorphism} shows that $\mathcal M_z$ and $\mathcal M_{\hat z}$ are isomorphic under the same fixed permutation.
Finally, because the coordinate transformation $H$ is common to all time steps, $\pi_z$ and the functions $\xi_i$ are also independent of time.
\end{proof}

\subsection{Transition-Noise Identification by IIA}
\label{app:transition_noise_identification}

This appendix adapts the transition-noise identification argument of independent innovation analysis (IIA)~\citep{Morioka2021a} to the encoder and learned inverse transition function used in our framework.
Here, transition noise plays the role of the innovations in IIA.
In the sufficient setting proved in this appendix, we take $\mathcal X=\mathbb R^{d_z}$, as in the IIA argument.
We consider two contrastive-learning approaches used in IIA~\citep{Morioka2021a}: time-contrastive learning (TCL) and generalized contrastive learning (GCL).
TCL predicts a discrete auxiliary-variable value from the estimated transition noise and a learned feature of the past observations, whereas GCL distinguishes matched tuples from tuples in which only the auxiliary-variable value is permuted.
Let $\phi$ be the encoder and define the maps $\rho_{\mathrm{aug}},\hat\rho_{\mathrm{aug}} :\mathbb R^{2d_z}\to\mathbb R^{2d_z}$ by
\begin{equation*}
\rho_{\mathrm{aug}}(\vx_t,\vx_{t-1})
:=(\veps_t,\vx_{t-1}),
\qquad
\hat\rho_{\mathrm{aug}}(\vx_t,\vx_{t-1})
:=\left(\hat r\bigl(\phi(\vx_t),\phi(\vx_{t-1})\bigr),\vx_{t-1}\right).
\end{equation*}
For the theoretical setting with $L=1$, the history feature $\vh_t=\eta(\phi(\vx_{t-1}))$ is a deterministic function of $\vx_{t-1}$, so we write the regression functions below in terms of $\vx_{t-1}$.
Accordingly, in the analysis below, we regard the implemented regression function $T(\hat{\veps}_t,\vh_t,u_t)$ as a function of $(\hat{\veps}_t,\vx_{t-1},u_t)$.
For either TCL or GCL, the regression function decomposes over the learned transition-noise components as
\begin{equation}
\label{eq:iia_structured_score}
T(\hat{\veps}_t,\vx_{t-1},u_t)
=
\sum_{i=1}^{d_z}\sum_{k=1}^{K}
\psi_{ik}(\hat\epsilon_{t,i})\mu_{ik}(u_t)
+\varphi(\vx_{t-1},u_t)
+B(\hat{\veps}_t,\vx_{t-1}).
\end{equation}
For TCL, $\mu_{ik}(u)=w_{uik}$, and $b_u$ is absorbed into $\varphi$.
The auxiliary-variable-independent term $B$ can be omitted because it cancels when the class logits are compared.
Here, $\psi_{ik}$ is a scalar-valued function of the estimated transition-noise component $\hat\epsilon_{t,i}$, $\mu_{ik}$ modulates its coefficient with the auxiliary variable, and the remaining terms account for dependence involving the past observation and terms independent of the auxiliary variable.
The corresponding finite-sample TCL and GCL objectives are given in Appendix~\ref{app:learning-framework-details}.
A regression function is population-optimal when it minimizes the expected multiclass cross-entropy for TCL or the expected binary cross-entropy for GCL under the corresponding population distributions.
Let $u_t^{(0)}$ denote a reference auxiliary-variable value and let $u_t^{(n)}$ denote the $n$-th comparison value.
For both TCL and GCL, comparing the reference value with each comparison value gives
\begin{equation}
\label{eq:iia_density_ratio}
T^\star(\hat{\veps}_t,\vx_{t-1},u_t^{(n)})
-T^\star(\hat{\veps}_t,\vx_{t-1},u_t^{(0)})
=
\log\frac{p(\hat{\veps}_t,\vx_{t-1}\mid u_t^{(n)})}
{p(\hat{\veps}_t,\vx_{t-1}\mid u_t^{(0)})}+c_n.
\end{equation}
\begin{assumption}[Conditions for transition-noise identification by IIA]
\label{assumption:supp_iia_transition_noise_identification}
The following conditions hold.

\appendixcondition{appcond:iia-regularity}{Invertibility and regularity.}
The learned inverse transition function $\hat r$ is $C^1$, and $\rho_{\mathrm{aug}}$, $\hat\rho_{\mathrm{aug}}$, and the observation function $g$ are $C^2$ diffeomorphisms on their respective domains.

\appendixcondition{appcond:iia-exponential-family}{Exponential-family modulation.}
There exist $K\ge2$, base measures $Q_i$, normalizing terms $Z_i$, twice continuously differentiable sufficient statistics $s_{ik}$, and natural parameters $\lambda_{ik}$ depending on the auxiliary variable such that
\begin{equation*}
p(\veps_t\mid u_t)
=
\prod_{i=1}^{d_z}
\frac{Q_i(\epsilon_{t,i})}{Z_i(u_t)}
\exp\left[
\sum_{k=1}^{K}s_{ik}(\epsilon_{t,i})\lambda_{ik}(u_t)
\right].
\end{equation*}
For each $i$, the sufficient statistics are nonredundant up to a constant: if $\sum_{k=1}^{K}c_k s_{ik}(e)$ is constant in $e$, then $c_1=\cdots=c_K=0$.
Let $\vlambda(u)$ collect all $d_zK$ natural parameters.
For an arbitrarily selected reference value $u_t^{(0)}$ and distinct comparison values $u_t^{(n_1)},\ldots,u_t^{(n_{d_zK})}$ satisfying $1\le n_1<\cdots<n_{d_zK}\le N_{\mathrm{aux}}$, the matrix
\begin{equation*}
\mLambda
:=
\left[
\vlambda(u_t^{(n_1)})-\vlambda(u_t^{(0)}),
\ldots,
\vlambda(u_t^{(n_{d_zK})})-\vlambda(u_t^{(0)})
\right]
\end{equation*}
has rank $d_zK$.
For a finite-valued auxiliary variable, this condition requires at least $d_zK+1$ distinct values, including the reference.

\appendixcondition{appcond:iia-regression}{Population regression and learned-feature rank.}
The learned regression function attains a population optimum satisfying~\eqref{eq:iia_density_ratio}.
For the feature functions $\psi_{ik}$ in this solution, define
\begin{equation*}
\mD_{\psi,i}(e)
:=
\left[
\frac{\partial\psi_{ik}(e)}{\partial e},
\frac{\partial^2\psi_{ik}(e)}{\partial e^2}
\right]_{k=1}^{K}
\in\mathbb R^{K\times2},
\end{equation*}
and suppose that $\operatorname{rank}\mD_{\psi,i}(e)=2$ for every $i$ and every $e$.
This condition plays the role of the learned-feature nonredundancy condition in the IIA argument of \citet{Morioka2021a}; we state the required derivative-rank property directly for simplicity.
\end{assumption}

Under these conditions, the transition-noise components are identified up to a permutation and component-wise $C^2$ diffeomorphisms.

\begin{theorem}[Transition-noise identification]
\label{theorem:supp_iia_transition_noise_identification}
Under Assumption~\ref{assumption:supp_iia_transition_noise_identification}, there exist a permutation $\pi_\epsilon$ and component-wise $C^2$ diffeomorphisms $q_i$, all independent of time and the auxiliary-variable value, such that
\begin{equation*}
\hat\epsilon_{t,\pi_\epsilon(i)}=q_i(\epsilon_{t,i})
\qquad(i=1,\ldots,d_z).
\end{equation*}
\end{theorem}

\begin{proof}
The assumptions $\veps_t\perp \vz_{t-1}\mid u_t$ from the problem setup and $\vx_{t-1}=g(\vz_{t-1})$ imply $\veps_t\perp \vx_{t-1}\mid u_t$.
By \assumptionconditionref{assumption:supp_iia_transition_noise_identification}{appcond:iia-exponential-family}, the density ratio between a reference value and a comparison value in the true augmented coordinates $(\veps_t,\vx_{t-1})$ decomposes as
\begin{equation}
\label{eq:iia_true_density_ratio}
\log\frac{p(\veps_t,\vx_{t-1}\mid u_t^{(n)})}
{p(\veps_t,\vx_{t-1}\mid u_t^{(0)})}
=
\sum_{i=1}^{d_z}\sum_{k=1}^{K}
s_{ik}(\epsilon_{t,i})
\left[\lambda_{ik}(u_t^{(n)})-\lambda_{ik}(u_t^{(0)})\right]
+P_n(\vx_{t-1}).
\end{equation}
Here, $P_n(\vx_{t-1})$ collects the density-ratio term for the past observations and the normalizing terms of the exponential-family distributions.

The transformation $\hat\rho_{\mathrm{aug}}\circ\rho_{\mathrm{aug}}^{-1}$ between the true and learned augmented coordinates is a diffeomorphism that does not depend on the auxiliary variable.
Consequently, the Jacobian terms from the change of variables cancel in the density ratio between the reference and comparison values, and at corresponding points,
\begin{equation}
\label{eq:iia_coordinate_density_ratio}
\frac{p(\hat{\veps}_t,\vx_{t-1}\mid u_t^{(n)})}
{p(\hat{\veps}_t,\vx_{t-1}\mid u_t^{(0)})}
=
\frac{p(\veps_t,\vx_{t-1}\mid u_t^{(n)})}
{p(\veps_t,\vx_{t-1}\mid u_t^{(0)})}.
\end{equation}

Let $\vy:=(\veps_t,\vx_{t-1})$, and write the learned transition noise expressed in the true augmented coordinates as $\hat{\veps}_t(\vy)$.
Stack the learned features, sufficient statistics, and learned coefficients into $\vpsi(\hat{\veps}_t)$, $\vs(\veps_t)$, and $\vmu(u_t)$, respectively, and define
\begin{equation*}
\mM
:=
\left[
\vmu(u_t^{(n_1)})-\vmu(u_t^{(0)}),
\ldots,
\vmu(u_t^{(n_{d_zK})})-\vmu(u_t^{(0)})
\right].
\end{equation*}
Taking the difference in~\eqref{eq:iia_density_ratio} eliminates the term $B$ because it does not depend on the auxiliary variable.
Using~\eqref{eq:iia_true_density_ratio} and~\eqref{eq:iia_coordinate_density_ratio} and stacking the resulting equations over the comparison values gives
\begin{equation}
\label{eq:iia_stacked_density_ratio}
\mM^\top
\vpsi\bigl(\hat{\veps}_t(\vy)\bigr)
+\Phi(\vx_{t-1})
=
\mLambda^\top \vs(\veps_t)
+P(\vx_{t-1})+c,
\end{equation}
where $\Phi$ and $P$ depend only on the past observation and $c$ stacks the constant terms.

By nonredundancy of the sufficient statistics, the space spanned by their derivative vectors has dimension $K$.
Hence, there exist evaluation points $e_i^{(1)},\ldots,e_i^{(K)}$ at which these derivative vectors are linearly independent.
Define the matrices obtained by evaluating the derivatives of the true sufficient statistics and learned features at these points as
\begin{equation*}
\begin{aligned}
\mA_i
&:=
\left[
\frac{\partial s_{ik}(e_i^{(m)})}{\partial e}
\right]_{k,m=1}^{K},
&\qquad
\mA&:=\operatorname{diag}(\mA_1,\ldots,\mA_{d_z}),
\\
\widetilde{\mA}
&:=
\left[
\left.
\frac{\partial}{\partial\epsilon_{t,i}}
\vpsi\bigl(\hat{\veps}_t(\vy)\bigr)
\right|_{\epsilon_{t,i}=e_i^{(m)}}
\right]_{(i,m)}.
\end{aligned}
\end{equation*}
Differentiating~\eqref{eq:iia_stacked_density_ratio} with respect to each $\epsilon_{t,i}$ eliminates the terms depending only on $\vx_{t-1}$ and the constants.
Evaluation at the selected points gives $\mM^\top\widetilde{\mA}=\mLambda^\top\mA$.
Since $\mLambda$ and $\mA$ are nonsingular, the equality above implies that $\mM$ is also nonsingular.

For distinct $j_1$ and $j_2$, take the mixed derivative of~\eqref{eq:iia_stacked_density_ratio} with respect to $\epsilon_{t,j_1}$ and $\epsilon_{t,j_2}$.
Since the right-hand side is a sum of univariate functions of the individual transition-noise components, its mixed derivative vanishes; the nonsingularity of $\mM$ therefore gives $\partial^2\vpsi(\hat{\veps}_t(\vy))/\partial\epsilon_{t,j_1}\partial\epsilon_{t,j_2}=0$.
Expanding the $K$ entries corresponding to $\hat\epsilon_{t,i}$ by the chain rule gives
\begin{align}
\mD_{\psi,i}(\hat\epsilon_{t,i})
\begin{pmatrix}
\dfrac{\partial^2\hat\epsilon_{t,i}}
{\partial\epsilon_{t,j_1}\partial\epsilon_{t,j_2}}
\\[2pt]
\dfrac{\partial\hat\epsilon_{t,i}}{\partial\epsilon_{t,j_1}}
\dfrac{\partial\hat\epsilon_{t,i}}{\partial\epsilon_{t,j_2}}
\end{pmatrix}
&=0,
\notag\\
\frac{\partial\hat\epsilon_{t,i}}{\partial\epsilon_{t,j_1}}
\frac{\partial\hat\epsilon_{t,i}}{\partial\epsilon_{t,j_2}}
&=0,
\label{eq:iia_noise_product_zero}
\end{align}
where the second relation follows from the full column rank of $\mD_{\psi,i}$.
Similarly, taking the mixed derivative of~\eqref{eq:iia_stacked_density_ratio} with respect to $\epsilon_{t,j}$ and $x_{t-1,a}$ and using the full column rank of $\mD_{\psi,i}$ gives
\begin{equation}
\frac{\partial\hat\epsilon_{t,i}}{\partial\epsilon_{t,j}}
\frac{\partial\hat\epsilon_{t,i}}{\partial x_{t-1,a}}
=0.
\label{eq:iia_history_product_zero}
\end{equation}

The composition $\hat\rho_{\mathrm{aug}}\circ\rho_{\mathrm{aug}}^{-1}$ has the form $(\veps_t,\vx_{t-1})\mapsto(\hat{\veps}_t(\veps_t,\vx_{t-1}),\vx_{t-1})$, and its Jacobian is block triangular:
\begin{equation}
J_{\hat\rho_{\mathrm{aug}}\circ\rho_{\mathrm{aug}}^{-1}}
=
\begin{pmatrix}
\dfrac{\partial\hat{\veps}_t}{\partial\veps_t}
&
\dfrac{\partial\hat{\veps}_t}{\partial\vx_{t-1}}
\\
0&I
\end{pmatrix}.
\label{eq:iia_augmented_jacobian}
\end{equation}
Since this map is a diffeomorphism, its Jacobian is nonsingular, and~\eqref{eq:iia_augmented_jacobian} therefore implies that $\partial\hat{\veps}_t/\partial\veps_t$ is nonsingular.
\Eqref{eq:iia_noise_product_zero} implies that each row of this Jacobian has at most one nonzero entry, while nonsingularity implies exactly one nonzero entry in each row and column.
Thus, its support defines a permutation of the transition-noise components at every point.
\Eqref{eq:iia_history_product_zero} then implies that each learned transition-noise component does not depend on the past observation.
Continuity makes this permutation locally constant, and connectedness of $\mathbb R^{2d_z}$ therefore fixes it globally to a single permutation $\pi_\epsilon$.
For each $i$, the fixed permutation together with the product-zero conditions implies that $\hat\epsilon_{t,\pi_\epsilon(i)}$ has zero derivatives with respect to all coordinates other than $\epsilon_{t,i}$.
Since the augmented-coordinate domain is $\mathbb R^{2d_z}$, it follows that $\hat\epsilon_{t,\pi_\epsilon(i)}$ is a function only of $\epsilon_{t,i}$, yielding
\begin{equation}
\label{eq:iia_innovation_correspondence}
\hat\epsilon_{t,\pi_\epsilon(i)}=q_i(\epsilon_{t,i})
\qquad(i=1,\ldots,d_z).
\end{equation}
Since $\hat\rho_{\mathrm{aug}}\circ\rho_{\mathrm{aug}}^{-1}$ is a diffeomorphism that preserves the past observation, the map $\veps_t\mapsto\hat{\veps}_t$ is bijective for each fixed $\vx_{t-1}$, and the corresponding derivatives are nonzero.
Thus, every $q_i$ is a $C^2$ diffeomorphism.
\end{proof}

\subsection{Causal-Graph Identification from the Jacobian of the Inverse Transition Function}
\label{app:causal_graph_identification}

This section combines transition-noise identification from Theorem~\ref{theorem:supp_iia_transition_noise_identification} with latent-component identification from Theorem~\ref{theorem:element_wise_identifiability} to identify the true instantaneous and lagged causal edges from the Jacobian support of the learned inverse transition function $\hat r$.
We first relate the Jacobian supports of the true and learned inverse transition functions.
Acyclicity of the instantaneous DAG then determines the correspondence between learned output rows and current latent components.
Finally, the Jacobian support of the true inverse transition function $r$ is related to the direct causal edges defined by the structural functions $f_i$.
Throughout this section, suppose that the conditions of Theorems~\ref{theorem:supp_iia_transition_noise_identification} and~\ref{theorem:element_wise_identifiability} hold.

Let $\pi_\epsilon$ and $\pi_z$ denote the permutations for the transition-noise and latent components, respectively.
Using component-wise $C^2$ diffeomorphisms $q_i$ and $\xi_i$, write
\begin{equation}
\hat\epsilon_{t,\pi_\epsilon(i)}=q_i(\epsilon_{t,i}),
\qquad
\hat z_{t,\pi_z(i)}=\xi_i(z_{t,i}).
\label{eq:supp_noise_and_latent_component_correspondences}
\end{equation}
Here, $\pi_\epsilon$ and $\pi_z$ specify the orderings of the learned output and latent input components, respectively, and are not assumed to be identical.
When describing the functional support of $f_i$, we regard it as a function of all candidate instantaneous and lagged variables, with derivatives with respect to nonparents understood to be identically zero.
We write the true inverse transition function as $r=(r_1,\ldots,r_{d_z})$, where $\epsilon_{t,i}=r_i(\vz_t,\vz_{t-1})$.

\begin{lemma}[Jacobian-support correspondence for inverse transition functions]
\label{lemma:supp_estimated_transition_jacobian_support}
The Jacobian supports of the true and learned inverse transition functions coincide up to the row permutation $\pi_\epsilon$ and column permutation $\pi_z$.
Specifically, for $\ell\in\{0,1\}$ and all $i,j$,
\begin{equation}
\label{eq:supp_estimated_transition_jacobian_support}
\frac{\partial\hat r_{\pi_\epsilon(i)}}
{\partial\hat z_{t-\ell,\pi_z(j)}}
\not\equiv0
\quad\Longleftrightarrow\quad
\frac{\partial r_i}{\partial z_{t-\ell,j}}
\not\equiv0.
\end{equation}
\end{lemma}

\begin{proof}
Substituting $\veps_t=r(\vz_t,\vz_{t-1})$ and $\hat{\veps}_t=\hat r(\hat{\vz}_t,\hat{\vz}_{t-1})$ into the transition-noise correspondence in~\eqref{eq:supp_noise_and_latent_component_correspondences} gives
\begin{equation*}
\hat r_{\pi_\epsilon(i)}(\hat{\vz}_t,\hat{\vz}_{t-1})
=q_i\bigl(r_i(\vz_t,\vz_{t-1})\bigr).
\end{equation*}
The latent-component correspondence in~\eqref{eq:supp_noise_and_latent_component_correspondences} also gives
$z_{t-\ell,j}=\xi_j^{-1}(\hat z_{t-\ell,\pi_z(j)})$.
Differentiating the preceding identity with respect to $\hat z_{t-\ell,\pi_z(j)}$ and applying the chain rule yields
\begin{equation}
\label{eq:supp_estimated_transition_jacobian_transformation}
\frac{\partial
\hat r_{\pi_\epsilon(i)}(\hat{\vz}_t,\hat{\vz}_{t-1})}
{\partial\hat z_{t-\ell,\pi_z(j)}}
=
\frac{\partial q_i(\epsilon_{t,i})}{\partial\epsilon_{t,i}}
\frac{\partial r_i(\vz_t,\vz_{t-1})}{\partial z_{t-\ell,j}}
\frac{\partial\xi_j^{-1}(\hat z_{t-\ell,\pi_z(j)})}
{\partial\hat z_{t-\ell,\pi_z(j)}}.
\end{equation}
Because $q_i$ and $\xi_j$ are univariate diffeomorphisms, the first and last factors on the right-hand side are nonzero, proving~\eqref{eq:supp_estimated_transition_jacobian_support}.
\end{proof}

We next use acyclicity of the instantaneous DAG to match the learned output rows ordered by $\pi_\epsilon$ uniquely to the current latent components ordered by $\pi_z$.

\begin{lemma}[Matching transition-noise and latent components using the instantaneous DAG]
\label{lemma:supp_unique_transition_state_matching}
Suppose that the instantaneous causal structure is a DAG as in the problem setup and that every structural function $f_i$ is continuously differentiable.
Consider the bipartite graph whose left nodes are $r_i$, whose right nodes are $z_{t,j}$, and which contains an edge whenever $\partial r_i/\partial z_{t,j}\not\equiv0$.
A perfect matching pairs every $r_i$ with exactly one $z_{t,j}$ and every $z_{t,j}$ with exactly one $r_i$.
This graph has a unique perfect matching, consisting of the edges between $r_i$ and $z_{t,i}$.
Consequently, the learned Jacobian support uniquely determines the pairing between each output row $\hat r_{\pi_\epsilon(i)}$ and its corresponding current latent component $\hat z_{t,\pi_z(i)}$.
\end{lemma}

\begin{proof}
We first show that the pairs $(r_i,z_{t,i})$ form a perfect matching and then prove that no other perfect matching exists.
\assumptionconditionref{assumption:supp_iia_transition_noise_identification}{appcond:iia-regularity} ensures that the true inverse transition function $r$ is continuously differentiable.
Substituting $\epsilon_{t,i}=r_i(\vz_t,\vz_{t-1})$ into the structural equation gives
\begin{equation*}
z_{t,i}
=
f_i\!\left(
\mathrm{pa}_{\mathrm{inst}}(z_{t,i}),
\mathrm{pa}_{\mathrm{lags}}(z_{t,i}),
r_i(\vz_t,\vz_{t-1})
\right).
\end{equation*}
Because $z_{t,i}$ is not its own instantaneous parent, differentiating this identity with respect to $z_{t,i}$ produces no direct parent term, leaving only the term through $r_i$.
For $j\ne i$, differentiating the same identity with respect to $z_{t,j}$ gives zero on the left-hand side.
Thus,
\begin{equation*}
1=
\frac{\partial f_i}{\partial\epsilon_{t,i}}
\frac{\partial r_i}{\partial z_{t,i}},
\qquad
0=
\frac{\partial f_i}{\partial z_{t,j}}
+
\frac{\partial f_i}{\partial\epsilon_{t,i}}
\frac{\partial r_i}{\partial z_{t,j}}.
\end{equation*}
For each past component $z_{t-1,j}$, differentiating the same identity gives
\begin{equation*}
0=
\frac{\partial f_i}{\partial z_{t-1,j}}
+
\frac{\partial f_i}{\partial\epsilon_{t,i}}
\frac{\partial r_i}{\partial z_{t-1,j}}.
\end{equation*}
The first equality implies that both factors are nonzero throughout the domain. 
Therefore,
\begin{equation*}
\frac{\partial r_i}{\partial z_{t,i}}
=
\left(\frac{\partial f_i}{\partial\epsilon_{t,i}}\right)^{-1}
\ne0,
\qquad
\frac{\partial r_i}{\partial z_{t,j}}
=-
\left(\frac{\partial f_i}{\partial\epsilon_{t,i}}\right)^{-1}
\frac{\partial f_i}{\partial z_{t,j}}.
\end{equation*}
Similarly, for each past component,
\begin{equation*}
\frac{\partial r_i}{\partial z_{t-1,j}}
=-
\left(\frac{\partial f_i}{\partial\epsilon_{t,i}}\right)^{-1}
\frac{\partial f_i}{\partial z_{t-1,j}}.
\end{equation*}
Thus, the edges joining $r_i$ to $z_{t,i}$ form a perfect matching, while the off-diagonal Jacobian support corresponds to the instantaneous causal edges.
It remains to show that this perfect matching is unique.
Suppose, for contradiction, that a different perfect matching exists.
Relative to the matching $(r_i,z_{t,i})$, its associated permutation contains a nontrivial cycle $i_1,\ldots,i_s$.
Writing $i_{s+1}=i_1$, this matching pairs each row $r_{i_h}$ with the distinct current component $z_{t,i_{h+1}}$.
Hence,
\begin{equation*}
\frac{\partial r_{i_h}}{\partial z_{t,i_{h+1}}}\not\equiv0
\qquad(h=1,\ldots,s).
\end{equation*}
Because $i_h\ne i_{h+1}$, the relation derived above between the derivatives of $r_i$ and $f_i$ implies
\begin{equation*}
\frac{\partial f_{i_h}}{\partial z_{t,i_{h+1}}}\not\equiv0
\qquad(h=1,\ldots,s).
\end{equation*}
These dependencies form the directed instantaneous cycle
\begin{equation*}
z_{t,i_2}\to z_{t,i_1},\quad
z_{t,i_3}\to z_{t,i_2},\quad\ldots,\quad
z_{t,i_1}\to z_{t,i_s},
\end{equation*}
contradicting the instantaneous DAG assumption.
Therefore, the matching between corresponding components is unique.
By Lemma~\ref{lemma:supp_estimated_transition_jacobian_support}, the true and learned supports coincide up to $\pi_\epsilon$ and $\pi_z$; hence, the unique learned matching pairs $\hat r_{\pi_\epsilon(i)}$ with $\hat z_{t,\pi_z(i)}$.
This argument uses only acyclicity of the instantaneous DAG and does not require a unique topological order.
\end{proof}

Together, the two lemmas match each output row of $\hat r$ to its corresponding current latent component.
The Jacobian supports with respect to its current and past inputs can then be interpreted as direct causal edges.

\begin{theorem}[Causal-graph identification from the learned inverse transition function]
\label{theorem:supp_causal_graph_from_transition_jacobian}
Suppose that Theorems~\ref{theorem:supp_iia_transition_noise_identification} and~\ref{theorem:element_wise_identifiability} hold.
Under the instantaneous DAG assumed in the problem setup, also suppose that every structural function $f_i$ is continuously differentiable.
Define the direct causal edges induced by $f_i$ by
\begin{equation*}
z_{t,j}\to z_{t,i}
\quad\Longleftrightarrow\quad
\frac{\partial f_i}{\partial z_{t,j}}\not\equiv0
\quad(j\ne i),
\qquad
z_{t-1,j}\to z_{t,i}
\quad\Longleftrightarrow\quad
\frac{\partial f_i}{\partial z_{t-1,j}}\not\equiv0.
\end{equation*}
Then, for all $i\ne j$,
\begin{equation*}
\frac{\partial\hat r_{\pi_\epsilon(i)}}
{\partial\hat z_{t,\pi_z(j)}}
\not\equiv0
\quad\Longleftrightarrow\quad
\frac{\partial f_i}{\partial z_{t,j}}
\not\equiv0,
\end{equation*}
and, for all $i,j$,
\begin{equation*}
\frac{\partial\hat r_{\pi_\epsilon(i)}}
{\partial\hat z_{t-1,\pi_z(j)}}
\not\equiv0
\quad\Longleftrightarrow\quad
\frac{\partial f_i}{\partial z_{t-1,j}}
\not\equiv0.
\end{equation*}
Consequently, the true instantaneous and one-step lagged causal edges are identifiable up to a single component permutation $\pi_z$ from the off-diagonal instantaneous and one-step lagged Jacobian supports of $\hat r$, respectively.
\end{theorem}

\begin{proof}
For $\ell\in\{0,1\}$, Lemma~\ref{lemma:supp_estimated_transition_jacobian_support} gives
\begin{equation*}
\frac{\partial\hat r_{\pi_\epsilon(i)}}
{\partial\hat z_{t-\ell,\pi_z(j)}}
\not\equiv0
\quad\Longleftrightarrow\quad
\frac{\partial r_i}{\partial z_{t-\ell,j}}
\not\equiv0.
\end{equation*}
The unique perfect matching in Lemma~\ref{lemma:supp_unique_transition_state_matching} assigns the learned output row $\hat r_{\pi_\epsilon(i)}$ uniquely to its child component $\hat z_{t,\pi_z(i)}$.
For $\ell=0$ with $j\ne i$, or for $\ell=1$, a nonzero entry in this output row at the column $\hat z_{t-\ell,\pi_z(j)}$ therefore corresponds to the edge
\begin{equation*}
\hat z_{t-\ell,\pi_z(j)}
\to
\hat z_{t,\pi_z(i)}.
\end{equation*}
This assignment does not assume $\pi_\epsilon=\pi_z$; the unique perfect matching associates each output row with its child component indexed according to $\pi_z$.

As derived in the proof of Lemma~\ref{lemma:supp_unique_transition_state_matching}, differentiating the structural equation after substituting $\epsilon_{t,i}=r_i(\vz_t,\vz_{t-1})$ gives, for $\ell=0$ with $j\ne i$, and for $\ell=1$,
\begin{equation*}
\frac{\partial r_i}{\partial z_{t-\ell,j}}
=-
\left(
\frac{\partial f_i}{\partial\epsilon_{t,i}}
\right)^{-1}
\frac{\partial f_i}{\partial z_{t-\ell,j}}.
\end{equation*}
As shown in the proof of
Lemma~\ref{lemma:supp_unique_transition_state_matching},
$\partial f_i/\partial\epsilon_{t,i}$ is nonzero throughout the domain.
Therefore, a nonzero off-diagonal instantaneous component or a nonzero one-step lagged component of the Jacobian of $r_i$ is equivalent to the corresponding partial derivative of $f_i$ being nonzero.
Thus, under the single node permutation $\pi_z$, the off-diagonal instantaneous and one-step lagged Jacobian supports of $\hat r$ correspond exactly to the true instantaneous and one-step lagged causal edges, respectively.
\end{proof}

\begin{remark}[Relation to Jacobian-based causal-graph identification]
The proof follows the idea of \citet{Reizinger2023a}: after identifying transition noise and latent states component-wise, the causal graph is identified from the Jacobian support of the inverse transition function.
We apply this idea to latent transition dynamics, using the instantaneous Jacobian to match transition-noise components to current latent components and the instantaneous and lagged Jacobians to identify their respective causal edges.
\end{remark}

\subsection{Extension to Multiple Lags}
\label{app:multiple_lags}

We extend the preceding identification results from $L=1$ to any finite maximum lag $L$.
The single past state is replaced by a finite history, while cross-time relations and variability are considered separately at each lag.
Write the finite history over the past $L$ time steps as $\vz_t^-:=(\vz_{t-1},\ldots,\vz_{t-L})$.
Let $\mathcal M_z$ denote the pairwise Markov network representing the instantaneous conditional-independence relations in $p(\vz_t\mid\vz_t^-,u_t)$, and let $\Einst$ denote its edge set.
For each $\ell=1,\ldots,L$, let $\Elag^{(\ell)}$ denote the past--current edges between $\vz_{t-\ell}$ and $\vz_t$ in the pairwise Markov network of the conditional joint distribution $p(\vz_t^-,\vz_t\mid u_t)$.
Define the corresponding learned Markov networks analogously, with instantaneous edge set $\widehatEinst$ and lag-specific past--current edge sets $\widehatElag^{(\ell)}$, and let
\begin{equation*}
\begin{aligned}
\Delta_n^{(L)}(\vz_t,\vz_t^-)
&:=\log p(\vz_t\mid \vz_t^-,u_t^{(n)})
-\log p(\vz_t\mid \vz_t^-,u_t^{(0)}),\\
\vv_{ia}^{(\ell)}(\vz_t,\vz_t^-)
&:=\left(
\frac{\partial^2\Delta_n^{(L)}}
{\partial z_{t,i}\partial z_{t-\ell,a}}
\right)_{n=1}^{N_{\mathrm{aux}}}.
\end{aligned}
\end{equation*}
Here, we retain the assumption on $H$ in Condition~\ref{cond:latent-regularity} of Theorem~\ref{thm:elementwise_identification} throughout this section and define $\widehat\Delta_n^{(L)}$ and $\hat{\vv}_{ia}^{(\ell)}$ analogously for the learned representation.
Define the candidate sets $\mathcal C_s$ as in Appendix~\ref{app:latent_component_identification}, using the instantaneous Markov network $\mathcal M_z$ defined above, and let
\begin{equation*}
\mathcal B^{(\ell)}(j,a)
:=
\left\{(k,m)\in[d_z]^2\ \middle|\
k\in\mathcal C_t(j),\ m\in\mathcal C_{t-\ell}(a),\
\{z_{t-\ell,m},z_{t,k}\}\in \Elag^{(\ell)}\right\}.
\end{equation*}
The corresponding future-directed candidate blocks are defined by shifting these quantities forward by $\ell$ time steps.

\begin{assumption}[Finite-lag extension conditions]
\label{assumption:supp_finite_lag_extension}
For every time $t$, the following conditions hold whenever the required past and future variables exist.

\medskip
\appendixcondition{appcond:finite-history-noise-identification}{Finite-history transition-noise identification.}
All conditions in Appendix~\ref{app:transition_noise_identification} hold after replacing the past observation $\vx_{t-1}$ by the finite history $\vx_t^-:=(\vx_{t-1},\ldots,\vx_{t-L})$ in the augmented-coordinate construction.

\medskip
\appendixcondition{appcond:finite-instantaneous-variation}{Instantaneous variability and edge-count minimality.}
The positivity and regularity condition in \assumptionconditionref{assumption:supp_instantaneous_identification}{appcond:instantaneous-regularity} holds with $\vz_t^-$ in place of $\vz_{t-1}$.
Let $N_{\mathrm{inst}}:=2d_z+|\Einst|$ and define
$\vw_n^{(L)}(\vz_t,\vz_t^-)\in\mathbb R^{N_{\mathrm{inst}}}$
from $\Delta_n^{(L)}$ in the same way as $\vw_n$ is defined in \assumptionconditionref{assumption:supp_instantaneous_identification}{appcond:instantaneous-variation}, replacing $\vz_{t-1}$ with $\vz_t^-$.
For every $\vz_t$, suppose that there exist auxiliary-variable indices and finite-history evaluation values
\begin{equation*}
(n_\alpha,\bar{\vz}_t^{-,(\alpha)})
\in
\{1,\ldots,N_{\mathrm{aux}}\}\times\mathbb R^{Ld_z},
\qquad
\alpha=1,\ldots,N_{\mathrm{inst}},
\end{equation*}
such that
\begin{equation*}
\mW^{(L)}(\vz_t)
:=
\left[
\vw_{n_\alpha}^{(L)}(\vz_t,\bar{\vz}_t^{-,(\alpha)})
\right]_{\alpha=1}^{N_{\mathrm{inst}}}
\end{equation*}
has rank $N_{\mathrm{inst}}$.
The indices $n_\alpha$ need not be distinct.
The learned instantaneous edge set also satisfies
\begin{equation*}
|\widehatEinst|\leq|\Einst|.
\end{equation*}

\medskip
\appendixcondition{appcond:finite-cross-time-variation}{Per-lag cross-time variability and edge-count minimality.}
For every $\ell=1,\ldots,L$ and $j,a\in[d_z]$, at almost every $(\vz_t,\vz_t^-)$,
\begin{equation*}
\operatorname{rank}
\left[
\vv_{km}^{(\ell)}(\vz_t,\vz_t^-)
\right]_{(k,m)\in\mathcal B^{(\ell)}(j,a)}
=
|\mathcal B^{(\ell)}(j,a)|.
\end{equation*}
The learned cross-time edge sets satisfy
\begin{equation*}
\sum_{\ell=1}^{L}|\widehatElag^{(\ell)}|
\leq
\sum_{\ell=1}^{L}|\Elag^{(\ell)}|.
\end{equation*}

\medskip
\appendixcondition{appcond:finite-intimate-neighbor-separation}{Intimate-neighbor separation across lags.}
For every $z_{t,j}\in\Psi_{\mathcal M_z}(z_{t,i})$, there exists at least one lag
$\ell\in\{1,\ldots,L\}$ and a past or future component at that lag whose
cross-time relation is present with $z_{t,i}$ but absent with $z_{t,j}$, as in \assumptionconditionref{assumption:supp_cross_time_conditions}{appcond:intimate-neighbor-separation}.
\end{assumption}

This yields the following extension of the $L=1$ result.
\begin{theorem}[Identification for a general finite maximum lag]
\label{theorem:supp_finite_lag_identification}
Suppose that Assumption~\ref{assumption:supp_finite_lag_extension} holds.
Then there exist a time-independent permutation $\pi_z$ and univariate $C^2$ diffeomorphisms $\xi_i$ such that $\hat z_{t,\pi_z(i)}=\xi_i(z_{t,i})$.
Furthermore, under the instantaneous DAG assumed in the problem setup, suppose that every structural function $f_i$ is continuously differentiable.
Let $\pi_\epsilon$ be the transition-noise permutation obtained by IIA.
Then, for all $i\ne j$,
\begin{equation*}
\frac{\partial\hat r_{\pi_\epsilon(i)}}
{\partial\hat z_{t,\pi_z(j)}}\not\equiv0
\quad\Longleftrightarrow\quad
\frac{\partial f_i}{\partial z_{t,j}}\not\equiv0,
\end{equation*}
and, for every $\ell=1,\ldots,L$ and all $i,j$,
\begin{equation}
\label{eq:supp_finite_lag_graph_support}
\frac{\partial\hat r_{\pi_\epsilon(i)}}
{\partial\hat z_{t-\ell,\pi_z(j)}}\not\equiv0
\quad\Longleftrightarrow\quad
\frac{\partial f_i}{\partial z_{t-\ell,j}}\not\equiv0.
\end{equation}
Consequently, the true instantaneous and lagged causal edges are identifiable up to a single component permutation $\pi_z$ from the off-diagonal instantaneous and per-lag Jacobian supports of $\hat r$.
\end{theorem}

\begin{proof}
In Appendix~\ref{app:transition_noise_identification}, replacing $\vx_{t-1}$ with $\vx_t^-$ preserves the block-triangular Jacobian of the augmented coordinate transformation, so the same mixed-derivative argument applies to every component of the finite history.
In Appendix~\ref{app:instantaneous_markov_identification}, replacing past-state evaluation values with finite-history evaluation values does not change the fact that the unknown coefficients constructed from the coordinate transformation $H$ depend only on $\vz_t$.
The full-rank condition on $\mW^{(L)}(\vz_t)$ and instantaneous edge-count minimality therefore imply an isomorphism of $\mathcal M_z$.

For each lag $\ell$, taking the mixed derivative of $\widehat\Delta_n^{(L)}$ with respect to $\hat{\vz}_t$ and $\hat{\vz}_{t-\ell}$ introduces no terms from other lags because $H$ acts separately at each time step, and gives
\begin{equation*}
\hat{\vv}^{(\ell)}_{
\pi_{\hat{\vz}_t}(j),\pi_{\hat{\vz}_{t-\ell}}(a)}
=\sum_{(k,m)\in\mathcal B^{(\ell)}(j,a)}
\frac{\partial z_{t,k}}{\partial\hat z_{t,\pi_{\hat{\vz}_t}(j)}}
\frac{\partial z_{t-\ell,m}}{\partial\hat z_{t-\ell,\pi_{\hat{\vz}_{t-\ell}}(a)}}
\vv_{km}^{(\ell)}.
\end{equation*}
The local full-rank condition preserves every true edge at each lag, so $|\Elag^{(\ell)}|\leq|\widehatElag^{(\ell)}|$.
Thus, each difference $|\widehatElag^{(\ell)}|-|\Elag^{(\ell)}|$ is nonnegative, while the aggregate edge-count minimality requires their sum to be nonpositive. 
Every difference must therefore be zero, yielding edge-set correspondence at each lag.
The argument of Appendix~\ref{app:latent_component_identification} then eliminates the remaining mixing and yields $\pi_z$ and $\xi_i$.

Finally, the chain-rule support correspondence in Appendix~\ref{app:causal_graph_identification} extends directly to every input lag, while the unique perfect matching between output rows and current latent components is determined from the instantaneous Jacobian as before.
Thus, at every lag, the Jacobian supports of the true and learned inverse transition functions correspond under the same row permutation $\pi_\epsilon$ and column permutation $\pi_z$, proving the stated graph-identification result.
\end{proof}

\section{Learning Framework Details}
\label{app:learning-framework-details}

This section provides details of the learning framework described in the main text, including its model parameterization, contrastive objectives, structural regularization, and optimization.

\subsection{Model Parameterization}
\label{app:learning-model-parameterization}

\paragraph{Shared encoder and learned inverse transition function.}

For maximum lag $L$, our method applies the shared encoder $\phi$ separately to each observation in $(\vx_{t-L},\ldots,\vx_t)$, yielding $\hat{\vz}_{t-\ell}=\phi(\vx_{t-\ell})$ for $\ell=0,\ldots,L$.
Each transition-noise component is then estimated by a component-specific learned inverse transition function.
Each component of $\hat r=(\hat r_1,\ldots,\hat r_{d_z})$ is parameterized as
\begin{equation}
 \hat\epsilon_{t,i}
 =\hat r_i\!\left(
   \hat z_{t,i},
   \mG^{(0)}_{i,-i}\odot\hat{\vz}_{t,-i},
   \left\{\mG^{(\ell)}_{i,:}\odot\hat{\vz}_{t-\ell}\right\}_{\ell=1}^{L}
 \right).
 \label{eq:supp_alg_weighted_inverse_transition}
\end{equation}
Here, $\mG^{(0)},\mG^{(\ell)}\in\mathbb R^{d_z\times d_z}$, $\ell=1,\ldots,L$, are learnable weight matrices that gate individual current and lagged input components, respectively.
For these matrices, row $i$ corresponds to the target component $\hat z_{t,i}$ and column $j$ to the source component $\hat z_{t-\ell,j}$.
The notation $-i$ denotes all components except component $i$, whereas $:$ denotes all components.
Accordingly, $\mG^{(0)}_{i,-i}$ gates the other current components and $\mG^{(\ell)}_{i,:}$ gates all components at lag $\ell$.
Because the current component $\hat z_{t,i}$ is supplied directly to $\hat r_i$ without a gate, the diagonal entries of $\mG^{(0)}$ are fixed to zero.
For each $i$, the gated inputs are concatenated and processed jointly by a general nonlinear network parameterizing $\hat r_i$.

\paragraph{History feature.}
For the contrastive objectives below, the past learned representations are
summarized by the history feature
\begin{equation}
 \vh_t=\eta(\hat{\vz}_{t-1},\ldots,\hat{\vz}_{t-L}),
 \label{eq:supp_alg_history_feature}
\end{equation}
where the history feature map $\eta$ is learned jointly with the other modules.

\subsection{Contrastive Learning with Auxiliary Variables}
\label{app:learning-contrastive}

We implement the IIA contrastive-learning framework~\citep{Morioka2021a}
in two forms.
\methodTCL{} treats a discrete auxiliary variable as a finite set of classes,
whereas \methodGCL{} accommodates a general observed auxiliary variable.
We describe the corresponding regression functions and training losses below.

\paragraph{\methodTCL{}.}
For a discrete auxiliary variable, \methodTCL{} treats each possible value as a class and predicts it from the estimated transition noise and history feature.
For a candidate value $u$, the logit is
\begin{equation}
 T_u^{\mathrm{TCL}}(\hat{\veps}_t,\vh_t)
 =
 \sum_{i=1}^{d_z}\sum_{k=1}^{K}
 w_{uik}\psi_{ik}(\hat\epsilon_{t,i})
 +\varphi(\vh_t,u)+b_u.
 \label{eq:supp_alg_tcl_score}
\end{equation}
Here, $K$ is the number of feature functions for each transition-noise component, and $\psi_{ik}$ is the $k$-th scalar-valued feature function of $\hat\epsilon_{t,i}$.
The feature functions may be learned or fixed; when the conditional noise family is specified, fixed functions can be chosen to include its sufficient statistics.
The coefficients $w_{uik}$ and biases $b_u$ are learned separately for each class.
The learned function $\varphi(\vh_t,u)$ represents the relation between the history feature and the auxiliary variable.
We omit an auxiliary-variable-independent common term because it cancels from the softmax probabilities.
The softmax cross-entropy for the observed $u_t$ defines $\mathcal L_{\mathrm{TCL}}$.

\paragraph{\methodGCL{}.}
\methodGCL{} uses changes in the transition-noise distributions associated with a general observed auxiliary variable through binary classification of matched and mismatched tuples.
Its binary logistic regression function is
\begin{equation}
 \begin{aligned}
 T^{\mathrm{GCL}}(\hat{\veps}_t,\vh_t,u_t)
 ={}&
 \sum_{i=1}^{d_z}\sum_{k=1}^{K}
 \psi_{ik}(\hat\epsilon_{t,i})
 \mu_{ik}(u_t)
 +\varphi(\vh_t,u_t)
 +B(\hat{\veps}_t,\vh_t).
 \end{aligned}
 \label{eq:supp_alg_gcl_score}
\end{equation}
Here, $\mu_{ik}$ is a learned function of the auxiliary variable that gives the coefficient of the $k$-th feature of component $i$.
The learned function $\varphi$ represents the relation between the history feature and the auxiliary variable as well as terms that depend only on the auxiliary variable, whereas $B$ collects terms independent of the auxiliary variable.
The component-wise additive restriction applies only to the auxiliary-variable-dependent term; $B$ may therefore depend jointly on all estimated transition-noise components and the history feature.
For each minibatch, matched tuples $(\hat{\veps}_t,\vh_t,u_t)$ are positive examples.
Mismatched tuples $(\hat{\veps}_t,\vh_t,u'_t)$ are constructed by permuting only the auxiliary-variable values within the minibatch, while keeping $\hat{\veps}_t$ and $\vh_t$ fixed.
We use a nonzero cyclic shift for this permutation so that each sample receives the auxiliary-variable value from a different sample.
The resulting binary cross-entropy defines $\mathcal L_{\mathrm{GCL}}$.

\subsection{Structural Regularization and Training Objective}
\label{app:learning-structural-regularization}

\paragraph{Sparsity regularization of input dependencies.}
Motivated by the edge-count minimality used in the identification analysis in Appendix~\ref{app:latent_component_identification}, we impose the following regularizers on the input gates in~\eqref{eq:supp_alg_weighted_inverse_transition}:
\begin{equation}
\mathcal R_{\mathrm{inst}}
=
\frac{1}{d_z^2}
\sum_{i,j}
\left|\mG^{(0)}_{ij}\right|,
\qquad
\mathcal R_{\mathrm{lag}}
=
\frac{1}{Ld_z^2}
\sum_{\ell=1}^{L}
\sum_{i,j}
\left|\mG^{(\ell)}_{ij}\right|.
\end{equation}
For any gated current or lagged input $\hat z_{t-\ell,j}$ supplied to $\hat r_i$ through $\mG^{(\ell)}_{ij}$, the chain rule gives
\begin{equation}
 \frac{\partial\hat r_i}{\partial\hat z_{t-\ell,j}}
 =\mG^{(\ell)}_{ij}
 \frac{\partial\hat r_i}
 {\partial\!\left(\mG^{(\ell)}_{ij}\hat z_{t-\ell,j}\right)}.
 \label{eq:supp_alg_gate_jacobian_relation}
\end{equation}
Consequently, $\mG^{(\ell)}_{ij}=0$ makes the corresponding Jacobian component zero.
Regularizing the gates therefore provides a computationally simple surrogate for encouraging sparse input dependencies, since it acts directly on the weight matrices without explicitly evaluating the Jacobian.
IDOL instead applies an $\ell_1$ penalty directly to the Jacobian~\citep{Li2025a}, which can also be used here but requires explicit Jacobian evaluation during training.
We combine the instantaneous and lagged penalties as
\begin{equation}
\mathcal R_{\mathrm{sparsity}}
:=
\mathcal R_{\mathrm{inst}}
+
\mathcal R_{\mathrm{lag}}.
\end{equation}

\paragraph{Instantaneous DAG regularization.}
When the instantaneous causal order is not fixed, we encourage acyclicity using the following NOTEARS-style regularizer~\citep{Zheng2018a}:
\begin{equation}
 \mathcal R_{\mathrm{DAG}}
 :=\operatorname{tr}\!\left[
  \exp\!\left(\mG^{(0)}\odot \mG^{(0)}\right)
 \right]-d_z.
 \label{eq:supp_alg_notears_penalty}
\end{equation}
Because this term is optimized as an ordinary loss penalty in our implementation, \eqref{eq:supp_alg_notears_penalty} is a soft penalty and does not strictly guarantee acyclicity. 
When a causal order is fixed in advance, a hard mask restricts the learnable entries to the triangular part consistent with that order, structurally guaranteeing an instantaneous DAG; in this case, $\mathcal R_{\mathrm{DAG}}$ can be omitted.

\paragraph{Joint optimization.}
The regularizers are combined with the selected TCL or GCL loss as in~\eqref{eq:learning_full_objective}, and all learnable parameters are optimized jointly.
A decoder and reconstruction loss may also be added, although neither is required by the learning framework or identification theory.
\section{Synthetic-data Experimental Details and Additional Results}
\label{app:experimental-details}

The main text presents the primary comparison between \methodname{} and the baselines in the default setting. 
This section details data generation, training conditions, baselines, and evaluation, followed by analyses of latent causal structure, latent dimension, sample size, data-generating function complexity, and objective components.
General definitions of the model, TCL/GCL losses, input-gate sparsity regularization, and DAG regularization are given in Appendix~\ref{app:learning-framework-details}.

\subsection{Common Experimental Settings}
\label{app:synthetic_common_experimental_settings}

\subsubsection{Data Generation}

\paragraph{Data-generating model.}
In the default setting, we generate sequences using nonlinear structural functions and a nonlinear observation function that remain invariant across auxiliary-variable values, while only the transition-noise distribution changes. We set $d_z=d_x=5$ and $L=1$ and generate data as
\begin{equation}
 z_{t,i}=f_i\!\left(\mathrm{pa}_{\mathrm{inst}}(z_{t,i}),
                     \mathrm{pa}_{\mathrm{lags}}(z_{t,i}),
                     \epsilon_{t,i}\right),
 \qquad \vx_t=g(\vz_t).
 \label{eq:supp_synthetic_true_model}
\end{equation}
We use the additive-noise special case $f_i(a,b,e)=\widetilde f_i(a,b)+e$. In the default setting and dimension sweep, the instantaneous graph is the chain DAG $z_{t,i}\to z_{t,i+1}$ for $i=1,\ldots,d_z-1$. The lagged edges are $z_{t-1,i}\to z_{t,i}$ and $z_{t-1,i}\to z_{t,1+((i+1)\bmod d_z)}$ for $i=1,\ldots,d_z$. Figure~\ref{fig:synthetic_graph_example} illustrates these graphs for the default five-component setting, together with their adjacency matrices.
\begin{figure}[tbp]
    \centering
    \includegraphics[width=0.55\textwidth]{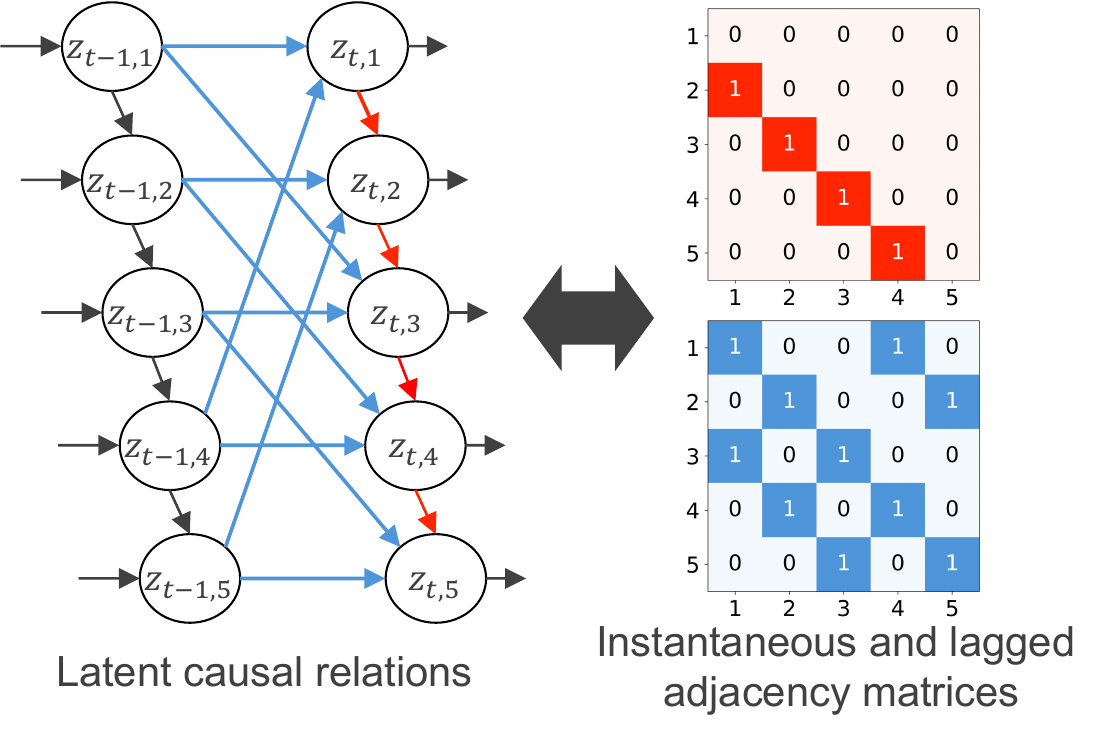}
    \caption{Instantaneous and lagged causal graphs in the default five-component setting and their adjacency matrices. Red and blue indicate instantaneous and lagged causal edges, respectively; rows and columns index targets and sources.}
    \label{fig:synthetic_graph_example}
\end{figure}

\paragraph{Function parameterization and data splits.}
Each $\widetilde f_i$ is a three-layer multilayer perceptron (MLP) that uses Leaky-ReLU activations and jointly takes all instantaneous and lagged parents specified by the true causal graph as input. The nonlinear observation function $g$ consists of three successive layers, each applying a Leaky-ReLU activation followed by an orthogonal linear transformation. For each seed, the training, validation, and test splits use independent transition-noise samples but share the true graph, structural and observation functions, and transition-noise parameters.

\paragraph{Discrete variation in transition-noise distributions.}
For discrete-variation data, the transition-noise distribution switches across discrete environments. Each split contains 256 independent sequences of length 256. For every environment $u\in\{1,\ldots,256\}$ and component $i$, we sample once and fix
\begin{equation}
 \mu_{u,i}\sim\operatorname{Unif}[-0.7,0.7],\qquad
 \sigma_{u,i}\sim\operatorname{Unif}[0.4,1.1],
 \label{eq:supp_synthetic_tcl_noise}
\end{equation}
and define
\begin{equation}
 p(\veps_t\mid u_t=u)
 =\prod_{i=1}^{d_z}\mathcal N\!\left(\epsilon_{t,i};\mu_{u,i},\sigma_{u,i}^2\right).
 \label{eq:supp_synthetic_tcl_conditional}
\end{equation}
Transition noise is sampled independently from this distribution at each time. Thus, its distribution is constant within a sequence, while its component-wise means and standard deviations switch across sequences. Each split contains $256\times256=2^{16}$ samples.

\paragraph{Continuous variation in transition-noise distributions.}
For continuous-variation data, the transition-noise distribution varies smoothly with continuous time. Each split is one sequence of length $T=2^{18}$, with $t=0,\ldots,T-1$ used as a continuous auxiliary variable. For each component $i$, we construct two smooth time series from separate random linear combinations of the sine and cosine functions $\{\sin(2\pi kt/T),\cos(2\pi kt/T)\}_{k=1}^{64}$. We standardize each time series, apply the standard Gaussian cumulative distribution function, and linearly rescale the results to $[-0.7,0.7]$ and $[0.4,1.1]$ to obtain $\mu_i(t)$ and $\sigma_i(t)$, respectively. We then set
\begin{equation}
 p(\veps_t\mid t)
 =\prod_{i=1}^{d_z}\mathcal N\!\left(\epsilon_{t,i};\mu_i(t),\sigma_i(t)^2\right).
 \label{eq:supp_synthetic_gcl_noise}
\end{equation}
For methods requiring discrete labels, we divide the sequence into 256 equal-width bins and use the pseudo-label $\widetilde u_t=\lfloor256t/T\rfloor\in\{0,\ldots,255\}$, which we distinguish from the continuous auxiliary variable $t$.
In both data types, the conditional transition-noise distribution is a component-wise Gaussian exponential family with sufficient statistics $(\epsilon_{t,i},\epsilon_{t,i}^2)$ and auxiliary-variable-dependent natural parameters, matching the form assumed in the IIA analysis in Appendix~\ref{app:transition_noise_identification}.

\paragraph{Windows used for training and evaluation.}
We remove windows containing extreme observations using a threshold determined solely from the training split. Specifically, we set
\begin{equation}
 \tau=Q_{0.999}\!\left(\left\{|x^{\mathrm{train}}_{t,j}|\right\}_{t,j}\right),
 \label{eq:supp_synthetic_observation_filter}
\end{equation}
and retain windows satisfying $\max_{\ell=0,\ldots,L}\max_j|x_{t-\ell,j}|\leq\tau$. For window-based training and validation, we apply this threshold without using the true latent states. Because LEAP is trained on contiguous sequences, we do not remove individual windows from its training or validation sequences. At test time, however, we evaluate all methods only on windows satisfying the threshold.

\subsubsection{Training Conditions and Baselines}

\paragraph{Model architecture.}
We use the model and objective defined in Appendix~\ref{app:learning-framework-details} with the dimensions and auxiliary variables specified for these experiments. The encoder and learned inverse transition function are five-layer Leaky-ReLU MLPs, and the history feature map is a three-layer Leaky-ReLU MLP; all hidden widths are $4d_z$. The main comparison and all additional analyses except the reconstruction-loss ablation use no decoder or reconstruction loss. \methodTCL{} classifies 256 discrete environments using the Gaussian sufficient statistics $(\hat\epsilon_{t,i},\hat\epsilon_{t,i}^2)$ for each estimated transition-noise component, together with the history-feature vector $\vh_t$ and its element-wise square.
\methodGCL{} applies the same element-wise transformations to the estimated transition noise and history feature. 
For discrete-variation data, each coefficient $\mu_{ik}(u)$, the coefficients multiplying the transformed history features in $\varphi$, and the offset in $\varphi$ are represented by separate parameters for each of the 256 environments.
For continuous-variation data, these coefficients and the offset are linear functions of a Fourier basis consisting of 64 pairs of sine and cosine functions.
Its auxiliary-variable-independent term $B(\hat{\veps}_t,\vh_t)$ in~\eqref{eq:supp_alg_gcl_score} is implemented as the sum of two MLP outputs, one for $\hat{\veps}_t$ and one for $\vh_t$, each with hidden width 32.
Although the theory assumes $C^2$ mappings, the experimental implementation uses Leaky-ReLU activations, which are only piecewise smooth. 
This mismatch at the activation level can be removed by replacing Leaky-ReLU with a smooth, strictly increasing activation applied element-wise.

\paragraph{Optimization.}
We clip the global gradient norm at 5 and set $\lambda_{\mathrm{sparsity}}=0.005$ and $\lambda_{\mathrm{DAG}}=10$.
Because the implementation averages the instantaneous and lagged mean $\ell_1$ quantities, this corresponds to an implementation-level sparsity coefficient of 0.01. 
We use AdamW with weight decay $10^{-5}$, initial learning rate $10^{-2}$, and minibatches of 256 windows, reducing the learning rate to 0.001 from epoch 120. 
We selected these coefficients by inspecting validation-loss trajectories and avoiding values for which overly strong regularization produced degenerate solutions. 
The ablation in Table~\ref{tab:supp_synthetic_loss_ablation} shows that removing either or both penalties does not consistently reduce latent-state or causal-graph recovery in the evaluated synthetic settings.

\paragraph{Training and checkpoint selection.}
All methods are trained independently with the same 10 seeds, and checkpoints are selected using validation objectives that do not use the true latent states or causal graph. Methods with the same seed share the generated data and splits; data generation, model initialization, training, and evaluation are independent across seeds. Each method is trained for at most 150 epochs, and we select the checkpoint minimizing its corresponding validation objective. Neither MCC nor F1 is used for checkpoint selection.

\paragraph{Baselines and their inputs.}
We follow public implementations where possible, adapting network sizes to the present setting and making limited adjustments to optimization settings, including learning rates, based on validation behavior. We evaluate IDOL~\citep{Li2025a}, G-CaRL~\citep{Morioka2024a}, CtrlNS~\citep{Song2024b}, LEAP~\citep{Yao2022b}, iVAE~\citep{Khemakhem2020a}, and $\beta$-VAE~\citep{Higgins2017a}, with $\beta=2$ for $\beta$-VAE. Tables~\ref{tab:supp_synthetic_baseline_optimization} and~\ref{tab:supp_synthetic_baseline_inputs} summarize the final optimization settings and auxiliary variables, respectively.

\begin{table}[tbp]
    \centering
    \caption{Main optimization settings for the baselines in the synthetic experiments.}
    \label{tab:supp_synthetic_baseline_optimization}
    \small
    \begin{tabular}{lp{0.72\textwidth}}
        \hline
        Method & Optimization \\
        \hline
        IDOL & Adam, initial learning rate $10^{-3}$; multiplied by 0.1 at epoch 120 \\
        G-CaRL & SGD, momentum 0.9, initial learning rate $10^{-2}$; multiplied by 0.1 after 80\% of updates \\
        CtrlNS & AdamW, initial learning rate $10^{-3}$ \\
        iVAE / $\beta$-VAE & AdamW, initial learning rate $10^{-2}$; multiplied by 0.1 at epoch 120 \\
        LEAP & AdamW (model/transition) and SGD (discriminator), initial learning rate $2\times10^{-3}$ \\
        \hline
    \end{tabular}
\end{table}

\begin{table}[tbp]
    \centering
    \caption{Auxiliary variables supplied to each method in the synthetic experiments.}
    \label{tab:supp_synthetic_baseline_inputs}
    \begin{tabular}{lcc}
        \hline
        Method & Discrete variation & Continuous variation \\
        \hline
        \methodTCL{} & Discrete environment $u_t$ & 256-bin pseudo-label $\widetilde u_t$ \\
        \methodGCL{} & Discrete environment $u_t$ & Continuous time $t$ \\
        IDOL & None & None \\
        G-CaRL & None & None \\
        CtrlNS & None & None \\
        LEAP & Discrete environment $u_t$ & 256-bin pseudo-label $\widetilde u_t$ \\
        iVAE & Discrete environment $u_t$ & Continuous time $t$ \\
        $\beta$-VAE & None & None \\
        \hline
    \end{tabular}
\end{table}

\paragraph{Method-specific sequence handling.}
For G-CaRL, CtrlNS, and LEAP, we adapt the inputs to the sequence structure of the default data. G-CaRL receives $(\vx_t,\vx_{t-1})$ as current- and past-time groups through a shared encoder, and we report past-to-current lagged F1. CtrlNS receives no auxiliary variable; its gating network estimates at most 256 regimes from consecutive learned representations. The 256 time bins are used only for chunk sampling and are not provided as labels. For LEAP, we compute the total-correlation loss from the residual components of each transition separately, rather than jointly across an entire residual sequence. For continuous-variation data, we divide the single sequence into the 256 pseudo-label bins and present each bin to LEAP as a separate sequence; consequently, transitions across bin boundaries are excluded.

\subsubsection{Evaluation Protocol}
\label{app:synthetic_evaluation_protocol}

\paragraph{Latent-representation and causal-graph evaluation.}
We choose the permutation of the learned latent components that maximizes the mean absolute correlation with the true components and use the same permutation to order the learned edge scores before computing F1. The resulting MCC is
\begin{equation}
 \operatorname{MCC}
 =\max_{\pi\in\mathfrak S_{d_z}}
 \frac{1}{d_z}\sum_{i=1}^{d_z}
 \left|\operatorname{corr}\!\left(\vz_{:,i},\hat{\vz}_{:,\pi(i)}\right)\right|,
 \label{eq:supp_synthetic_mcc}
\end{equation}
where $\mathfrak S_{d_z}$ is the set of permutations of the $d_z$ components, $\operatorname{corr}$ denotes the sample Pearson correlation, and $\vz_{:,i}$ and $\hat{\vz}_{:,i}$ collect the values of component $i$ over the evaluated test samples. Let $\widehat{\mathbb E}_{\mathrm{test}}$ denote the empirical mean over the evaluated test windows. For \methodname{}, we define the empirical edge scores as
\begin{equation}
 \begin{aligned}
 \mS^{(0)}_{ij}
 &=\widehat{\mathbb E}_{\mathrm{test}}
 \left[
 \left|\frac{\partial\hat r_i}{\partial\hat z_{t,j}}\right|
 \right],
 && i\ne j,\\
 \mS^{(\ell)}_{ij}
 &=\widehat{\mathbb E}_{\mathrm{test}}
 \left[
 \left|\frac{\partial\hat r_i}{\partial\hat z_{t-\ell,j}}\right|
 \right],
 && \ell=1,\ldots,L,
 \end{aligned}
 \label{eq:supp_synthetic_finite_sample_graph_scores}
\end{equation}
and use the edge scores produced by each baseline. Here, $t$ denotes the target time of each evaluated window. Diagonal entries of $\mS^{(0)}$ are excluded from instantaneous edges, and rows and columns index targets and sources, respectively. Using the row-to-target convention of the learned inverse transition function, we apply the MCC-based latent-component permutation to the target and source axes of the empirical graph scores. We then select for each target the same number of highest-scoring sources as its true number of parents. This evaluation depends on the ranking of the edge scores rather than their absolute scale, which is not invariant to component-wise invertible transformations. From this binary support and the true support, we compute
\begin{equation}
 \operatorname{F1}
 =\frac{2\mathrm{TP}}{2\mathrm{TP}+\mathrm{FP}+\mathrm{FN}}.
 \label{eq:supp_synthetic_graph_f1}
\end{equation}
Here, $\mathrm{TP}$, $\mathrm{FP}$, and $\mathrm{FN}$ denote the numbers of true-positive, false-positive, and false-negative edges, respectively.
We report instantaneous F1 and lagged F1, abbreviating the former as Inst. F1 in tables and figures. 
True latent states and true parent counts are not used for training or checkpoint selection.
For G-CaRL, we separately align the current- and past-time groups and apply the two permutations to the target and source axes of its lagged-edge scores. 
Tables and figures report the mean $\pm$ standard deviation over 10 seeds; \texttt{--} marks metrics not applicable to the method. Figure~\ref{fig:supp_synthetic_jacobian_scores} illustrates two seeds that recover the same support despite having different score scales.

\begin{figure}[tbp]
    \centering
    \textbf{(a)} Instantaneous\\
    \includegraphics[width=0.72\textwidth]{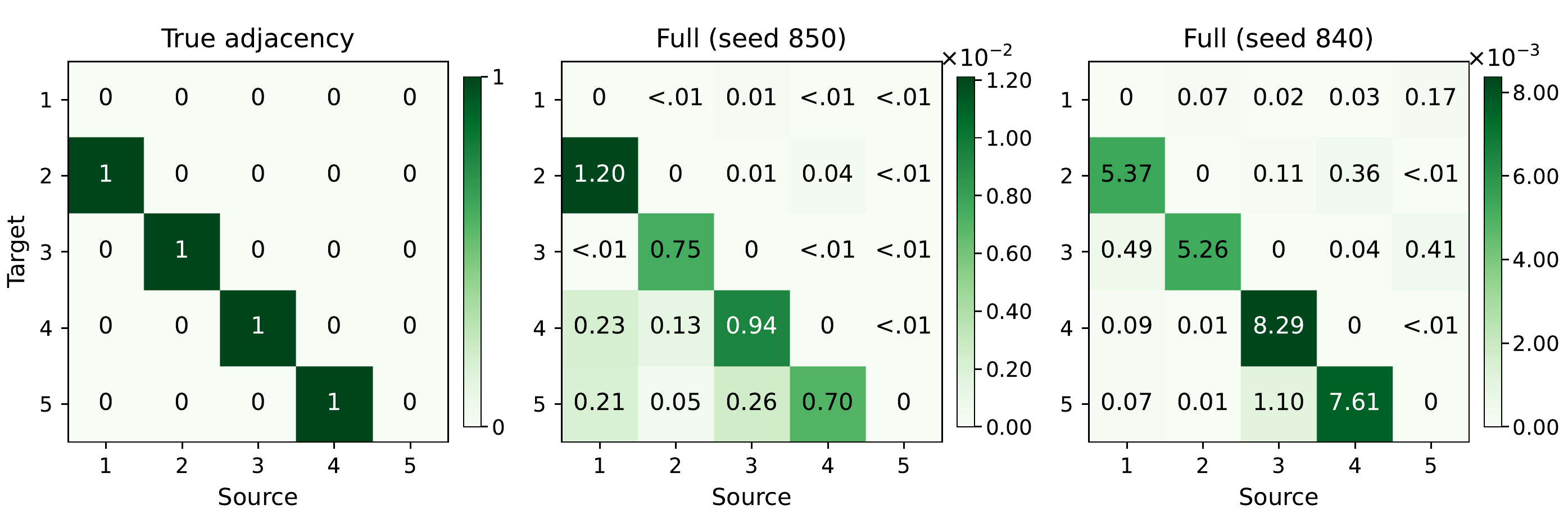}\\
    \textbf{(b)} Lagged\\
    \includegraphics[width=0.72\textwidth]{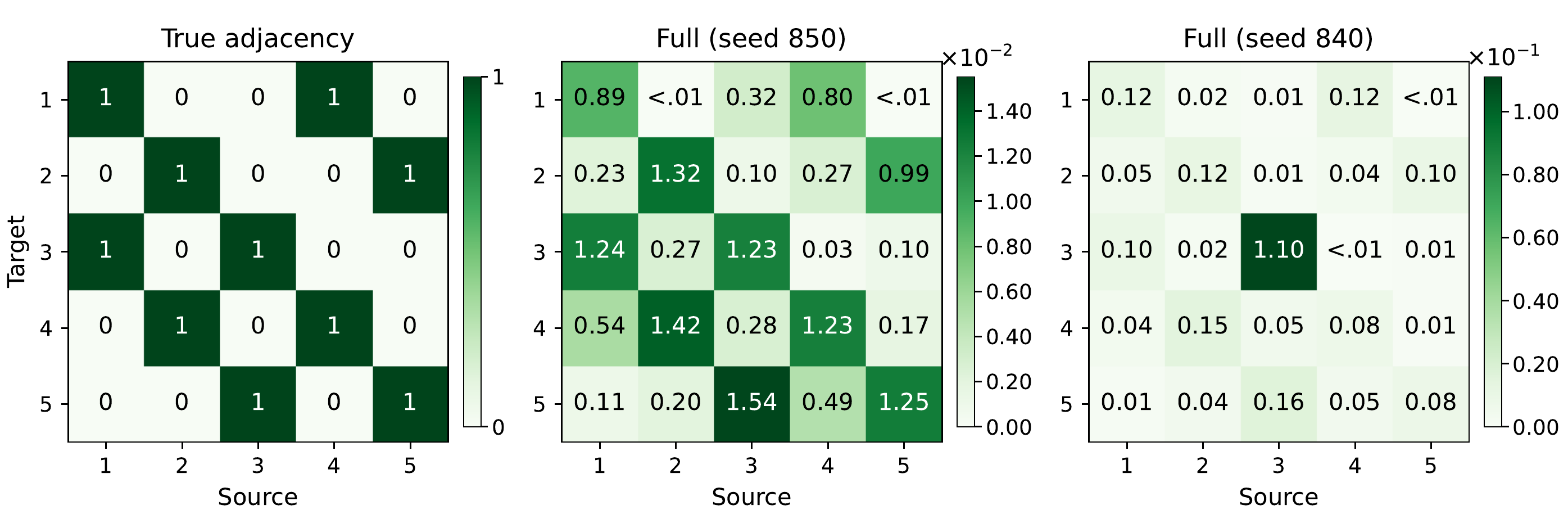}
    \caption{Jacobian edge scores in the default setting with Gaussian transition-noise distributions that vary across discrete environments: (a) instantaneous and (b) lagged. In each panel, the left matrix is the true support, and the middle and right matrices show seeds 850 and 840 trained with the full objective. Both seeds recover the same support, although their score scales differ under component-wise invertible transformations. Rows and columns index targets and sources, respectively.}
    \label{fig:supp_synthetic_jacobian_scores}
\end{figure}

\subsection{Sensitivity to Latent Causal Structure}
\label{app:synthetic_causal_structure_sensitivity}

To examine sensitivity to the latent causal structure, we evaluate all four combinations of two instantaneous and two lagged graphs.
For the instantaneous graph, we compare the default chain with ten distinct randomly generated DAGs, one for each seed.
Each random DAG has five nodes and four edges, respects the fixed component order, is connected but not chain-shaped when edge directions are ignored, and contains at least one fork and one collider.
For the lagged graph, we compare the default pattern, in which each component has two lagged parents, with a graph containing only the lagged self-loop $z_{t-1,i}\to z_{t,i}$ for each component $i$.
Because this replacement reduces the number of lagged parents per component from two to one, we double the weight assigned to the retained self-loop input in each true structural function. 
Table~\ref{tab:supp_synthetic_causal_structure_sensitivity} shows that \methodTCL{} attains high MCC and F1 scores across all four combinations.

\begin{table}[tbp]
    \centering
    \caption{Latent-state and causal-graph recovery by \methodTCL{} under different instantaneous and lagged causal structures on synthetic data with Gaussian transition-noise distributions that vary across discrete environments. Random DAGs are distinct across seeds. Values are the mean $\pm$ standard deviation over 10 seeds.}
    \label{tab:supp_synthetic_causal_structure_sensitivity}
    \small
    \begin{tabular}{llccc}
        \hline
        Instant. graph & Lagged graph & MCC & Inst. F1 & Lagged F1 \\
        \hline
        Default chain & Default    & $0.951\pm0.011$ & $0.950\pm0.158$ & $0.980\pm0.063$ \\
        Random DAGs   & Default    & $0.960\pm0.008$ & $1.000\pm0.000$ & $0.970\pm0.067$ \\
        Default chain & Self-loops & $0.955\pm0.009$ & $0.925\pm0.169$ & $0.940\pm0.135$ \\
        Random DAGs   & Self-loops & $0.948\pm0.032$ & $0.975\pm0.079$ & $0.960\pm0.126$ \\
        \hline
    \end{tabular}
\end{table}

\subsection{Sensitivity to Latent Dimension}
\label{app:synthetic_dimension_sensitivity}

We vary the latent dimension of discrete-variation data over $d_x=d_z\in\{4,5,8,12,16\}$ while fixing the total sample size at $2^{16}$. 
For $d_z=4$, the lagged graph contains only self-loops; otherwise, we use the default instantaneous and lagged graphs. 
Table~\ref{tab:supp_synthetic_dimension_sweep} shows that MCC and F1 scores generally decrease as $d_z$ increases, with the most pronounced decrease in instantaneous F1.

\begin{table}[tbp]
    \centering
    \caption{Latent-state and causal-graph recovery by \methodTCL{} as the latent dimension varies on synthetic data with Gaussian transition-noise distributions that vary across discrete environments. Values are the mean $\pm$ standard deviation over 10 seeds.}
    \label{tab:supp_synthetic_dimension_sweep}
    \small
    \begin{tabular}{rccc}
        \hline
        $d_z$ & MCC & Inst. F1 & Lagged F1 \\
        \hline
        4  & $0.958\pm0.009$ & $1.000\pm0.000$ & $1.000\pm0.000$ \\
        5  & $0.951\pm0.011$ & $0.950\pm0.158$ & $0.980\pm0.063$ \\
        8  & $0.860\pm0.021$ & $0.743\pm0.284$ & $0.869\pm0.127$ \\
        12 & $0.886\pm0.017$ & $0.527\pm0.200$ & $0.812\pm0.118$ \\
        16 & $0.792\pm0.056$ & $0.140\pm0.058$ & $0.569\pm0.136$ \\
        \hline
    \end{tabular}
\end{table}

\subsection{Sensitivity to Total Sample Size}
\label{app:synthetic_sample_size_sensitivity}

We vary the total sample size from $2^{13}$ to $2^{18}$ while retaining the respective discrete-variation and continuous-variation generation schemes. 
For discrete-variation data, we fix 256 environments and vary the sequence length per environment, whereas for continuous-variation data we vary the length of the single continuous sequence. 
Figure~\ref{fig:supp_synthetic_sample_count_sweep} shows that MCC and F1 generally improve with sample size for both variants under both data-generation schemes.
For both data-generation schemes, \methodGCL{} is more sensitive to sample size than \methodTCL{}, particularly in instantaneous F1. 
At $N=2^{16}$, its instantaneous F1 is $0.700\pm0.158$ for discrete variation and $0.700\pm0.197$ for continuous variation.
Increasing $N$ to $2^{18}$ raises both values to $0.975\pm0.079$, indicating that the lower discrete-variation result at the default sample size is partly a finite-sample effect.

\begin{figure}[tbp]
    \centering
    \includegraphics[width=\textwidth]{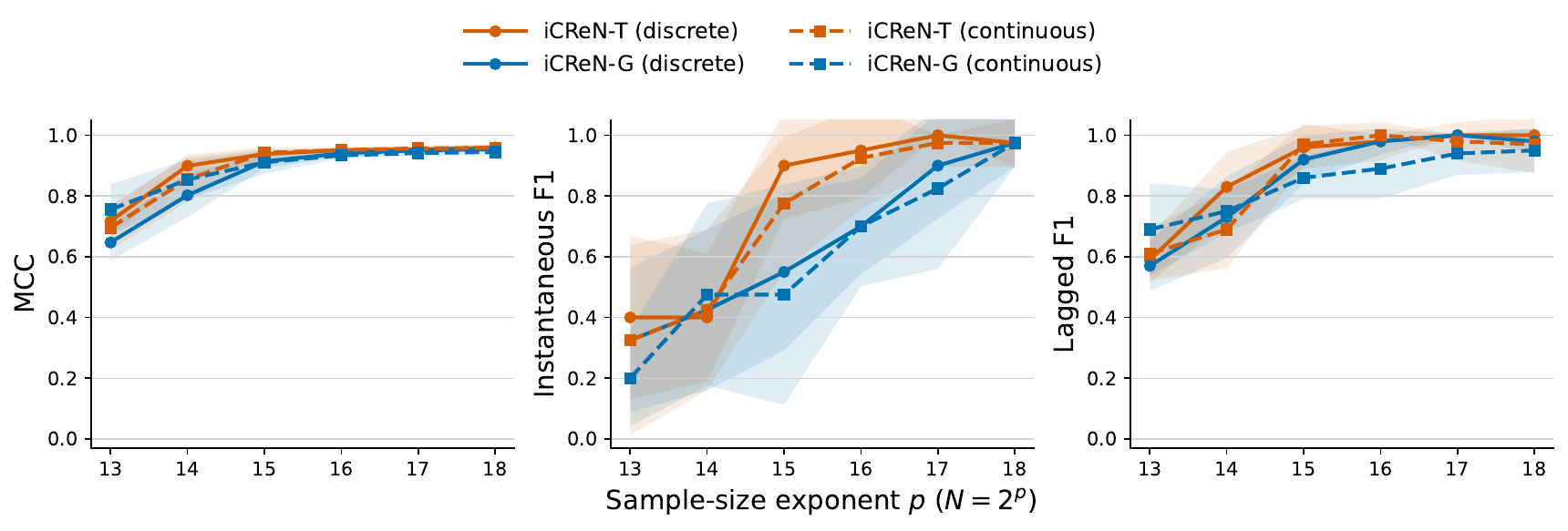}
    \caption{Latent-state and causal-graph recovery as total sample size varies on synthetic data with Gaussian transition-noise distributions. Orange and blue denote \methodTCL{} and \methodGCL{}, respectively; solid curves with circular markers and dashed curves with square markers denote discrete variation across environments and continuous variation over time, respectively. Shaded regions show mean $\pm$ standard deviation over 10 seeds.}
    \label{fig:supp_synthetic_sample_count_sweep}
\end{figure}

\subsection{Sensitivity to Data-generating Function Complexity}
\label{app:synthetic_mapping_complexity}

\begin{table}[tbp]
    \centering
    \caption{Latent-state and causal-graph recovery by \methodTCL{} as the complexity of the true data-generating functions varies on synthetic data with Gaussian transition-noise distributions that vary across discrete environments. We vary structural-function depth with a fixed three-layer observation function in (a), and observation-function depth with fixed three-layer structural functions in (b). Values are the mean $\pm$ standard deviation over 10 seeds.}
    \label{tab:supp_synthetic_mapping_complexity}
    \small
    \begin{tabular}{rccc}
        \hline
        \multicolumn{4}{c}{(a) Structural-function depth} \\
        Depth & MCC & Inst. F1 & Lagged F1 \\
        \hline
        1 & $0.913\pm0.012$ & $0.275\pm0.275$ & $1.000\pm0.000$ \\
        2 & $0.953\pm0.007$ & $0.975\pm0.079$ & $1.000\pm0.000$ \\
        3 & $0.951\pm0.011$ & $0.950\pm0.158$ & $0.980\pm0.063$ \\
        4 & $0.952\pm0.007$ & $0.925\pm0.237$ & $0.950\pm0.127$ \\
        5 & $0.955\pm0.006$ & $1.000\pm0.000$ & $1.000\pm0.000$ \\
        \hline
        \multicolumn{4}{c}{(b) Observation-function depth} \\
        Depth & MCC & Inst. F1 & Lagged F1 \\
        \hline
        1 & $0.957\pm0.008$ & $1.000\pm0.000$ & $1.000\pm0.000$ \\
        2 & $0.953\pm0.007$ & $0.975\pm0.079$ & $1.000\pm0.000$ \\
        3 & $0.951\pm0.011$ & $0.950\pm0.158$ & $0.980\pm0.063$ \\
        4 & $0.946\pm0.011$ & $0.950\pm0.105$ & $1.000\pm0.000$ \\
        5 & $0.930\pm0.008$ & $0.875\pm0.177$ & $0.980\pm0.063$ \\
        \hline
    \end{tabular}
\end{table}

We vary the depth of the true structural or observation function to examine its effect on latent-state and causal-graph recovery by \methodTCL{}. 
We use the default dimensions $d_x=d_z=5$ and the default instantaneous and lagged graphs. 
First, we fix the observation function at three layers and vary the structural-function depth from one to five layers. One layer gives linear structural functions, whereas two to five layers are Leaky-ReLU MLPs with increasing numbers of hidden layers; all conditions use the same 10 seeds. 
As shown in Table~\ref{tab:supp_synthetic_mapping_complexity}(a), MCC and both F1 scores remain high for the nonlinear two- to five-layer structural functions. For the linear functions, instantaneous F1 decreases to $0.275\pm0.275$, whereas MCC and lagged F1 remain high. 
An additively separable linear transition cannot satisfy the sufficient-variability rank condition when instantaneous edges are present (Remark~\ref{remark:supp_instantaneous_additive_separation_rank_bound}). 
The reduced instantaneous F1 is consistent with this limitation.

Next, we fix the structural functions at three layers and vary the observation-function depth from one to five layers. 
Table~\ref{tab:supp_synthetic_mapping_complexity}(b) shows that \methodTCL{} has mean MCC and instantaneous F1 that tend to decrease as the observation-function depth increases, whereas lagged F1 shows no clear trend. 
We also compare \methodTCL{}, \methodGCL{}, and the baselines using one- and three-layer observation functions while keeping the structural functions at three layers.
Within each seed, the two settings share the same latent states, transition noise, environment labels, and time indices.
Table~\ref{tab:supp_synthetic_observation_depth} shows higher mean MCC at one layer for every evaluated method, with the largest increases for IDOL and LEAP; both proposed methods have higher mean MCC than all baselines at both depths.

\begin{table}[tbp]
    \centering
    \caption{MCC of both proposed methods and the baselines with one- and three-layer observation functions on discrete-variation synthetic data. Each entry is the mean $\pm$ standard deviation over 10 seeds.}
    \label{tab:supp_synthetic_observation_depth}
    \small
    \begin{tabular}{lcc}
        \hline
        Method & One layer & Three layers \\
        \hline
        \methodTCL{} & $0.957\pm0.008$ & $0.951\pm0.011$ \\
        \methodGCL{} & $0.947\pm0.010$ & $0.940\pm0.010$ \\
        % \hline
        IDOL & $0.704\pm0.061$ & $0.468\pm0.046$ \\
        G-CaRL & $0.775\pm0.060$ & $0.764\pm0.065$ \\
        CtrlNS & $0.503\pm0.050$ & $0.432\pm0.071$ \\
        LEAP & $0.536\pm0.063$ & $0.373\pm0.107$ \\
        iVAE & $0.264\pm0.037$ & $0.230\pm0.106$ \\
        $\beta$-VAE & $0.582\pm0.058$ & $0.482\pm0.067$ \\
        \hline
    \end{tabular}
\end{table}

\subsection{Ablation of Objective Components}

\paragraph{Structural regularization.}

We remove the sparsity and DAG regularizers separately or jointly to assess their empirical contributions to latent-state and causal-graph recovery. 
On discrete-variation data with the default $d_z=5$ and with $d_z=8$, where recovery begins to deteriorate, we compare the full objective with three ablations: No sparsity with $\lambda_{\mathrm{sparsity}}=0$, No DAG with $\lambda_{\mathrm{DAG}}=0$, and Neither with both coefficients set to zero. 
The results in Table~\ref{tab:supp_synthetic_loss_ablation} show that removing either or both regularizers does not consistently lower MCC, instantaneous F1, or lagged F1 in the evaluated settings, although the regularizers are motivated by the structural conditions used in the theory.

\begin{table}[tbp]
    \centering
    \caption{Ablation of structural regularization on synthetic data with Gaussian transition-noise distributions that vary across discrete environments, with $d_z=5$ (left) and $d_z=8$ (right). Values are the mean $\pm$ standard deviation over 10 seeds.}
    \label{tab:supp_synthetic_loss_ablation}
    \scriptsize
    \setlength{\tabcolsep}{2pt}
    \resizebox{\textwidth}{!}{%
    \begin{tabular}{lccc|ccc}
        \hline
        & \multicolumn{3}{c|}{$d_z=5$} & \multicolumn{3}{c}{$d_z=8$} \\
        \cline{2-7}
        Objective & MCC & Inst. F1 & Lagged F1 & MCC & Inst. F1 & Lagged F1 \\
        \hline
        Full & $0.951\pm0.011$ & $0.950\pm0.158$ & $0.980\pm0.063$ & $0.860\pm0.021$ & $0.743\pm0.284$ & $0.869\pm0.127$ \\
        No sparsity & $0.955\pm0.007$ & $1.000\pm0.000$ & $1.000\pm0.000$ & $0.867\pm0.016$ & $0.900\pm0.191$ & $0.950\pm0.097$ \\
        No DAG & $0.957\pm0.004$ & $0.950\pm0.105$ & $0.940\pm0.135$ & $0.861\pm0.027$ & $0.786\pm0.205$ & $0.894\pm0.135$ \\
        Neither & $0.955\pm0.008$ & $0.925\pm0.121$ & $0.890\pm0.191$ & $0.861\pm0.017$ & $0.800\pm0.204$ & $0.800\pm0.164$ \\
        \hline
    \end{tabular}%
    }
\end{table}

\paragraph{Reconstruction loss.}

We add a decoder and reconstruction loss to the default objective to assess how observation reconstruction affects latent-state and causal-graph recovery. 
For the default discrete-variation data, let $\hat g$ be a three-layer decoder from the learned representation to the observation space. 
For a minibatch $\mathcal B$ of observation windows, define
\begin{equation}
 \mathcal L_{\mathrm{rec}}
 =\frac{1}{|\mathcal B|}\sum_{b\in\mathcal B}\sum_{\ell=0}^{L}
 \left\|\vx^{(b)}_{t-\ell}
 -\hat g\!\left(\phi\!\left(\vx^{(b)}_{t-\ell}\right)\right)\right\|_2^2.
 \label{eq:supp_synthetic_reconstruction_loss}
\end{equation}
We compare the full objective with Full + reconstruction, which adds $\mathcal L_{\mathrm{rec}}$ with $\lambda_{\mathrm{rec}}=1$. 
Table~\ref{tab:supp_synthetic_reconstruction_ablation} shows that adding reconstruction yields similarly high MCC and F1 scores but no consistent improvement across the three metrics.

\begin{table}[tbp]
    \centering
    \caption{Effect of adding a decoder and reconstruction loss to the default objective on synthetic data with Gaussian transition-noise distributions that vary across discrete environments. Values are the mean $\pm$ standard deviation over 10 seeds.}
    \label{tab:supp_synthetic_reconstruction_ablation}
    \small
    \begin{tabular}{lcccc}
        \hline
        Objective & MCC & Inst. F1 & Lagged F1 & Reconstruction loss \\
        \hline
        Full & $0.951\pm0.011$ & $0.950\pm0.158$ & $0.980\pm0.063$ & \texttt{--} \\
        Full + reconstruction & $0.957\pm0.009$ & $1.000\pm0.000$ & $0.970\pm0.048$ & $0.0022\pm0.0005$ \\
        \hline
    \end{tabular}
\end{table}

\section{Real-world Experimental Details}
\label{app:additional-results}

We use \methodTCL{} in both real-world experiments because it generally achieved higher mean recovery scores in the synthetic comparisons.
Its classifier applies the element-wise map $a\mapsto(a,a^2,\log(1+a^2))$ to both the estimated transition noise and the history feature.
Inputs are standardized using statistics computed only from the training data, and each method's pretraining checkpoint is selected using its method-specific validation objective.
After pretraining, each encoder is frozen, and a forecasting head with the same architecture across methods is trained, selected using validation data, and evaluated on test data.
Each random seed defines one complete run, including representation pretraining, training of the forecasting head, and evaluation.

\subsection{Treadmill Motion Forecasting}
\label{app:supp_treadmill_motion_forecasting}

\paragraph{Dataset and preprocessing.}
We use the treadmill walking data provided by \citet{Riglet2024}, including sequences recorded at Slow, Comfortable, and Fast walking speeds. 
From the three-dimensional coordinates of markers on the left and right sides of the pelvis, thigh, knee, ankle, heel, and toe, we construct 36-dimensional sequences. 
We align the coordinate axes, convert the coordinates from millimeters to meters, center each frame by subtracting the midpoint of the left and right pelvic markers, and retain every eighth frame of the original 100-Hz trajectories, yielding a sampling rate of 12.5~Hz. 
We split the 30 participants into 19 for training, 6 for validation, and 5 for testing. 
We estimate the normalization statistics using only sequences recorded at the Slow and Fast walking speeds from the training participants and apply them unchanged to validation and test data. 
We use Slow and Fast for training and model selection and Comfortable only for final test evaluation.

\begin{table}[tbp]
    \centering
    \caption{Main pretraining objectives and auxiliary information for the treadmill motion forecasting experiment. KL denotes Kullback--Leibler.}
    \label{tab:supp_treadmill_method_settings}
    \small
    \begin{tabular}{lp{0.47\textwidth}p{0.25\textwidth}}
        \hline
        Method & Pretraining objective & Auxiliary information \\
        \hline
        \methodname{} & Reconstruction, TCL, and sparsity regularization & 12-class gait phase \\
        IDOL & Reconstruction, KL divergence, and sparsity regularization & None \\
        G-CaRL & Group-contrastive objective & None \\
        CtrlNS & Variational transition objective & No observed label; at most five inferred regimes \\
        LEAP & Variational transition, total correlation, and sparsity regularization & 12-class gait phase \\
        $\beta$-VAE & Reconstruction and $\beta$-weighted KL divergence & None \\
        \hline
    \end{tabular}
\end{table}

\paragraph{Pretraining.}
We pretrain each method for at most 20,000 steps with latent dimension 5 and batch size 256. 
Following the motion-capture setting of IDOL~\citep{Li2025a}, we set the maximum lag to one for all dynamic methods. 
Using the right-foot-strike events provided with the dataset, we define each gait cycle as the interval between consecutive strikes, divide it uniformly into 12 phases, and use the phase index as the auxiliary variable for \methodname{} and LEAP. 
We fit each representation model on the Slow and Fast walking speeds from the 19 training participants and select its checkpoint on the same walking speeds from the six validation participants using the method-specific validation objective. 
For \methodname{}, a frame-wise decoder reconstructs the 36-dimensional marker coordinates from each latent representation, and the pretraining objective combines reconstruction and TCL classification losses with sparsity regularization on the input gates, using $\lambda_{\mathrm{rec}}=1$, $\lambda_{\mathrm{IIA}}=0.001$, and $\lambda_{\mathrm{sparsity}}=5\times10^{-5}$.
Because we use a fixed lower-triangular mask for the instantaneous structure, no DAG regularization is applied in this experiment. 
Two LEAP runs produced non-finite losses or gradients during pretraining and were therefore excluded from downstream evaluation without replacement; we report the remaining eight runs. 
Table~\ref{tab:supp_treadmill_method_settings} summarizes the pretraining objectives and auxiliary information.

\begin{table}[tbp]
    \centering
    \caption{Mean marker error (mm) at every third predicted frame for Comfortable test sequences, which are excluded from training and model selection. Values are the mean $\pm$ standard deviation across runs; boldface indicates the best value in each column. We report 10 runs for \methodname{}, IDOL, G-CaRL, CtrlNS, and $\beta$-VAE and the eight numerically stable runs for LEAP.}
    \label{tab:supp_treadmill_per_horizon}
    \footnotesize
    \setlength{\tabcolsep}{4pt}
    \begin{tabular}{lcccc}
        \hline
        \multicolumn{5}{c}{(a) Early prediction horizons} \\
        Prediction horizon [s] & 0.24 & 0.48 & 0.72 & 0.96 \\
        \hline
        \methodname{} & $\mathbf{35.518\pm1.236}$ & $\mathbf{42.060\pm1.939}$ & $\mathbf{47.863\pm1.255}$ & $\mathbf{58.650\pm2.051}$ \\
        IDOL & $36.654\pm2.106$ & $44.076\pm2.304$ & $50.717\pm1.786$ & $62.400\pm2.119$ \\
        G-CaRL & $45.942\pm4.276$ & $52.197\pm4.917$ & $60.568\pm3.640$ & $75.790\pm7.785$ \\
        CtrlNS & $37.152\pm1.859$ & $43.593\pm2.267$ & $48.861\pm1.610$ & $62.337\pm1.978$ \\
        LEAP & $37.281\pm1.326$ & $44.220\pm1.825$ & $50.902\pm1.568$ & $63.431\pm2.825$ \\
        $\beta$-VAE & $36.429\pm0.877$ & $43.391\pm1.865$ & $51.058\pm2.257$ & $63.305\pm2.546$ \\
        \hline
    \end{tabular}

    \vspace{0.5em}

    \begin{tabular}{lcccc}
        \hline
        \multicolumn{5}{c}{(b) Later prediction horizons} \\
        Prediction horizon [s] & 1.20 & 1.44 & 1.68 & 1.92 \\
        \hline
        \methodname{} & $\mathbf{67.673\pm2.831}$ & $\mathbf{82.034\pm2.725}$ & $\mathbf{95.912\pm2.611}$ & $\mathbf{107.309\pm4.345}$ \\
        IDOL & $71.989\pm2.691$ & $86.791\pm2.547$ & $102.094\pm3.543$ & $115.146\pm3.822$ \\
        G-CaRL & $86.037\pm6.819$ & $103.133\pm9.088$ & $118.191\pm10.616$ & $135.169\pm8.541$ \\
        CtrlNS & $70.025\pm1.759$ & $86.040\pm4.300$ & $101.372\pm3.605$ & $117.059\pm7.501$ \\
        LEAP & $70.623\pm2.281$ & $87.338\pm4.548$ & $99.185\pm3.018$ & $113.210\pm7.738$ \\
        $\beta$-VAE & $73.886\pm2.671$ & $91.089\pm4.090$ & $106.227\pm4.279$ & $123.960\pm5.890$ \\
        \hline
    \end{tabular}
\end{table}

\paragraph{Downstream forecasting.}
After pretraining, we freeze each selected encoder and train the same two-layer forecasting head with hidden width 64 and Gaussian error linear unit (GELU) activations for all methods to directly predict the marker coordinates for the following 24 frames (1.92 seconds) from the latent representations of the preceding 12 frames (0.96 seconds). 
For downstream learning, we further divide the 19 training participants into 15 for fitting the forecasting head and 4 for selecting its checkpoint; both subsets contain only the Slow and Fast walking speeds. 
We train the head for at most 100,000 steps using Adam with learning rate $10^{-3}$ and batch size 256 and select the checkpoint with the lowest mean marker error averaged over all 24 prediction horizons on the four-participant selection subset. 
We then evaluate the selected encoder and forecasting head on the Comfortable walking speed from the five test participants without further parameter updates or model selection.

\paragraph{Evaluation metric.}
Before computing the error, we return the standardized predictions and targets to the physical coordinate scale using the training-set normalization statistics. 
Let $n=1,\ldots,N$ index test windows, $h=1,\ldots,H$ index prediction horizons, and $j=1,\ldots,J$ index markers, where $N=7{,}214$, $H=24$, and $J=12$. 
Let $\vy_{n,h,j}$ and $\hat{\vy}_{n,h,j}$ denote the true and predicted three-dimensional marker coordinates, respectively. 
We compute the mean marker error at horizon $h$ as
\begin{equation}
 \frac{1}{NJ}\sum_{n=1}^{N}\sum_{j=1}^{J}
 \left\|\vy_{n,h,j}-\hat{\vy}_{n,h,j}\right\|_2.
 \label{eq:supp_treadmill_marker_error}
\end{equation}
We convert the resulting distances from meters to millimeters and report every third prediction horizon, from $h=3$ (0.24 seconds) to $h=24$ (1.92 seconds). 
As shown in Table~\ref{tab:supp_treadmill_per_horizon}, \methodname{} attains the lowest mean marker error at every reported horizon.

\subsection{Weather Forecasting}
\label{app:supp_weather_forecasting}

\paragraph{Dataset and preprocessing.}
We construct regional weather time series from the ERA5 global climate reanalysis dataset~\citep{Hersbach2020,Copernicus2023} for the eight regions listed in Table~\ref{tab:supp_weather_cities}. 
For each of the five cities in a region, we use the nearest $0.25^\circ$ grid point and sample the data at three-hour intervals. 
At each time point, the encoder input comprises 2-m temperature (\texttt{t2m}), mean sea-level pressure (\texttt{msl}), 10-m zonal and meridional wind components (\texttt{u10} and \texttt{v10}), and total precipitation over the preceding three hours (\texttt{tp}) for all five cities. 
We convert the three-hour precipitation totals to millimeters and apply a \texttt{log1p} transformation. 
We use only windows consisting of consecutive three-hour observations. Within each month of the 1981--2013 pretraining period, we divide the days into four periods (1--7, 8--14, 15--21, and 22 onward) and reserve the final day of each period for validation. 
The preceding $L$ time points used as context for these validation windows are also excluded from training and from the computation of normalization statistics.
We standardize each input variable separately for each city using only the remaining 1981--2013 training samples and apply the resulting statistics unchanged to the pretraining validation and later downstream data.

\paragraph{Pretraining.}
We use a calendar index as the auxiliary variable during pretraining. 
The index combines month (12 categories), the four within-month periods defined above, and UTC time (8 three-hour categories), yielding 384 classes.
We exclude February 29 and do not include the year, so the same 384 classes recur across years. 
Only \methodname{} and LEAP use this observed auxiliary variable during pretraining; the other methods receive no calendar label. 
We pretrain each method for at most 10,000 steps on data from 1981--2013, use latent dimension 5, and set the maximum lag to 8 for the dynamic methods, covering the preceding 24 hours at the three-hour sampling interval.
For \methodname{}, the pretraining objective combines reconstruction and TCL classification losses with sparsity regularization on the input gates, using $\lambda_{\mathrm{rec}}=1$, $\lambda_{\mathrm{IIA}}=0.005$, and $\lambda_{\mathrm{sparsity}}=5\times10^{-5}$.
A fixed lower-triangular mask is used for the instantaneous structure, so no DAG penalty contributes to the objective.
Table~\ref{tab:supp_weather_method_settings} summarizes the pretraining objectives and auxiliary information.

\begin{table}[bp]
    \centering
    \caption{Cities used in the weather forecasting experiment. We use the ERA5 $0.25^\circ$ grid point nearest to each city center.}
    \label{tab:supp_weather_cities}
    \small
    \begin{tabular}{lp{0.66\textwidth}}
        \hline
        Region & Cities \\
        \hline
        Southeastern Australia & Brisbane, Sydney, Canberra, Melbourne, Hobart \\
        Canada & Toronto, Kingston, Ottawa, Montreal, Quebec City \\
        China & Harbin, Beijing, Shanghai, Wuhan, Guangzhou \\
        France & Lille, Paris, Strasbourg, Bordeaux, Marseille \\
        Germany & Hamburg, Berlin, Cologne, Frankfurt, Munich \\
        Japan & Sapporo, Sendai, Tokyo, Osaka, Fukuoka \\
        South Korea & Seoul, Daejeon, Daegu, Gwangju, Busan \\
        US Atlantic & Boston, New York, Washington, D.C., Atlanta, Miami \\
        \hline
    \end{tabular}
\end{table}

\paragraph{Downstream forecasting.}
After pretraining, we freeze each selected encoder and train the same two-layer forecasting head with hidden width 256, GELU activations, and dropout 0.1 for all methods. 
The head uses the latent representations at the preceding eight time points to predict the relative humidity of all five cities three hours ahead. 
We fit the head on data from 2014--2016 for at most 5,000 steps using Adam with learning rate $10^{-3}$ and select its checkpoint by standardized MSE on 2017 validation data. 
We then evaluate the selected encoder and head on data from 2018--2020 without further parameter updates or model selection. 
The calendar index is not provided to the forecasting head.

\paragraph{Target and evaluation metric.}
We derive relative humidity, which is not used during pretraining, from 2-m temperature and dew-point temperature and evaluate the three-hour-ahead forecasts using standardized MSE. 
Let $T$ and $T_d$ denote 2-m temperature and dew-point temperature in kelvin, and define their values in degrees Celsius as $\tau=T-273.15$ and $\tau_d=T_d-273.15$. 
Using Bolton's approximation of saturation vapor pressure~\citep{Bolton1980}, we compute
\begin{equation}
 \operatorname{RH}(T,T_d)
 =\operatorname{clip}_{[0,100]}\!\left[
 100\exp\!\left(
 \frac{17.67\tau_d}{\tau_d+243.5}
 -\frac{17.67\tau}{\tau+243.5}
 \right)
 \right].
 \label{eq:supp_weather_rh}
\end{equation}
Dew-point temperature is used only for this calculation and is not provided to the encoder. 
We standardize relative humidity separately for each city using statistics from the 1981--2013 pretraining period. 
Let $n=1,\ldots,N$ index test time points and $c=1,\ldots,C$ index cities, and denote the standardized true and predicted relative humidity by $y_{n,c}$ and $\hat y_{n,c}$, respectively. 
We compute
\begin{equation}
 \operatorname{MSE}_{\mathrm{std}}
 =\frac{1}{NC}\sum_{n=1}^{N}\sum_{c=1}^{C}
 \left(y_{n,c}-\hat y_{n,c}\right)^2.
 \label{eq:supp_weather_standardized_mse}
\end{equation}
For each run, we evaluate the metric over all test time points and cities from 2018--2020. 
Table~\ref{tab:supp_weather_results_by_region} reports the results. 
We also report persistence, which uses the most recently observed relative humidity as the forecast. 
Because the learned methods do not receive relative humidity as an input, we treat persistence as a reference and exclude it from the boldface comparison.

\begin{table}[tbp]
    \centering
    \caption{Main pretraining objectives and auxiliary information for the weather forecasting experiment.}
    \label{tab:supp_weather_method_settings}
    \small
    \begin{tabular}{lp{0.45\textwidth}p{0.27\textwidth}}
        \hline
        Method & Pretraining objective & Auxiliary information \\
        \hline
        \methodname{} & Reconstruction, TCL, and sparsity regularization & 384-class calendar index \\
        IDOL & Reconstruction, KL divergence, and sparsity regularization & None \\
        G-CaRL & Group-contrastive objective & None \\
        CtrlNS & Variational transition objective & No observed label; at most five inferred regimes \\
        LEAP & Variational transition, total correlation, and sparsity regularization & 384-class calendar index \\
        % LEAP & Variational transition, innovation density, and sparsity regularization & 384-class calendar index \\
        $\beta$-VAE & Reconstruction and $\beta$-weighted KL divergence & None \\
        \hline
    \end{tabular}
\end{table}

\begin{table}[tbp]
    \centering
    \caption{Standardized MSE for three-hour-ahead relative-humidity forecasting in eight regions. Each region comprises five cities. Values for the learned methods are the mean $\pm$ standard deviation over 10 runs; boldface indicates the best learned method in each column. The two panels partition the columns only for readability. Persistence is a deterministic forecasting reference and is shown without a standard deviation.}
    \label{tab:supp_weather_results_by_region}
    \small
    \setlength{\tabcolsep}{5pt}
    \begin{tabular}{l|cccc}
        \hline
        \multicolumn{5}{c}{(a) Southeastern Australia, Canada, China, and France} \\
        Method & AU & CAN & CN & FR \\
        \hline
        \methodname{} & $\mathbf{0.352\pm0.004}$ & $0.481\pm0.011$ & $\mathbf{0.546\pm0.010}$ & $\mathbf{0.422\pm0.015}$ \\
        IDOL & $0.359\pm0.004$ & $0.484\pm0.010$ & $0.555\pm0.006$ & $0.438\pm0.007$ \\
        G-CaRL & $0.649\pm0.017$ & $0.673\pm0.013$ & $0.664\pm0.009$ & $0.750\pm0.019$ \\
        CtrlNS & $0.385\pm0.021$ & $0.522\pm0.084$ & $0.618\pm0.028$ & $0.459\pm0.061$ \\
        LEAP & $0.396\pm0.006$ & $0.524\pm0.009$ & $0.591\pm0.012$ & $0.486\pm0.014$ \\
        $\beta$-VAE & $0.357\pm0.003$ & $\mathbf{0.475\pm0.011}$ & $0.561\pm0.015$ & $0.445\pm0.011$ \\
        \hline
        Persistence (reference) & $0.477$ & $0.373$ & $0.356$ & $0.434$ \\
        \hline
    \end{tabular}

    \vspace{0.5em}

    \begin{tabular}{l|cccc}
        \hline
        \multicolumn{5}{c}{(b) Germany, Japan, South Korea, and US Atlantic} \\
        Method & GER & JP & KR & US \\
        \hline
        \methodname{} & $0.474\pm0.011$ & $\mathbf{0.414\pm0.013}$ & $\mathbf{0.291\pm0.007}$ & $0.427\pm0.009$ \\
        IDOL & $\mathbf{0.469\pm0.010}$ & $0.440\pm0.009$ & $0.304\pm0.007$ & $0.440\pm0.008$ \\
        G-CaRL & $0.754\pm0.028$ & $0.615\pm0.026$ & $0.535\pm0.016$ & $0.554\pm0.020$ \\
        CtrlNS & $0.504\pm0.043$ & $0.501\pm0.057$ & $0.382\pm0.048$ & $0.462\pm0.026$ \\
        LEAP & $0.566\pm0.018$ & $0.482\pm0.015$ & $0.327\pm0.011$ & $0.460\pm0.012$ \\
        $\beta$-VAE & $0.484\pm0.017$ & $0.426\pm0.009$ & $0.328\pm0.009$ & $\mathbf{0.416\pm0.008}$ \\
        \hline
        Persistence (reference) & $0.437$ & $0.446$ & $0.374$ & $0.334$ \\
        \hline
    \end{tabular}
\end{table}

\subsection{Data Sources and Availability}

The ERA5 hourly data on single levels used in the weather experiments are publicly available from the Copernicus Climate Change Service Climate Data Store~\citep{Copernicus2023} under DOI~\texttt{10.24381/cds.adbb2d47}; we retrieved them on September 2, 2026. 
We do not redistribute the data.
The results contain modified Copernicus Climate Change Service information 2020. 
Neither the European Commission nor ECMWF is responsible for any use that may be made of the Copernicus information or data it contains. 

The treadmill experiment uses the publicly available motion-capture dataset of
\citet{Riglet2024}, distributed through Figshare under
DOI~\texttt{10.6084/m9.figshare.c.7056797.v1}. 
We do not redistribute the source data.

\end{document}